\documentclass[runningheads]{llncs}

\usepackage{eccv}

\usepackage{eccvabbrv}

\usepackage{graphicx}
\usepackage{booktabs}

\usepackage[accsupp]{axessibility}  

\usepackage{hyperref}

\usepackage{orcidlink}
\usepackage{pifont}
\usepackage{algorithmicx,algorithm}
\usepackage{algpseudocode}
\usepackage{wrapfig}
\usepackage{diagbox}
\usepackage{enumitem}
\usepackage{multirow}
\usepackage{xcolor}
\usepackage{soul}
\usepackage{hyperref}
\usepackage{url}
\usepackage{thmtools,thm-restate}
\usepackage{tikz}
\usepackage{bm}

\usepackage{amsthm}

\usepackage{adjustbox}

\begin{document}
\newcommand{\name}{\text{CXGNN}}

\title{Graph Neural Network Causal Explanation \\ via Neural Causal Models} 

\titlerunning{Graph Neural Network Causal Explanation}

\author{Arman Behnam\inst{1}\orcidlink{0000-0002-4087-3274} \and
Binghui Wang\inst{1}\orcidlink{0000-0001-5616-060X}}

\authorrunning{Arman Behnam and Binghui Wang}

\institute{Illinois Institute of Technology, Chicago IL 60616, USA\\
\email{abehnam@hawk.iit.edu} \quad \email{bwang70@iit.edu}}

\maketitle

\begin{abstract}
  Graph neural network (GNN) explainers identify the important subgraph that ensures the prediction for a given graph. Until now, almost all GNN explainers are based on association, which is prone to spurious correlations. We propose {\name}, a GNN causal explainer via causal inference. Our explainer is based on the observation that a graph often consists of a causal underlying subgraph.  {\name} includes three main steps: 1) It builds causal structure and the corresponding structural causal model (SCM) for a graph, which enables the cause-effect calculation among nodes. 2) Directly calculating the cause-effect in real-world graphs is computationally challenging. It is then enlightened by the recent neural causal model (NCM), a special type of SCM that is trainable, and design customized NCMs for GNNs. By training these GNN NCMs, the cause-effect can be easily calculated. 3) It uncovers the subgraph that causally explains the GNN predictions via the optimized GNN-NCMs. Evaluation results on multiple synthetic and real-world graphs validate that {\name} significantly outperforms existing GNN explainers in exact groundtruth explanation identification\footnote{Code is available at \url{https://github.com/ArmanBehnam/CXGNN}}.
\end{abstract}
\section{Introduction}

Graph is a pervasive data type that represents complex relationships among entities. Graph Neural Networks (GNNs)~\cite{kipf2016semi,hamilton2017inductive,xu2018how,dwivedi2023benchmarking}, a mainstream learning paradigm for processing graph data, take a graph as input and learn to model the relation between nodes in the graph. GNNs have demonstrated state-of-the-art performance across various graph-related tasks such as node classification, link prediction, and graph classification, to name a few \cite{wu2020comprehensive}.

Explainable GNN provides a human-understandable way of the prediction outputted by a GNN. Given a graph and a label (correctly predicted by a GNN model), a GNN explainer aims to determine the important subgraph (called \emph{explanatory subgraph}) that is able to predict the label. Various GNN explanation methods~\cite{pope2019explainability,feng2023degree,baldassarre2019explainability,vu2020pgm,huang2022graphlime,duval2021graphsvx,zhang2021relex,pereira2023distill,yuan2020xgnn,wang2022gnninterpreter,shan2021reinforcement,li2023dag,GNNEx19,luo2020parameterized,schlichtkrull2021interpreting,funke2022z,wang2021towards,subgraphx_icml21,zhang2022gstarx} have been proposed, wherein almost all of them are based on \emph{associating} the prediction with a subgraph that has the maximum predictability (more details see Section~\ref{sec:related}). 
However, recent studies~\cite{wu2022discovering,fan2022debiasing,sui2022causal} show that association-based explanation methods are prone to biased subgraphs as the valid explanation due to \emph{spurious correlations} in the training data. For instance, when the groundtruth explanatory subgraph is the \emph{House}-motif, it often occurs with the \emph{Tree} bases. Then a GNN may not learn the true relation between the label and the \emph{House}-motif, but the \emph{Tree} base, due to it being easier to learn. We argue 
a truly explainable GNN should uncover the intrinsic \emph{causal relation} between the explanatory subgraph and the label, which we call the \emph{causal explanation}~\cite{schwab2019cxplain,beckers2022causal,geiger2022inducing}. Note that a few GNN explanation methods~\cite{sui2022causal,lin2021generative,lin2022orphicx} are motivated by the causality concepts, e.g., Granger causality~\cite{granger1969investigating}\footnote{Granger causality can only identify that one variable helps predict another, but it does not tell you which variable is the cause and which is the effect.}, but they are not causal explanations in essence.

{\bf Our GNN causal explainer:}  
In this paper, we take the first step to propose a {GNN explainer} via {causal inference}~\cite{pearl2016causal}, which focuses on understanding and quantifying cause-and-effect relations between variables in the task of interest. In the context of GNN causal explanation, we base on a common observation that a graph often consists of a causal subgraph 
and a non-causal counterpart~\cite{lin2021generative,lin2022orphicx,wu2022discovering,sui2022causal,fan2022debiasing,wang2022reinforced}.
Then given a graph and its (predicted) label, we aim to identify the \emph{causal explanatory subgraph} that causally yields such prediction. Our key idea is that the causal explainer should be able to identify the causal interactions among nodes/edges and interpret the label based on these interactions.  

Specifically, we propose a GNN causal explainer, called {\name}, that consists of three steps. 1) We first define causal structure (w.r.t. a reference node) for the graph, which admits structure causal models (we call GNN-SCM). Such GNN-SCM enables interventions to calculate cause and effects among nodes via do-calculus~\cite{pearl2016causal}.  
2) In real-world graphs, however, it is computationally challenging to perform do-calculus computation due to a large number of nodes and edges. To address it, we are inspired by the recent neural causal model (NCM)~\cite{xia2021causal},  which is a special type of SCM that can be trainable. We prove that, for each GNN-SCM, we can build a family of the respective GNN-NCMs. We then construct a parameterized GNN-NCM such that when it is optimized, the cause-effects defined on the GNN-NCM  are easily calculated. 3)  We finally determine the causal explanatory subgraph. To do so, we first introduce the node expressivity that reflects how well the reference node is in the causal explanatory subgraph. Then we iterate all nodes in the input graph and identify the trained GNN-NCM leading to the highest node expressivity. The underlying causal structure of this GNN-NCM is then the causal explanatory subgraph.

We evaluate {\name} on multiple synthetic and real-world graph datasets with groundtruth explanations, and compare them with state-of-the-art association-based and causality-inspired GNN explainers. Our results show {\name} significantly outperforms the baselines in exactly finding the groundtruth explanations. 
Our contributions are summarized below:
\begin{itemize}
    \item We propose the first GNN causal explainer {\name}. 
    \item We leverage the neural-causal connection, design the GNN neural causal models and train them to identify the causal explanatory subgraph. 
    \item 
    {\name} shows superiority over the state-of-the-art GNN explainers. 
\end{itemize}

\section{Related Work}
\label{sec:related}

\noindent{\bf Association-based explainable GNN:}
Almost all existing GNN explainers are based on association. These methods can be roughly classified into five types. (i) \textit{Decomposition-based methods~\cite{pope2019explainability,feng2023degree}} consider the prediction of a GNN model as a score and decompose it backward layer-by-layer until it reaches the input. The score of different parts of the input can be used to explain its importance to the prediction. (ii) \textit{Gradient-based methods~\cite{baldassarre2019explainability,pope2019explainability}} take the gradient of the prediction with respect to the input to show the sensitivity of a prediction to the input. The sensitivity can be used to explain the input for that prediction.
(iii) \textit{Surrogate-based methods~\cite{vu2020pgm,huang2022graphlime,duval2021graphsvx,zhang2021relex,pereira2023distill}} replace the GNN model with a simple and interpretable surrogate one. (iv) \textit{Generation-based methods~\cite{yuan2020xgnn,wang2022gnninterpreter,shan2021reinforcement,li2023dag}} use generative models or graph generators to generate explanations. 
(v) \textit{Perturbation-based methods~\cite{GNNEx19,luo2020parameterized,schlichtkrull2021interpreting,funke2022z,wang2021towards,subgraphx_icml21,zhang2022gstarx}} aim to 
find the important subgraphs as explanations by perturbing the input. 
State-of-the-art explainers from (iii)-(iv) show better performance than those from (i) and (ii).  

\noindent{\bf Causality-inspired explainable GNN:} 
Recent GNN explainers \cite{lin2021generative,lin2022orphicx,wu2022discovering,sui2022causal,fan2022debiasing,wang2022reinforced} are motivated by causality. 
These methods are based on a common observation that a graph consists of the causal subgraph and its non-causal counterpart. For instance, OrphicX~\cite{lin2022orphicx} uses information-theoretic measures of causal influence~\cite{ay2008information}, and proposes to identify the (non)causal factors in the embedding space via information flow maximization. CAL \cite{sui2022causal} introduces edge and node attention modules to estimate the causal and non-causal graph features. 

\noindent {\bf CXGNN vs. causality-inspired explainers.} 
The key difference lies in CXGNN focuses on identifying the \emph{causal explanatory subgraph} by \emph{directly quantifying the cause-and-effect relations} among nodes/edges in the graph. Instead, causality-inspired explainers are inspired by causality concepts to infer the explanatory subgraph, but they inherently do not provide causal explanations.  

\section{Preliminaries}
\label{sec:prelim}

In this section, we provide the necessary background on GNNs and causality to understand this work. For brevity, we will consider GNNs for graph classification.  

\noindent{\bf Notations:}  We denote a graph as $G=(\mathcal{V}, \mathcal{E})$, where $\mathcal{V}$ and $\mathcal{E}$ are the node set and edge set, respectively. $v \in \mathcal{V}(G)$ represents a node and $e_{u,v} \in \mathcal{E}(G)$ is an edge between $u$ and $v$. 
Each graph $G$ is associated with a label $y_G \in \mathcal{Y}$, with $\mathcal{Y}$ the label domain. An uppercase letter $X$ and the corresponding lowercase one $x$ indicate a random variable and 
its value, respectively; bold  ${\bf X}$ and ${\bf x}$ denote a set of random variables and its corresponding values, respectively. 
We use $\textrm{Dom}(X)$ to denote the domain of $X$, and $P({\bf X})$ as a probability distribution over ${\bf X}$. 

\noindent{\bf Graph neural network (GNN):} A GNN is a multi-layer neural network that operates on the graph and iteratively learns node/graph representations via message passing. A GNN mainly uses two operations to compute node representations in each layer. Assume a node $v$'s representations in the $(l-1)$-th layer is learnt and denoted as $h_v^{(l-1)}$. In the $l$-th layer, the message between two connected nodes $u$ and $v$ is defined as $m^l_{u,v}$ = MSG($h^{l-1}_{u}$, $h^{l-1}_{v}$, $e_{u,v}$). The aggregated message for node $u$ is then defined as $u$'s representation in the current layer $l$:  $h^{l}_{u}$ = AGG($m^l_{u,v}: v \in \mathcal{N}_1(u)$). Assume $L$ layers of computation, the final representation for $v$ is $\boldsymbol{z}_{v}=\boldsymbol{h}_{v}^{(L)}$ and ${\bf Z} = \{{\bf z}_v\}_{v \in \mathcal{V}}$. 
GNN can add a predictor on top of ${\bf Z}$ to perform graph-relevant tasks. For instance, when graph classification is the task of interest, GNNs use ${\bf Z}$ to predict the label for a whole graph.

\noindent{\bf Structural Causal Model (SCM):} 
SCMs~\cite{pearl2009causality} provide a rigorous definition of cause-effect relations between random variables. 
An SCM $\mathcal{M}$ is a four-tuple $\mathcal{M} \equiv ({\bf U}, {\bf V},\mathcal{F},P({\bf U}))$, where ${\bf U}$ is a set of exogenous (or latent) variables determined by factors outside the model and they are the only source of randomness in an SCM; $V$ is a set  $\{V_1, V_2, \ldots, V_n\} $ of $n$ endogenous (or observable) variables of interest determined by other variables within the model, i.e., in ${\bf U} \cup {\bf V}$; $\mathcal{F}$ is set of functions (define causal mechanisms) $\{f_{V_1}, f_{V_2}, \ldots, f_{V_n}\}$ such that each $f_{V_i}$ is a mapping function from ${\bf U}_{V_i} \bigcup \textbf{Pa}_{V_i} $ to $V_i$, where ${\bf U}_{V_i} \subseteq {\bf U}$ and $\textbf{Pa}_{V_i} \subseteq {\bf V} \setminus V_i$ is the parent of $V_i$. That is, $v_i \leftarrow f_{V_i}(\textbf{pa}_{V_i}, {\bf u}_{V_i})$ and $\mathcal{F}$ forms a mapping from ${\bf U}$ to ${\bf V}$. $P({\bf U})$ is the probability function over the domain of ${\bf U}$. With SCM, one can perform interventions to find causes and effects and design a model that has the capability of predicting the effect of interventions. 
\begin{definition}[Intervention and Causal Effects]
\label{Valuation}
Interventions are changes made to a system to study the causal effect of a particular variable or treatment on an outcome of interest. An SCM $\mathcal{M}$ induces a set of interventional distributions over ${\bf V}$, one for each intervention $do({\bf X}={\bf x})$ (short for $do({\bf x})$), which forces the value of variable ${\bf X} \subseteq {\bf V}$ to be ${\bf x}$. Then for each ${\bf Y} \subseteq {\bf V}$: 
{
\small
\begin{equation} \label{eq:k_2 valuation}
        p^{\mathcal{M}}({\bf y} |do({\bf x})) = \sum_{\{{\bf u}|{\bf Y}_{\bf x}({\bf u})={\bf y}\}} P({\bf u}). 
    \end{equation}
}
\end{definition}
In words, an intervention forcing a set of variables ${\bf X}$ to take values ${\bf x}$ is modeled by replacing the original mechanism $f_X$ for each $X \in {\bf X}$ with its corresponding value in ${\bf x}$. 
The impact of the intervention ${\bf x}$ on an outcome variable ${\bf Y}$ is called potential response ${\bf Y}_{{\bf x}}({\bf u})$, which expresses \emph{causal effects} and is the solution for ${\bf Y}$ after evaluating: 
$\mathcal{F}_{\bf x}:= \{f_{V_i}: V_i \in {\bf V} \setminus {\bf X}\} \cup \{ f_X \leftarrow x: X \in {\bf X}\}$.

One possible strategy to estimate the underlying SCM of a task is using 
its observational inputs and outputs. However, a critical issue is that causal properties are \emph{provably impossible} to recover solely from the joint distribution over the input graphs and labels~\cite{pearl2016causal}. 
In this paper, we are inspired by the emerging  \emph{Neural Causal Model (NCM)}~\cite{xia2021causal}, which is a special type of SCM that is amenable to gradient descent-based optimization. 

\noindent{\bf Neural Causal Model (NCM):} 
A NCM~\cite{xia2021causal} $\widehat{\mathcal{M}}(\theta)$ over variables ${\bf V}$ with parameters $\theta=\{ \theta_{V_i}; V_i \in V\}$ is an SCM estimation as: $\widehat{\mathcal{M}}(\theta)  \equiv (\widehat{\bf U}, {\bf V},\mathcal{\widehat{F}} ,P(\widehat{\bf U}))$, where 1) $\widehat{\bf U} \subseteq \{{\widehat{U}_{\bf C}; {\bf C} \subseteq {\bf V}}\}$, with each $\widehat{U}$ associated with some subset of variables ${\bf C} \subseteq {\bf V}$, and $\text{Dom}({\widehat{\bf U})} = [0,1]$; 2) $\mathcal{\widehat{F}} = \{\widehat{f}_{V_i}: V_i \in {\bf V} \}$, with each $\widehat{f}_{V_i}$ a  neural network parameterized by $\theta_{V_i}$ that maps $\widehat{\bf U}_{V_i} \cup \textbf{Pa}_{V_i}$ to $V_i$, and $\widehat{\bf U}_{V_i} = \{\widehat{U}_{\bf C}: V_i \in {\bf C}\}$; 3) $P(\widehat{\bf U}): \widehat{U} \sim \text{Unif}(0,1), \forall \widehat{U} \in \widehat{\bf U}$. 

\cite{xia2021causal} shows that \emph{NCM is proved to be as expressive as SCM}, and hence \emph{all NCMs are SCMs}. However, expressiveness does not mean the learned NCM model has the same empirical observations as the SCM model. 
To ensure equivalence, there should be a necessary structural assumption on NCMs, called \emph{causal structure consistency}. More details are referred to \cite{xia2021causal} and Appendix~\ref{app:background}.  
\section{GNN Causal Explanation via NCMs}
\label{sec:GNN-NCM}

In this section, we propose our GNN causal explainer, {\name}, 
for explaining graph classification. Our explainer also utilizes the common observation that a graph consists of a causal subgraph and a non-causal counterpart~\cite{lin2021generative,lin2022orphicx,wu2022discovering,sui2022causal,fan2022debiasing,wang2022reinforced}. The overview of CXGNN is shown in Figure~\ref{fig:Causal graph explanation} in Appendix and all proofs are deferred to Appendix~\ref{app:proofs}. 

\subsection{Overview}
Given a graph $G=(\mathcal{V},\mathcal{E})$ and a ground truth or predicted label by a GNN model, our {causal explainer}  bases on {causal learning} and identifies the \emph{causal explanatory subgraph} (denoted as $\Gamma$) that intrinsically yields the label. 

Our {\name} consists of three key steps: 
1) define the causal structure $\mathcal{G}$ for the graph $G$ and the respective SCM $\widehat{\mathcal{M}}(\mathcal{G})$ (we call GNN-SCM) to enable causal effect calculation via interventions; 
2) However, directly calculating the causal effect
in real graphs is computationally challenging. We then construct and train a family of parameterized GNN neural causal model $\widehat{\mathcal{M}}(\mathcal{G}, \theta))$ (we call GNN-NCM), a special type of GNN-SCM that is trainable.
3) We uncover the causal explanatory subgraph (denoted as $\Gamma$) based on the trained GNN-NCM that best yields the graph label. Next, we will illustrate step-by-step in detail. 


\subsection{Causal Structure and Induced SCM on a Graph}
 
In the context of causality, the problem of GNN explanation can be solved by cause and effect identification among nodes and their connections in a graph. Interventions enable us to interpret the causal relation between nodes. To perform interventions on a graph, one often needs to first define the causal structure for this graph, which involves the observable and latent variables. 

\noindent {\bf Observable and latent variables in a graph:}
Given a $G=(\mathcal{V},\mathcal{E})$. For each 
node $v \in \mathcal{V}$, there are both known and unknown effects from other nodes and edges on $v$, which we call observable variables (denoted as ${\bf V}_{v}$) and latent variables (denoted as ${\bf U}_{v}$), respectively. 
With it, we define a congruent causal structure for enabling the graphs to admit SCMs. 
\begin{definition} [Causal structure of a graph] \label{causal network}
Consider a graph $G=(\mathcal{V},\mathcal{E})$, we define the causal structure $\mathcal{G}$ of $G$ as a subgraph that centers on a reference node  $v$ and accepts the SCM structure: 
{
\small
\begin{align}
    & \mathcal{G}(G) = \Big\{ {\bf V}_{v} = \{y_{v}\} \cup \{y_{v_i}:
    v_i \in  \mathcal{N}_{\leq k}(v) \}, 
{\bf U}_{v} = \{{\bf U}_{v_i}: v_i \in \mathcal{N}_{\leq k}(v) \} \cup \{{\bf U}_{v,v_i}:{e_{v,v_i} \in \mathcal{E}}\}\Big\}, \label{causal net} 
\end{align}
}%
where $v$ is to be learnt (see Section~\ref{sec:causalexp}), $y_{v_i}$ is the node $v_i$'s label, $\mathcal{N}_{\leq k}(v)$ means  nodes within the $k$-hop neighbors of $v$, ${\bf U}_{v_i}$ is $v_i$'s latent variable, called \emph{node effect}; and ${\bf U}_{v, v_i}$ the edge $e_{v, v_i}$ latent variable, called \emph{edge effect}. In practice, we can specify  ${\bf U}_{v_i}$ and ${\bf U}_{v, v_i}$ as random variable, e.g., from a Gaussian distribution.
\end{definition}

With a causal structure for a graph, we can build the corresponding SCM in the following theorem:

\begin{restatable}[GNN-SCM]{theorem}{thmgnnscm}
\label{thm:GNN-SCM}
For a GNN operating on a graph $G$, there exists an SCM $\mathcal{M}(\mathcal{G})$ w.r.t. the causal structure $\mathcal{G}$ of the graph $G$.
\end{restatable}
Appendix~\ref{app:scm-example} shows an example on how to compute the causal effects on a toy graph via a SCM truth table. 

\subsection{GNN Neural Causal Model}

In reality, it is computationally challenging to build a truth table for variables in GNN-SCM and perform do-calculus computation due to the large number of nodes/edges in real-world graphs. Such a challenge impedes the calculation of causal effects. 
To address it, we are motivated by estimating the causal effect via NCM (see Section \ref{sec:prelim}). 
Specifically, Definition~\ref{def:g-cons-nscm} in Appendix shows: to ensure the equivalence between NCM and SCM, NCM is required to be $\mathcal{G}$-constrained. However, the general $\mathcal{G}$-constrained NCM cannot be directly applied in our setting. 
To this end, we first define a customized $\mathcal{G}$-constrained GNN-NCM as below: 
\begin{definition} [$\mathcal{G}$-Constrained GNN-NCM (constructive)] 
\label{lem:g-cons-gnn-ncm}
Let GNN-SCM $\mathcal{M}(\mathcal{G, \theta})$ be induced from the causal structure $\mathcal{G}(G)$ on a graph $G$. Then GNN-NCM $\widehat{\mathcal{M}}(\mathcal{G}, \theta)$ will be constructed based on the causal structure  $\mathcal{G}(G)$. 
\end{definition}

This construction ensures that any inferences made by $\widehat{\mathcal{M}}_{NCM}(\mathcal{G}, \theta)$ respect the causal dependencies as captured by $\mathcal{G}(G)$. Note that $\widehat{\mathcal{M}}(\mathcal{G},\theta)$  represents a family of GNN-NCMs since the parameters $\theta$ of the neural networks are not specified by the construction. Next, we propose a construction of a $\mathcal{G}$-constrained GNN-NCM, following Definition~\ref{lem:g-cons-gnn-ncm}.

\noindent {\bf GNN Neural Causal Model Construction} One should consider the sound and complete structure of GNN-NCMs that are consistent with Definition \ref{causal network}. Here, we define the general GNN-NCM structure as shown in below Equation \ref{eq:gnn-ncm}, which is an instantiation of Theorem~\ref{thm:thmgnnncm}. 
{
\begin{align}
\scriptsize
\label{eq:gnn-ncm}
    {\widehat{\mathcal{M}}(\mathcal{G},\theta) =} 
        \begin{cases}
            {\bf V}:= \mathcal{V}(\mathcal{G}) \\
            \widehat{\bf U} := \{\widehat{\bf U}_{v_i}, 
            v_i \in \mathcal{V}(\mathcal{G}\}, 
            P(\widehat{\bf U}) := \{\widehat{\bf U}_{v_i} \sim \text{Unif}(0,1) 
            \}  \cup \{ T_{k, v_i} \sim \mathcal{N}(0,1): 
            k \in \{0,1\} \} \\
            \widehat{f}_{v_i}(\widehat{\bf u}_{v_i}, \widehat{\bf u}_{{v_i},{v_j}}) := 
            \arg\max\limits_{k\in\{0,1\}} T_{k,v_i}+
            \begin{cases}
                \log \sigma (ff_{v_i}(\widehat{\bf u}_{v_i}, \widehat{\bf u}_{{v_i},{v_j}};\theta_{v_i})) \text{ if } {k=1}\\
                \log (1 - \sigma (ff_{v_i}(\widehat{\bf u}_{v_i}, \widehat{\bf u}_{{v_i},{v_j}};\theta_{v_i}))) \text{ if } {k=0}, \\
            \end{cases} 
            \\
             \widehat{\mathcal{F}} := \{ \widehat{f}_{v_i}(\widehat{\bf u}_{v_i},\widehat{\bf u}_{{v_i},{v_j}})\}
        \end{cases}
\end{align}
}%

\begin{restatable}[GNN-NCM]{theorem}{thmgnnncm}
\label{thm:thmgnnncm}
Given causal structure $\mathcal{G}$ of a graph $G$ and the underlying GNN-SCM $\mathcal{M}(\mathcal{G})$, 
there exists a $\mathcal{G}$-constrained GNN-NCM $\widehat{\mathcal{M}}(\mathcal{G}, \theta)$ that enables any inferences consistent with $\mathcal{M}(\mathcal{G})$.
\end{restatable}

In Equation \ref{eq:gnn-ncm}, 
${\bf V}$ are the nodes in the causal structure $\mathcal{G}(G)$; each $T_{v_i}$ is a standard Gaussian random variable; each $ff_{v_i}$ is a feed-forward neural network on $v_i$ parameterized by $\theta_{v_i}$ (note one requirement of $ff_{v_i}$ is it could approximate any continuous function), and $\sigma$ is sigmoid activation function. The parameters $\{\theta_{v_i}\}$ are not yet specified and must be learned through training the NCM.

\begin{algorithm}[!t]
\caption{GNN Neural Causal Model Training}
\label{Algorithm1}
\footnotesize
{\bf Input:} The causal structure $\mathcal{G}$ (including a reference node $v$, its within $k$-hop neighbors $\mathcal{N}_{\leq k}(v)$, and set of latent variables ${\bf U}_{v}$), node label $y_v$ \\
{\bf Output:} An optimized GNN-NCM $\widehat{\mathcal{M}}(\mathcal{G}, \theta^{*})$ for the causal structure $\mathcal{G}$ centered at $v$ 
\begin{algorithmic}[1]
\State Build the GNN-NCM $\widehat{\mathcal{M}}(\mathcal{G},\theta)$ 
based on $\mathcal{G}$ and Eqn. \ref{eq:gnn-ncm}
\For{each node $v_i \in \mathcal{N}_{\leq k}(v)$}
\State  Calculate $p^{\widehat{\mathcal{M}}(\mathcal{G},\theta)}(y_{v} \mid do(v_i))$ via Eqn. \ref{eq:NM3}
\EndFor
\State Calculate $p^{\widehat{\mathcal{M}}(\mathcal{G},\theta)}(y_{v})$ via Eqn.~\ref{eq:NM6}
\State Calculate the loss $\mathcal{L}({\widehat{\mathcal{M}}(\mathcal{G},\theta)}; v)$ via Eqn. \ref{eq:cel}
\State Minimize the loss to reach the GNN-NCM $\widehat{\mathcal{M}}(\mathcal{G}, \theta^{*})$  
\end{algorithmic}
\end{algorithm}

\setlength{\textfloatsep}{2mm}

\noindent {\bf Training Neural Networks for GNN-NCMs}
We now compute the causal effects on a target node $v$. Based on Definition~\ref{Valuation} and the constructed GNN-NCM $\widehat{\mathcal{M}}(\mathcal{G},\theta)$ in Equation \ref{eq:gnn-ncm}, 
the causal effect on $v$ of an intervention $do(v_i)$ 
($v_i \in \mathcal{N}_{1}(v))$ is $p^{\widehat{\mathcal{M}}(\mathcal{G},\theta)}(y_{v}|do(v_i))$. 
This do-calculus then can be calculated as the expected value of nodes and edges affects values for $v$ shown below: 
{
\footnotesize
\begin{align}\label{eq:NM3}
    & p^{\widehat{\mathcal{M}}(\mathcal{G},\theta)}(y_{v} \mid do(v_i)) = \mathbb{E}_{p(\widehat{\bf u}_{v})} \Big[ \prod_{(v, v_j) \in \mathcal{E}(\mathcal{G})} \widehat{f}_{v_i}(\widehat{\bf u}_{v_j},\widehat{\bf u}_{v, v_j})\Big] \nonumber \\
     & \quad \quad \approx \frac{1}{|\mathcal{N}_{\leq k}(v)|} \sum_{v_i \in \mathcal{N}_{\leq k}(v)} \prod_{(v, v_j) \in \mathcal{E}(\mathcal{G})} \widehat{f}_{v_i}(\widehat{\bf u}_{v_j},\widehat{\bf u}_{v, v_j}).
\end{align}
}%
Then one can calculate the probability of the target node label $y_{v}$  as the expected value of all the effects from the neighbor nodes on  $v$:
{
\footnotesize
\begin{align} 
\label{eq:NM6}
    & p^{\widehat{\mathcal{M}}(\mathcal{G},\theta)}(y_{v}) = \mathbb{E}_{p(\hat{\bf u}_{v})} \left[ \widehat{f}_{v}\right] 
    \approx \frac{1}{|\mathcal{N}_1(v)|}  
    \frac{1}{|\mathcal{Y}|} \sum_{y \in \mathcal{Y}} \sum_{v_i \in \mathcal{N}_1(v)} p^{\widehat{\mathcal{M}}(\mathcal{G},\theta)}(y_{v} = y \mid do(v_i)) 
\end{align}
}

The true GNN-SCM induces a causal structure that encodes constraints over the interventional distributions. We now first investigate the feasibility of causal inferences in the class of $\mathcal{G}$-constrained GNN-NCMs. These models approximate the likelihood of the observed data based on the graph's latent variables. The cross-entropy loss measures the discrepancy between the target node's label prediction and its true label.
Inspired by \cite{xia2021causal}, we define the GNN-NCM loss as: 
{
\footnotesize
\begin{align}\label{eq:cel}
\mathcal{L} (\widehat{\mathcal{M}}(\mathcal{G},\theta); v) = - \sum_{y_v \in \mathcal{Y}} y_{v} \log(p^{\widehat{\mathcal{M}}(\mathcal{G},\theta)}({y_{v}})) 
\end{align}
}%

To train neural networks for GNN-NCMs, one should generate samples from the GNN-SCM. If provided, it is 
the specific realization of the interventions. Specifically, GNN-NCMs are trained on node effects $\hat{\bf u}_{v_i}$ and edge effects $\hat{\bf u}_{{v_i},{v_j}}$ on the target node, as shown in Equation  \ref{eq:gnn-ncm}, and should specify $\widehat{f}_{v_i}(\hat{\bf u}_{v_i}, \hat{\bf u}_{{v_i},{v_j}})$. 
Then a model, denoted as $\theta^*$, is achieved by minimizing the GNN-NCM loss:
{
\small
\begin{equation}\label{eq:theta}
    \theta^{*} \in \arg\min_{\theta}  \mathcal{L}(\widehat{\mathcal{M}}(\mathcal{G},\theta); v)
\end{equation}
}%
Details of training GNN-NCMs are shown in Algorithm \ref{Algorithm1}. 
Basically, this algorithm takes the causal structure \(\mathcal{G}\) with respect to a reference node $v$ as input and returns an optimized GNN-NCM model \(\widehat{\mathcal{M}}(\mathcal{G}, \theta^*)\).

\begin{algorithm}[!t]
\caption{{\name:} GNN Causal Explainer}
\label{Algorithm2}
{\bf Input:}  Graph $G$ with label, and expressivity threshold $\delta$ \\
{\bf Output:}  Explanatory subgraph $\Gamma$
\begin{algorithmic}[1]
\For{each node $v \in \mathcal{V}(G)$}
\State Build $\mathcal{G}_{v}$ based on the reference node $v$  
\State Train the GNN-NCM $\widehat{\mathcal{M}}(\mathcal{G}_v, \theta_v^{*})$ via Alg. \ref{Algorithm1} and calculate the node expressivity $\text{exp}_v(\widehat{\mathcal{M}}(\mathcal{G}_v, \theta_v^{*}))$
\EndFor
\State Find $v^* = \text{argmax}_{v \in \mathcal{V}(G)} \text{exp}_{v}(\widehat{\mathcal{M}}(\mathcal{G}_v,\theta_v^*))$; 
\State Return the explanatory subgraph $\Gamma$  induced by $\mathcal{G}_{v^*}$
\end{algorithmic}
\end{algorithm}

\subsection{Realizing GNN Causal Explanation}
\label{sec:causalexp}
The remaining question is: how to find the causal explanatory subgraph $\Gamma$ from a graph $G$ to causally explain GNN predictions? 
The answer is using the trained GNN-NCMs $\widehat{\mathcal{M}}(\mathcal{G},\theta^*)$. Before that, the first step is to clarify a node's role in GNN-NCMs for explanation.
\begin{restatable}[Node explainability]{theorem}{nodeex}
\label{thm:nodeex}
Let a prediction for a graph $G$ be explained. A node $v \in {G}$  is causally explainable, if $p^{\widehat{\mathcal{M}}(\mathcal{G}(G),\theta)}(y_{v})$ can be computed.
\end{restatable}
The $\mathcal{G}$-constrained GNN-NCM is trained on interventions and can interpret the GNN predictions. Moreover, the information extracted from interventions can be used for interpreting nodes. Specifically, we define expressivity to measure the information for an explainable node.  
\begin{restatable}[Explainable node expressivity]{theorem}{nodeexp}
\label{thm:nodeexp}
An explainable node $v$ has expressivity defined as 
$\text{exp}_{v}(\widehat{\mathcal{M}}(\mathcal{G},\theta))  = \sum_{y_{v}} y_{v} p^{\widehat{\mathcal{M}}(\mathcal{G},\theta)}(y_{v})$.
\end{restatable}
In other words, the node expressivity reflects how well the node is in the causal explanatory subgraph. Now we are ready to realize GNN causal explanation based on learned GNN-NCMs. Given a graph $G$, we start from a random node $v$, and build the causal structure $\mathcal{G}$ centered on $v$. By Algorithm \ref{Algorithm1}, we can reach an optimized GNN-NCM $\widehat{\mathcal{M}}(\mathcal{G}, \theta^{*})$ and obtain the $v$'s expressivity. 

We repeat this process for all nodes in the graph $G$ and find the node $v^*$ with the associated $\widehat{\mathcal{M}}(\mathcal{G},\theta^*)$ yielding the highest expressivity $\text{exp}_{v^*}(\widehat{\mathcal{M}}(\mathcal{G},\theta^*))$. The underlying subgraph of the causal structure centered by $v^*$ is then treated as the causal explanatory subgraph $\Gamma$. Algorithm \ref{Algorithm2} describes the learning 
process.

\section{Experiments}
\label{sec:exp}

\begin{table}[!t]
\centering
\footnotesize
\addtolength{\tabcolsep}{-2pt}
\caption{Dataset statistics.} 
\begin{tabular}{lccc}
\toprule
& \textbf{Avg. \#nodes} & \textbf{ Avg. \#edges} & \textbf{ \#test graphs} \\
\midrule
\textbf{BA+House} & $11.97$ & $18.17$ & $500$ \\
\textbf{BA+Grid} & $15.96$ & $24.20$ & $500$ \\
\textbf{BA+Cycle} & $10.0$ & $10.5$ & $500$ \\
\textbf{Tree+House} & $12$ & $13$ & $500$ \\
\textbf{Tree+Cycle}  & $13$ & $13.50$ & $500$ \\
\textbf{Tree+Grid}  & $24$ & $27$ & $500$ \\
\textbf{Benzene} & $20.48$ & $21.73$ & $100$ \\
\textbf{Fluoride carbonyl} & $20.66$ & $22.03$ & $100$ \\
\bottomrule
\end{tabular}
\label{tab:dataset}
\end{table}

\subsection{Experimental Setup}

\noindent {\bf Datasets:}
Following prior works~\cite{GNNEx19,lin2021generative}, we use six synthetic datasets,   and two real-world datasets with groundtruth explanation for evaluation. 
Dataset statistics are shown in Table~\ref{tab:dataset}.

\begin{itemize}
\item \emph{Synthetic graphs}: {\bf 1) BA+House:} This graph stems from a base random Barabási-Albert (BA) graph attached with a 5-node ``house"-structured motif as the groundtruth explanation;  {\bf 2) BA+Grid:} This graph contains a base random BA graph and is attached with a 9-node ``grid" motif as the groundtruth explanation; 
{\bf 3) BA+Cycle:} A 6-node ``cycle" motif is appended to randomly chosen nodes from the base BA graph. The ``cycle" motif is the groundtruth explanation;
{\bf 4) Tree+House:}  The core of this graph is a balanced binary tree. The 5-node ``house" motif, as the groundtruth explanation, is attached to random nodes from the base tree. 
{\bf 5) Tree+Grid:} Similarly, binary tree a the core graph and a 9-node ``grid" motif as the groundtruth explanation is attached;
{\bf 6) Tree+Cycle:} A 6-node ``cycle" motif, the groundtruth explanation, is appended to nodes from the binary tree.  
The label of the synthetic graph is decided by the label of nodes in the  groundtruth explanation. Following existing works~\cite{GNNEx19,lin2021generative}, a node $v$'s label $y_v$ is set to be 1 if $v$ is in the groundtruth, and 0 otherwise. Hence, in these graphs, the base graph acts as the non-causal subgraph that can cause the spurious correlation, while the attached motif can be seen as the causal subgraph, as it does not change across graphs and decides the graph label.
\item \emph{Real-world graphs:} We use two representative real-world graph datasets with groundtruth~\cite{agarwal2023evaluating}. {\bf 1) Benzene:} it includes 12,000 molecular graphs extracted from the ZINC15~\cite{sterling2015zinc} database and the task is to identify whether a given molecule graph has a benzene ring or not. The groundtruth explanations are the nodes (atoms) forming the benzene ring. {\bf 2) Fluoride carbonyl:}  This dataset contains 8,671 molecular graphs with two classes: a positive class means a molecule graph contains a fluoride (F-) and a carbonyl (\text{C=O}) functional group. The groundtruth explanation consists of combinations of fluoride atoms and carbonyl functional groups within a given molecule. 
\end{itemize} 

\noindent {\bf Models and parameter setting:} 
In {\name}, we use a feedforward neural network to parameterize GNN-NCM. The neural network consists of an input layer, two fully connected hidden layers, and an output layer. ReLU activation functions is used in all hidden layers, while a softmax activation function is applied to the output layer. The input to the network is the target node $v$'s node effects and edge effects (see Equation~\ref{causal net}), whose values are sampled from a standard Gaussian distribution, and the output is the predicted causal effect on $v$.  The detailed hyperparameters are shown in Appendix~\ref{app:setup}. The hyperparameters in the compared GNN explainers are optimized based on their source code.

\noindent {\bf Baseline GNN explainers:}
We compare {\name} with both association-based and causality-inspired GNN explainers.
We choose 4 representative 
ones: gradient-based Guidedbp~\cite{gu2019saliency},  perturbation-based  GNNExplainer~\cite{GNNEx19},  surrogate-based PGMExplainer~\cite{vu2020pgm}, and causality-inspired GEM~\cite{lin2021generative}, RCExplainer~\cite{wang2022reinforced}, and
OrphicX~\cite{lin2022orphicx}. 
We use the public source code of these explainers for comparison. The causality-inspired explainers are inspired by causality concepts to infer the explanatory subgraph, but they inherently do not provide causal explanations.

\noindent {\bf Evaluation metrics:}
Given a set of testing graphs $\mathbb{G}$. For each test graph $G \in \mathbb{G}$, we let its groundtruth explanatory subgraph be $\Gamma_G$ and the estimated explanatory subgraph by a GNN explainer be $\Gamma$. We use two common metrics, i.e., graph explanation accuracy and explanation recall from the literature
\cite{agarwal2023evaluating}. In addition, to justify the superiority of our causal explainer, we introduce a third metric groundtruth match accuracy, which is the most challenging one.  

\begin{itemize}
\item {\bf Graph explanation accuracy:} For a graph $G$, the graph explanation accuracy is defined as the fraction of nodes in the estimated explanatory subgraph $\Gamma$ that are contained in the groundtruth $\Gamma_G$, i.e.,  ${|V(\Gamma) \cap V(\Gamma_G)|}/{|V(\Gamma_G)|}$. We then report the average accuracy across all testing graphs.  

\item {\bf Graph explanation recall:} Different GNN explainers output the estimated explanatory subgraph with different node sizes. When two explainers output the same number of nodes in  $\Gamma_G$, the one with a smaller node size should be treated as having a better quality.  To account for this, we use the explanation recall metric that is defined as ${|V(\Gamma) \cap V(\Gamma_G)|}/{|V(\Gamma)|}$ for a given graph $G$. 
We then report the average recall across all testing graphs.

\item {\bf Groundtruth match accuracy:} For a testing graph $G$, we count a 1 if the estimated $\Gamma$ and groundtruth  $\Gamma_G$ exactly match, i.e., $\Gamma_G = \Gamma$,  and 0 otherwise. In other words, the groundtruth match accuracy of all testing graphs $\mathbb{G}$ is defined as 
${\sum_{G \in \mathbb{G}} {\bf 1}[{\Gamma_G = \Gamma}]}/{|\mathbb{G}|}$, 
where ${\bf 1}[\cdot]$ is an indicator function. 
\vspace{-4mm}
\end{itemize}

\begin{table}[!t]\renewcommand{\arraystretch}{0.85}
\centering
\footnotesize
\caption{Comparison results on the synthetic datasets. B.H.: BA+House; B.G.: BA+Grid; B.C.: BA+Cycle; T.H.: Tree+House; T.G.: Tree+Grid; T.C.: Tree+Cycle.}
\begin{tabular}{lcccccc}
\toprule
\multicolumn{7}{c}{\textbf{Graph explanation accuracy (\%)}} \\ 
\noalign{\smallskip}
\hline 
\noalign{\smallskip}
& \textbf{B.H.} & \textbf{B.G.} & \textbf{B.C.} & \textbf{T.H.} & \textbf{T.C.}& \textbf{T.G.}  \\
\textbf{GNNExp.~\cite{GNNEx19}} & 75.60 & 76.16 & 75.13 & 77.24 & 71.60 & 72.18  \\ 
\textbf{PGMExp.~\cite{vu2020pgm}} & 61.60 & 44.98 & 63.07 & 58.28 & 49.90 & 37.42 \\
\textbf{Guidedbp~\cite{gu2019saliency}} & 60.00 & 0.00 & 66.67 & 0.00 & 0.00 & 0.00  \\
\textbf{GEM~\cite{lin2021generative}} & 98.2 & 88.19 & 97.91 & 96.23 & 95.51 & 86.96  \\
\textbf{RCExp.~\cite{wang2022reinforced}} & \textbf{100.00} & 88.89 & \textbf{100.00} & \textbf{100.00} & \textbf{100.00} & \textbf{100.00}  \\
\textbf{OrphicX~\cite{lin2022orphicx}} & 88.00 & 89.00 & 55.65 & 96.20 & \textbf{100.00} & 99.93  \\
\textbf{\name} & \textbf{100.0} & \textbf{100.00} & 83.33 & \textbf{100.0} & 82.67 & \textbf{100.00} \\
\hline \hline 
\noalign{\smallskip}

\multicolumn{7}{c}{\textbf{Graph explanation recall (\%)}} \\ 
\noalign{\smallskip}
\hline 
\noalign{\smallskip}
& \textbf{B.H.} & \textbf{B.G.} & \textbf{B.C.} & \textbf{T.H.} & \textbf{T.C.}& \textbf{T.G.}  \\
\textbf{GNNExp.~\cite{GNNEx19}} & 37.62 & 52.72 & 45.08 & 32.18 & 33.05 & 40.60  \\ 
\textbf{PGMExp.~\cite{vu2020pgm}} & 30.80 & 31.14 & 37.84 & 24.28 & 23.93 & 21.05 \\
\textbf{Guidedbp~\cite{gu2019saliency}} & 12.40 & 17.94  & 16.18 & 5.99 & 12.98 & 15.38  \\
\textbf{GEM~\cite{lin2021generative}} & 39.18 & 50.86 & 45.40 & 38.75 & 34.65 & 41.20  \\
\textbf{RCExp.~\cite{wang2022reinforced}} & 100.00 & 60.00 & 89.52 & 45.45 & 46.60 & 39.13   \\
\textbf{OrphicX~\cite{lin2022orphicx}} & 98.08 & \textbf{97.71} & 60.00 & 41.38 & 59.22 & 40.61   \\
\textbf{\name} & \textbf{100.0} & 60.55 & \textbf{90.00} & \textbf{61.67} & \textbf{68.15} & \textbf{49.05} \\
\hline \hline 
\noalign{\smallskip}
\multicolumn{7}{c}{\textbf{Groundtruth match accuracy (\%)}} \\ 
\noalign{\smallskip}
\hline 
\noalign{\smallskip}
& \textbf{B.H.} & \textbf{B.G.} & \textbf{B.C.} & \textbf{T.H.} & \textbf{T.C.}& \textbf{T.G.}  \\
\textbf{GNNExp.~\cite{GNNEx19}} & 0.20 & 2.20 & 2.20 & 0.80 & 0.40 & 0.20  \\ 
\textbf{PGMExp.~\cite{vu2020pgm}} & 1.00 & 0.00 & 0.00 & 3.80 & 0.00 & 0.00 \\
\textbf{Guidedbp~\cite{gu2019saliency}} & 1.00 & 0.6 & 0.6 & 0.6 & 0.2 & 0.6  \\
\textbf{GEM~\cite{lin2021generative}} & 0.80 & 6.00 & 6.00 & 2.50 & 1.20 & 1.00  \\
\textbf{RCExp.~\cite{wang2022reinforced}} & 100.00 & 0.00 & 49.60 & 0.00 & 0.00 & 0.00 \\
\textbf{OrphicX~\cite{lin2022orphicx}} & 39.00 & 43.00 & 5.00 & 1.40 & 21.00 & 33.00 \\
\textbf{\name} & \textbf{100.0} & \textbf{44.00} & \textbf{67.60} & \textbf{99.40} & \textbf{61.20} & \textbf{46.00} \\
\hline
\end{tabular}
\label{tab:synthetic_comparison}
\end{table}

\subsection{Results on Synthetic Datasets}

\noindent {\bf Comparison results:}
Table~\ref{tab:synthetic_comparison} shows the results of all the compared GNN explainers on the 6 synthetic datasets with 500 testing graphs and 3 metrics. We have several observations. In terms of explanation accuracy, {\name} performs comparable or slightly worse than causality-inspired methods. This is because, to ensure high accuracy, the estimated explanatory subgraph of these methods should have a large size.
This can be reflected by the explanation recall, where 
explanation recall is significantly reduced. Overall, the causality-inspired methods obtain higher accuracies than purely association-based methods. 

More importantly, {\name} drastically outperforms all the compared GNN explainers in terms of groundtruth match. Such a big difference demonstrates all the association-based and causality-inspired GNN explainers are insufficient to uncover the exact groundtruth. 
The is due to existing GNN explainers inherently learning from \emph{correlations} among nodes/edges in the graph, and capturing spurious correlations. 
Instead, our causal explainer 
 {\name} can do so much more accurately. This verifies the causal explainer indeed can intrinsically uncover the causal relation between the explanatory subgraph and the graph label.

\begin{figure*}[!t]
  \centering
  \begin{subfigure}[b]{0.158\textwidth}
    \includegraphics[width=\linewidth]{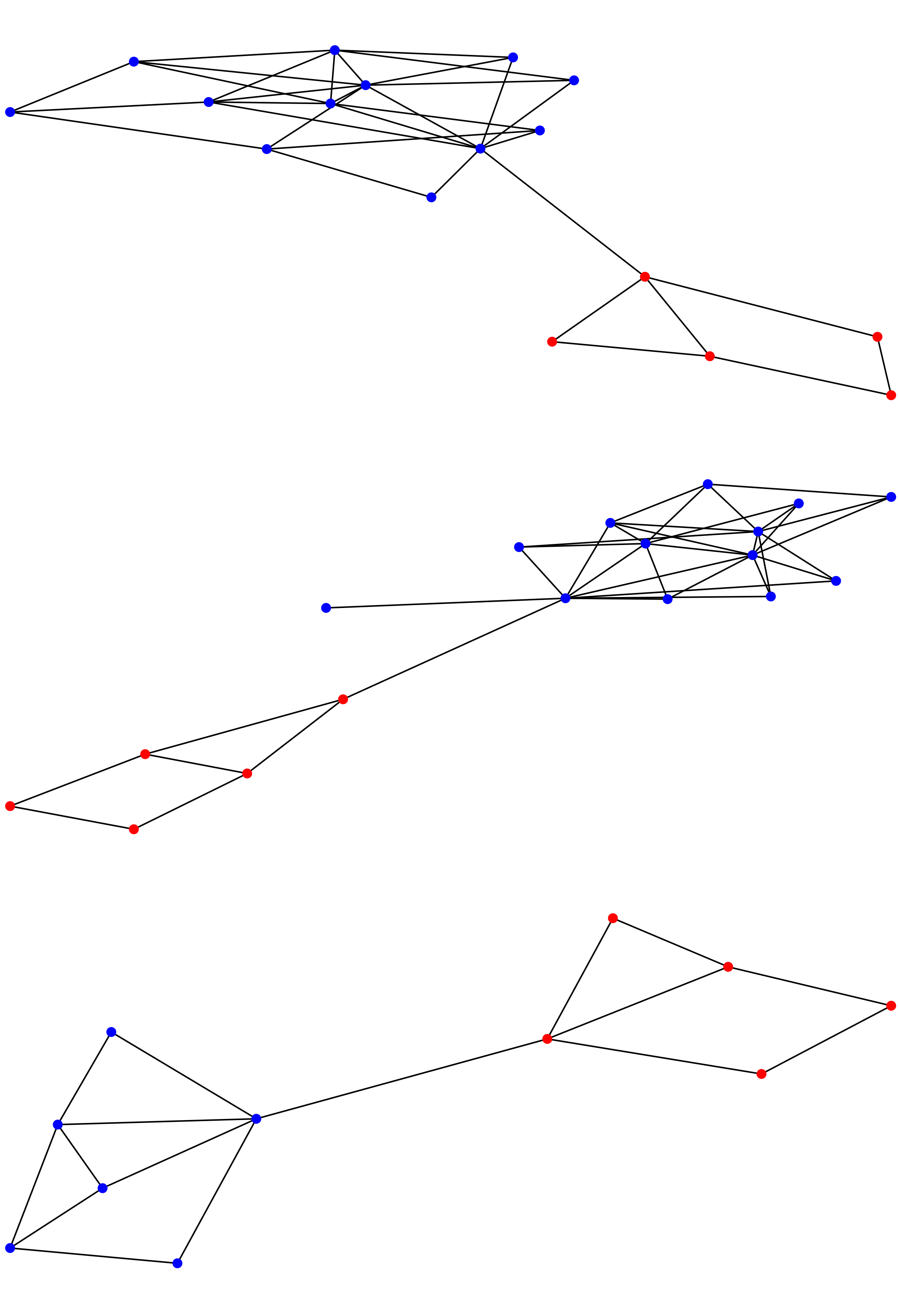}
    \caption{BA+House}
  \end{subfigure}
  \begin{subfigure}[b]{0.158\textwidth}
    \includegraphics[width=\linewidth]{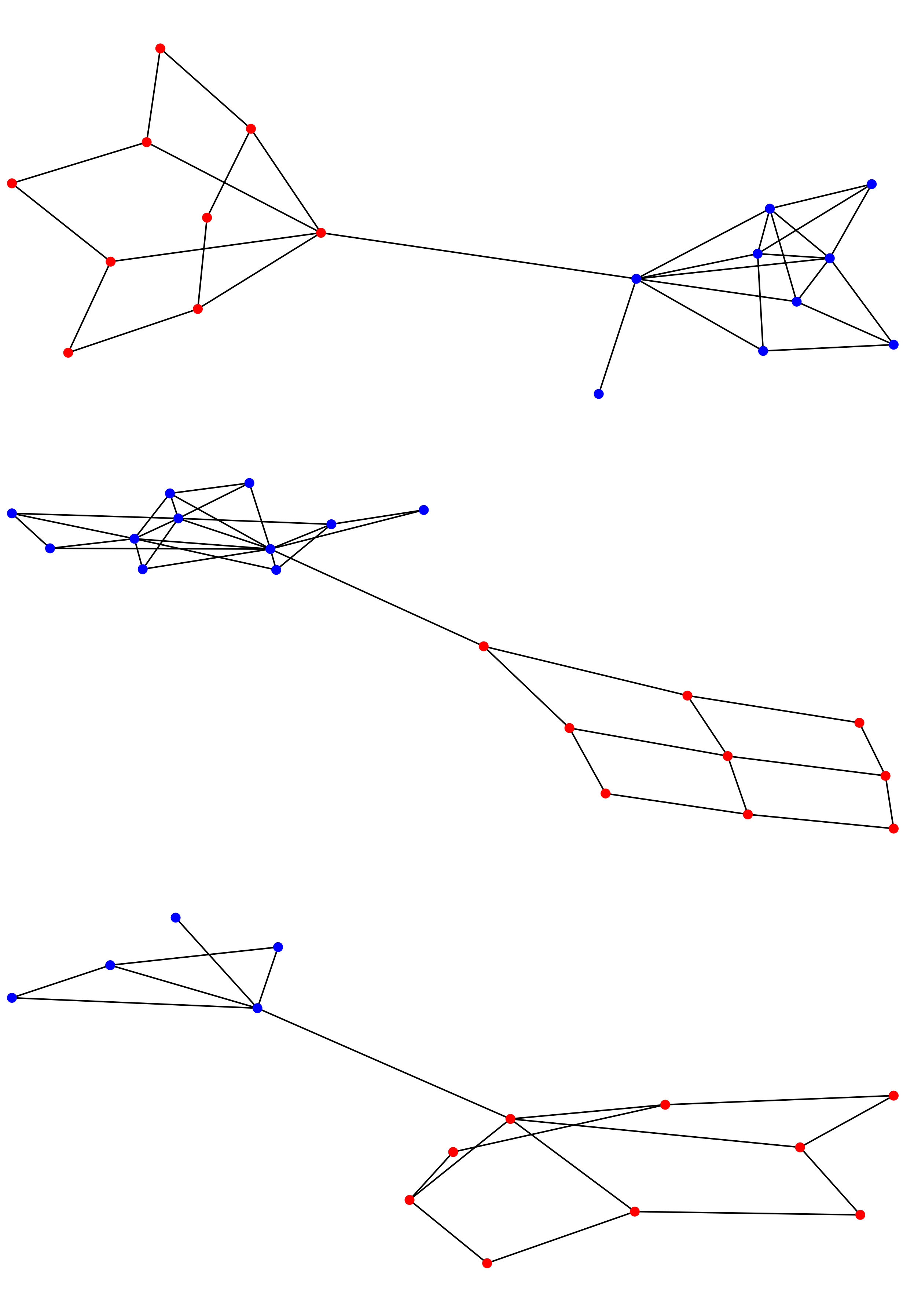}
    \caption{BA+Grid}
  \end{subfigure}
    \begin{subfigure}[b]{0.158\textwidth}
    \includegraphics[width=\linewidth]{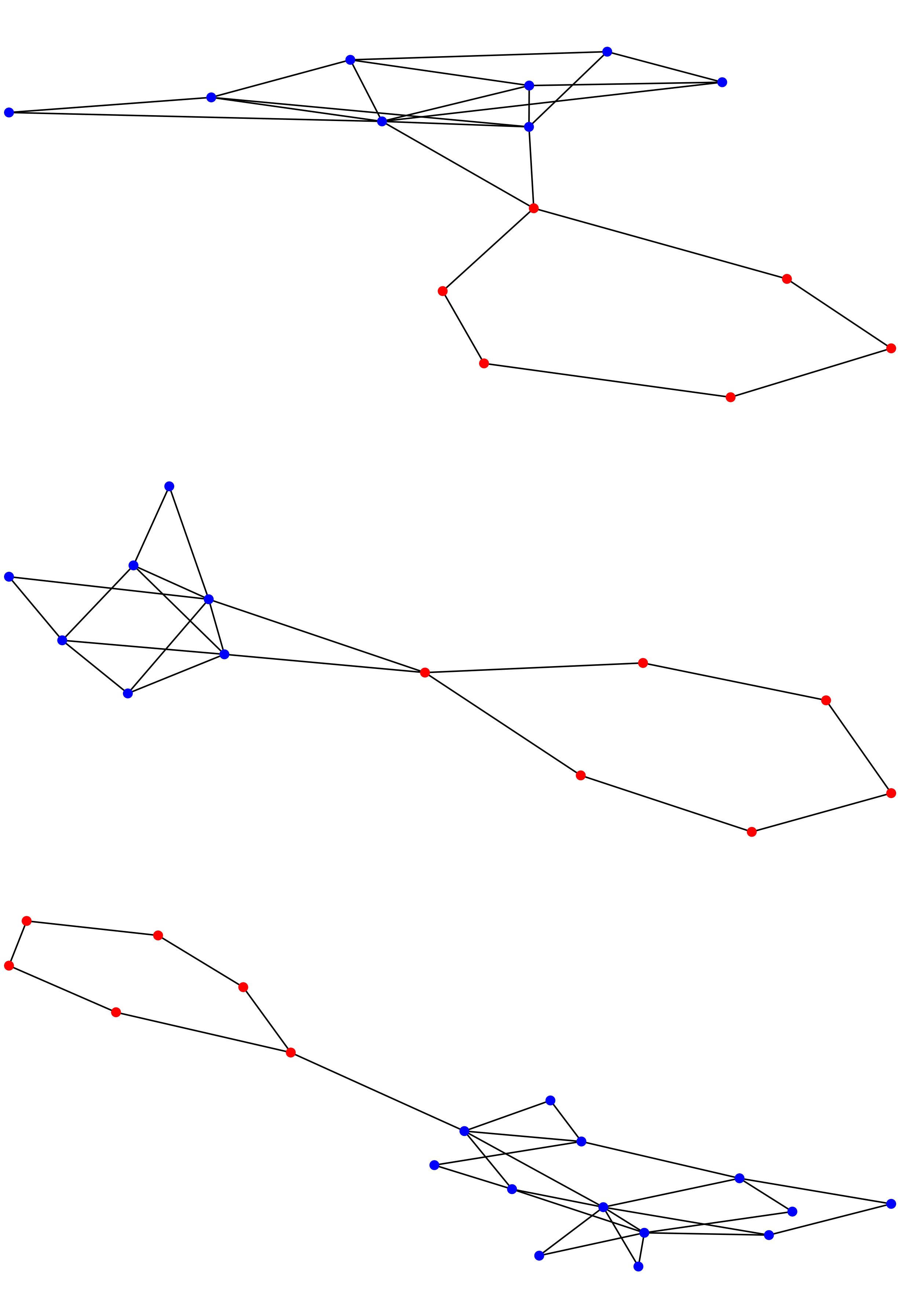}
    \caption{BA+Cycle}
  \end{subfigure}
  \begin{subfigure}[b]{0.165\textwidth}
    \includegraphics[width=\linewidth]{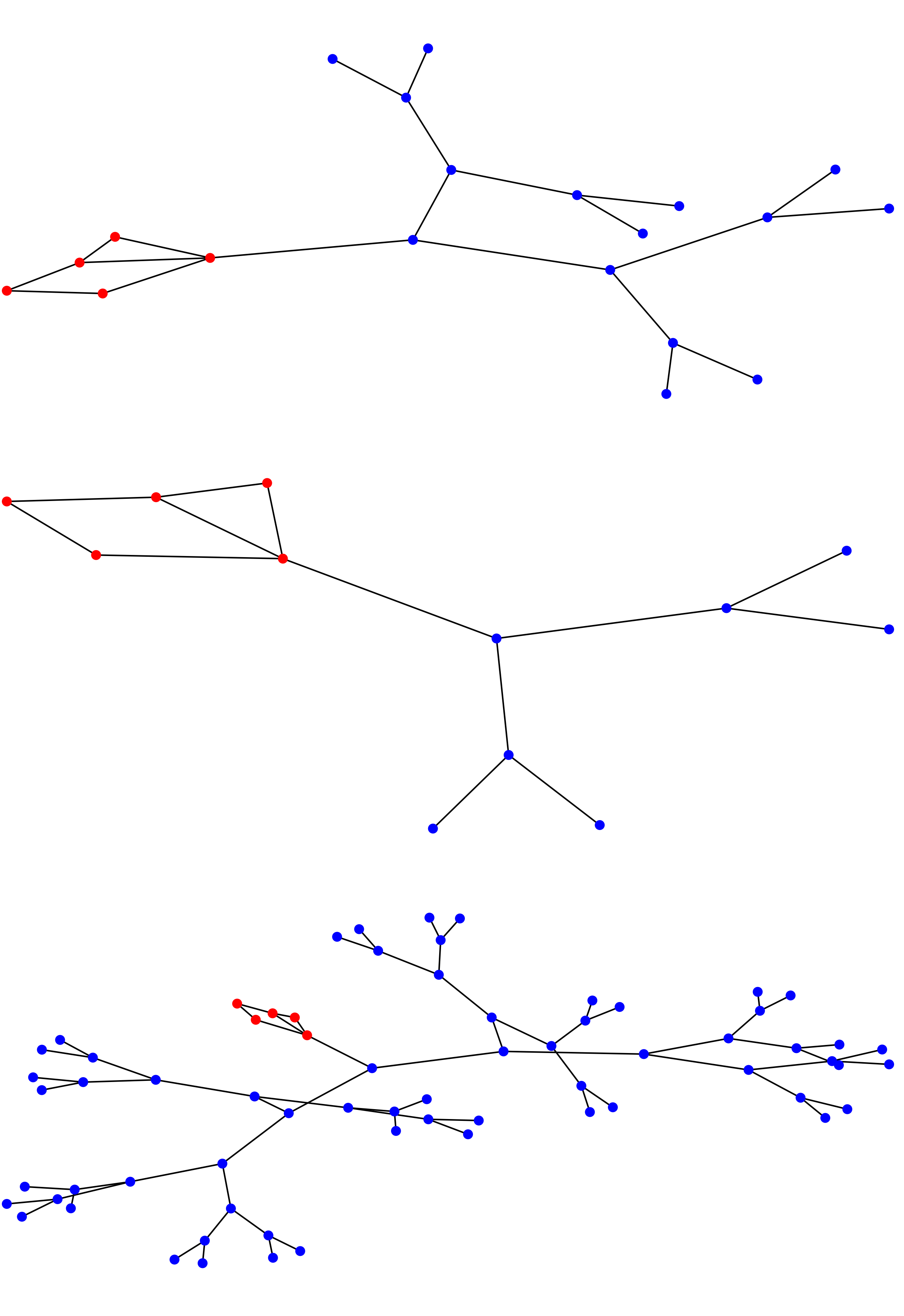}
    \caption{Tree+House}
  \end{subfigure}
  \begin{subfigure}[b]{0.158\textwidth}
    \includegraphics[width=\linewidth]{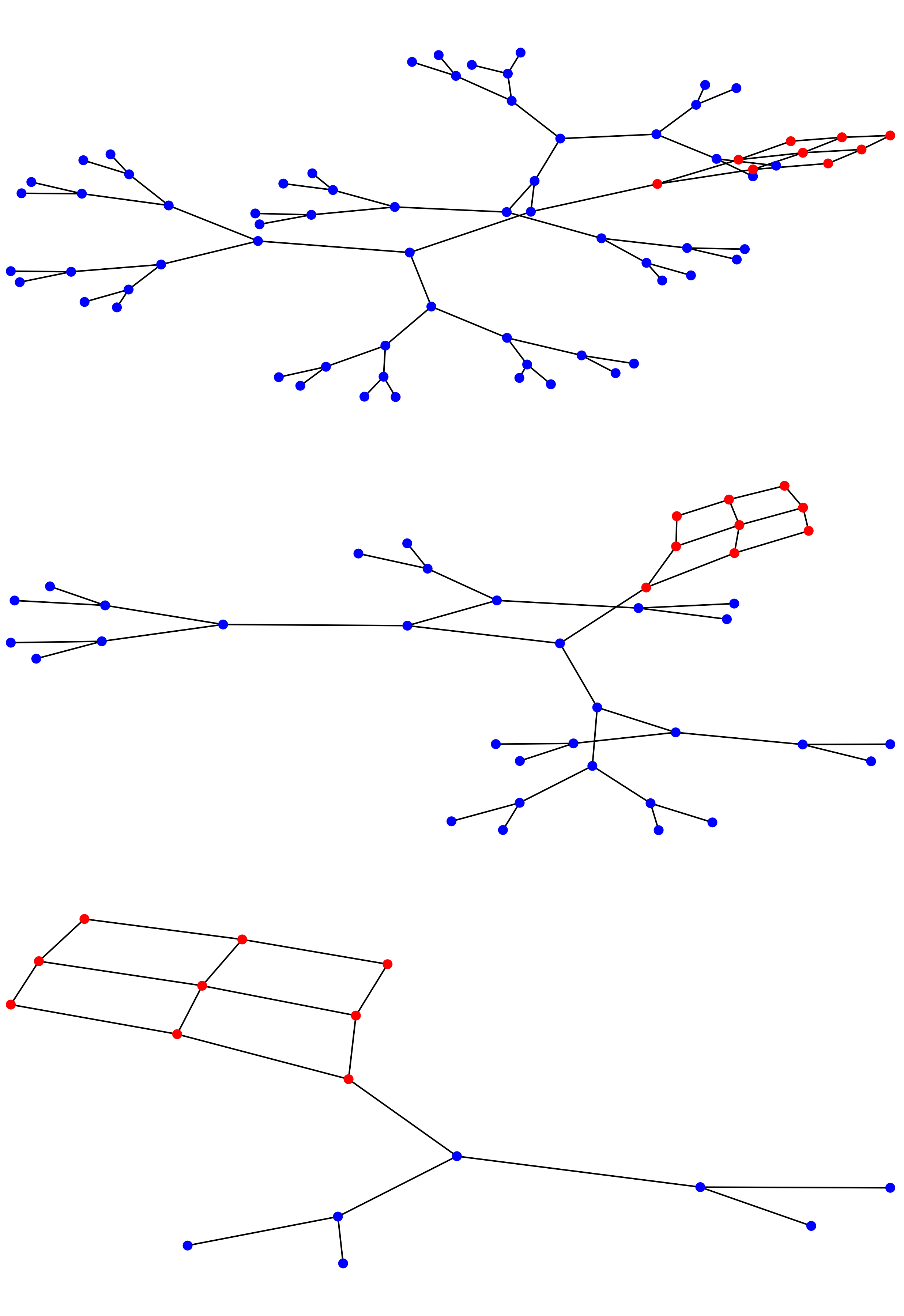}
    \caption{Tree+Grid}
  \end{subfigure}
  \begin{subfigure}[b]{0.158\textwidth}
    \includegraphics[width=\linewidth]{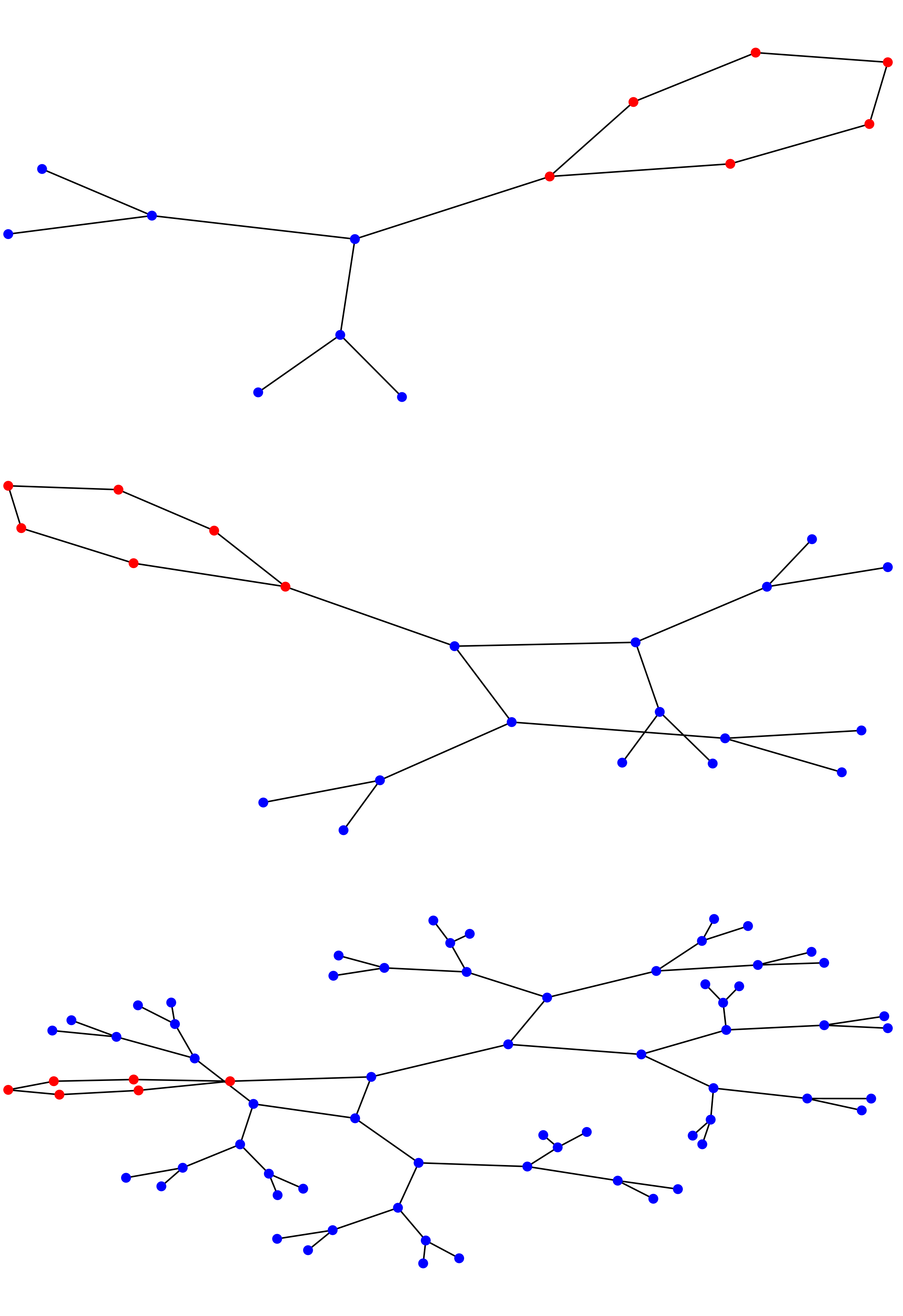}
    \caption{Tree+Cycle}
  \end{subfigure}
  \caption{Visualizing explanation results (subgraph containing the {\color{red} red} nodes) by our {\name} on synthetic graphs.
  }
\label{fig:vis_syn}
\end{figure*}

\begin{figure}[!t]
  \centering
    \begin{subfigure}{0.45\linewidth}
      \includegraphics[width=\linewidth]{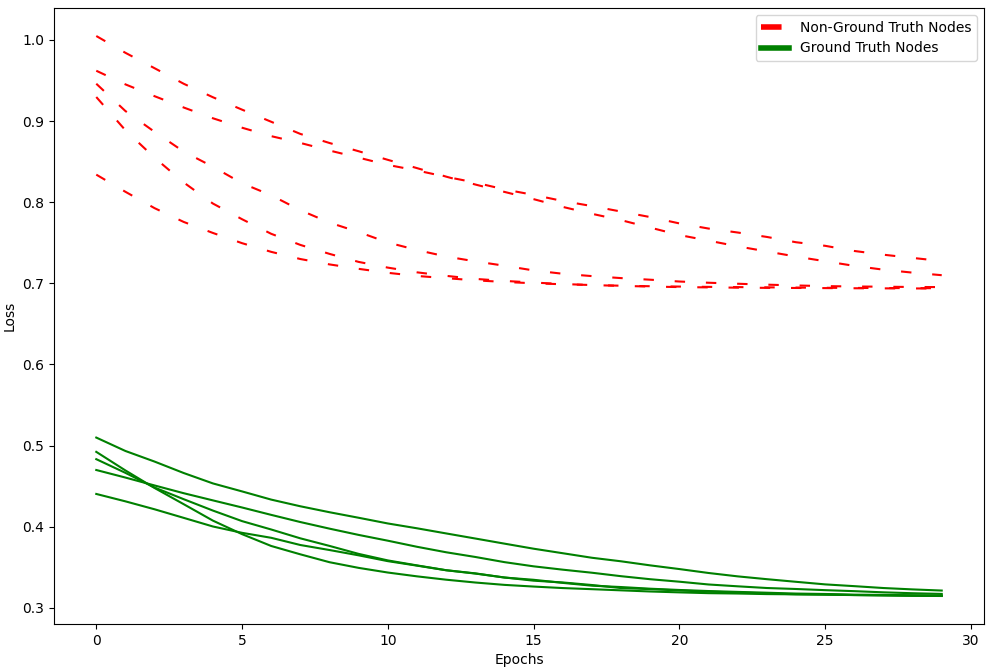}
    \end{subfigure}
    \begin{subfigure}{0.45\linewidth}
      \includegraphics[width=\linewidth]{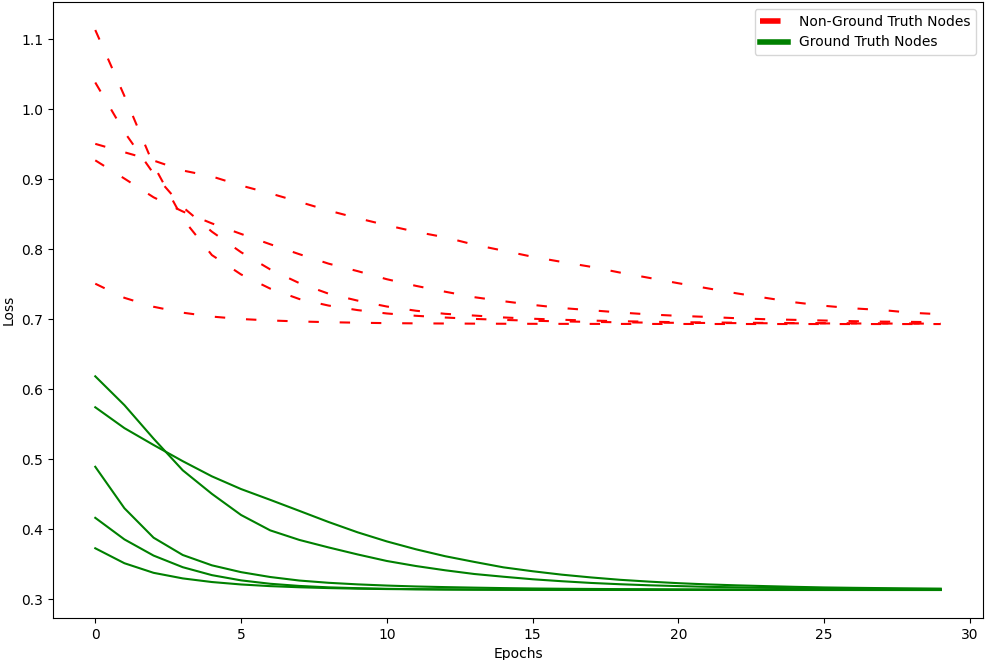}
    \end{subfigure}
    \caption{Loss curves of training the GNN-NCMs on the groundtruth nodes ({\color{green} green curves}) and non-groundtruth ones ({\color{red} red curves}) on two random chosen graphs from BA+House. More examples in other datasets are shown in Appendix~\ref{app:exp}.}
        \label{fig:loss_syn}
\end{figure}

\begin{figure}[!t]
  \centering
  \begin{subfigure}[b]{0.45\textwidth}
    \includegraphics[width=\linewidth]{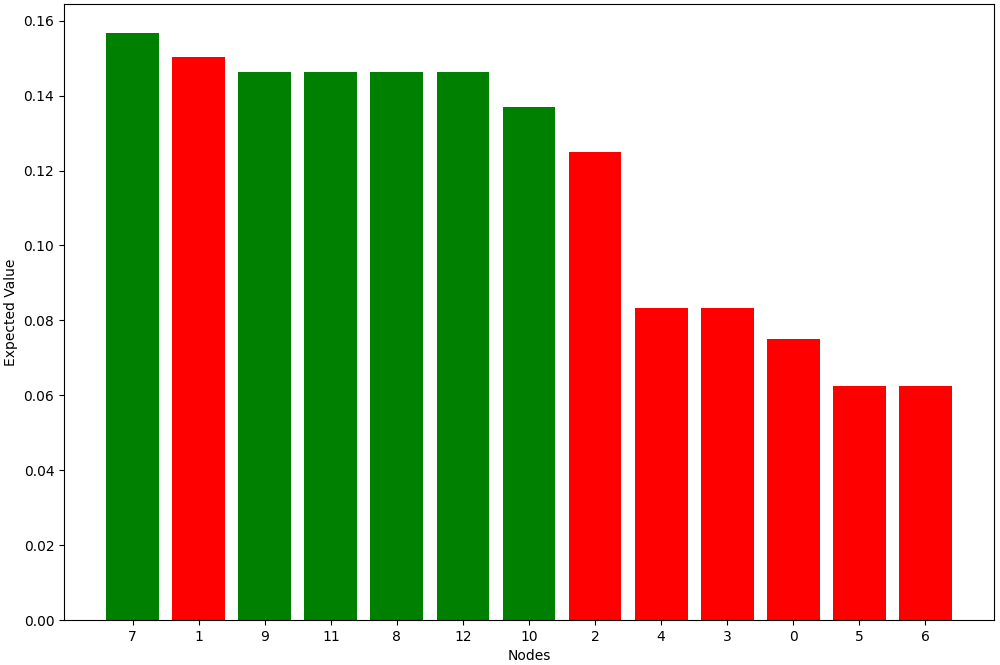}
  \end{subfigure}
  \begin{subfigure}[b]{0.45\textwidth}
    \includegraphics[width=\linewidth]{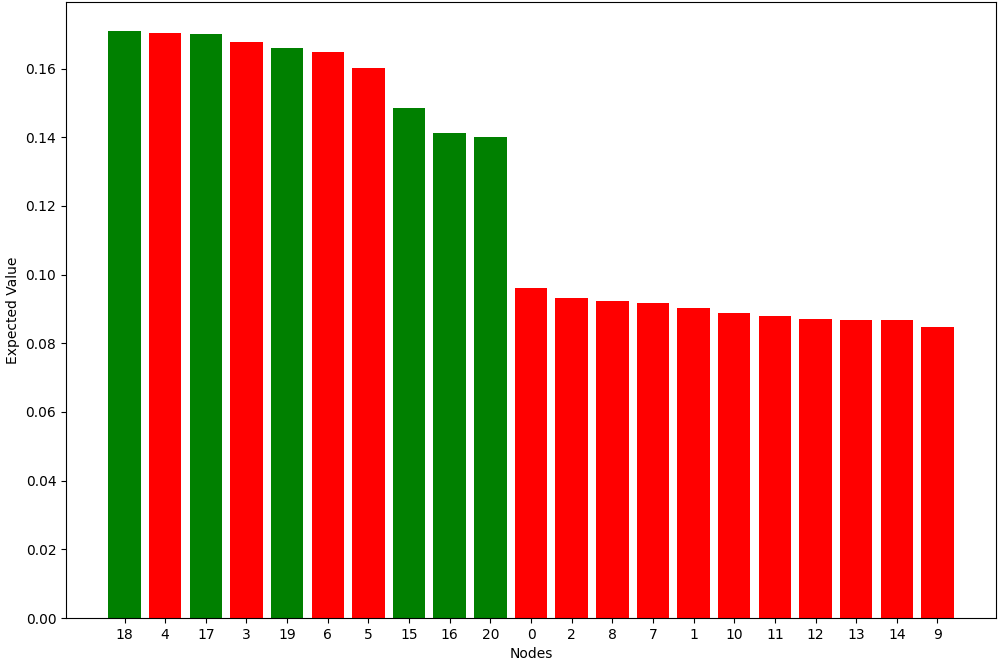}
  \end{subfigure}
   \caption{Node expressivity distributions on two unsuccessful graphs from BA+Cycle. {\color{green} Green bars} correspond to nodes that are in the groundtruth, while {\color{red} red bars} correspond to nodes that are not. More examples in other datasets are shown in Appendix~\ref{app:exp}.}
    \label{fig:exp_dist}
\end{figure}

\noindent {\bf Visualization results:} Figure~\ref{fig:vis_syn} visualizes the explanations results of some testing graphs in the four synthetic datasets. We note that there are different ways for the groundtruth subgraph to attach to the base synthetic graph. We can see {\name}'s output exactly matches the groundtruth in these cases, while the existing GNN explainers cannot. One reason could be that existing GNN explainers are sensitive to the spurious relation. 

\noindent {\bf Loss curve:} 
Figure~\ref{fig:loss_syn} shows the loss curves to train our GNN-NCM on a set of nodes, where some nodes are in the groundtruth and some are not from BA+House. We can see the loss decreases stably for groundtruth nodes, while the loss for nodes not from the groundtruth are relatively high. This reflects our designed GNN-NCM makes it easier to learn groundtruth nodes. That being said, {\name} indeed tends to find the causal subgraph.    

\noindent {\bf Node expressivity distribution:} We notice {\name} still misses finding the groundtruth explanatory subgraph for some graphs. 
One possible reason could be that, theoretically, our GNN-SCM can always uncover the causal subgraph, but practically, it is challenging to train the optimal one.  
Here, we randomly select 2 such unsuccessful graphs in BA+Cycle and plot their distributions on the node expressivity in Figure~\ref{fig:exp_dist}. 
We observe that, though the groundtruth nodes are not always having the best expressivity, they are still at the top. 

\begin{table}[!t]\renewcommand{\arraystretch}{0.8}
\centering
\footnotesize
\caption{Comprehensive comparison results on the real-world datasets.}
\begin{tabular}{lcccccc}
\noalign{\smallskip}
\toprule
& \multicolumn{2}{c}{\textbf{Exp. Acc. (\%)}} & \multicolumn{2}{c}{\textbf{Exp. Recall (\%)}} & \multicolumn{2}{c}{\textbf{GT Match Acc. (\%)}} \\
\cmidrule(r){2-3} \cmidrule(r){4-5} \cmidrule(r){6-7}
\textbf{Method} & \textbf{Benzene} & \textbf{F.C.} & \textbf{Benzene} & \textbf{F.C.} & \textbf{Benzene} & \textbf{F.C.} \\
\midrule
{\bf GNNExp.~\cite{GNNEx19}} & 66.05 & 44.44 & 18.88 & 14.42 & 0.00 & 0.00 \\
{\bf PGMExp.~\cite{vu2020pgm}} & 33.33 & 17.78 & 7.51 & 4.98 & 0.00 & 0.00 \\
{\bf Guidedbp~\cite{gu2019saliency}} & 0.00 & 0.00 & 9.06 & 8.00 & 0.00 & 0.00 \\
{\bf GEM~\cite{lin2021generative}} & 71.98 & 46.22 & 19.80 & 14.57 & 0.00 & 0.00 \\
{\bf RCExp.~\cite{wang2022reinforced}} & 0.20 & 0.05 & 10.85 & 2.01 & 0.00 & 0.00 \\
{\bf OrphicX~\cite{lin2022orphicx}} & 47.63 & 11.14 & 30.31 & 10.01 & 3.40 & 5.50 \\
\textbf{\name} & \textbf{73.46} & \textbf{66.67} & \textbf{21.35} & \textbf{16.43} & \textbf{66.67} & \textbf{75.00} \\
\bottomrule
\end{tabular}
\label{tab:real_comparison}
\end{table}

\begin{figure}[!t]
  \centering
  \begin{subfigure}[b]{.45\textwidth}
    \begin{subfigure}[b]{0.45\textwidth}
    \includegraphics[width=\linewidth]{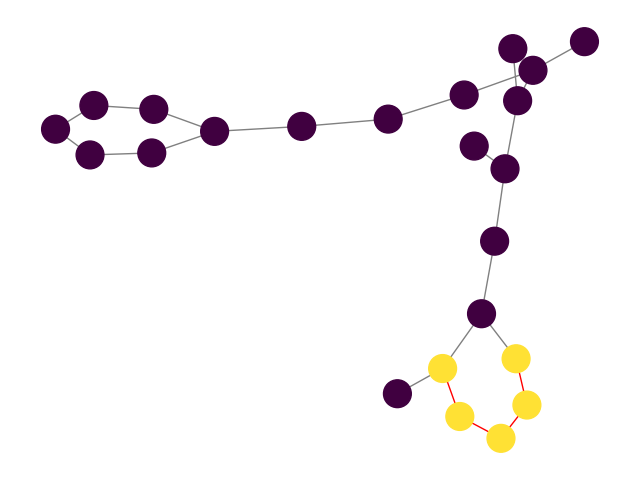}
    \end{subfigure}
    \begin{subfigure}[b]{0.45\textwidth}
    \includegraphics[width=\linewidth]{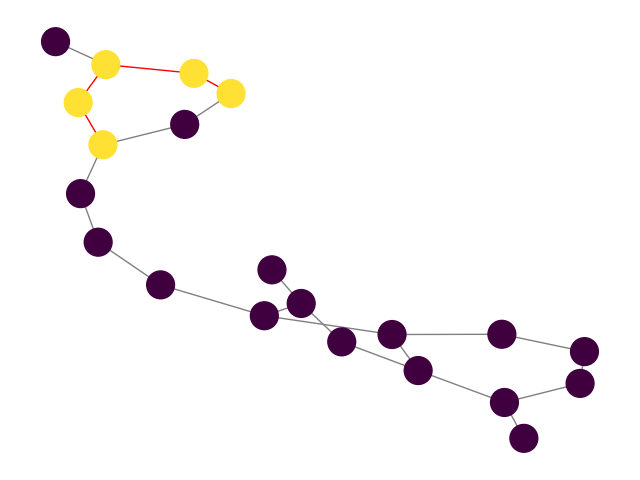}
    \end{subfigure}
  \end{subfigure}
  \begin{subfigure}[b]{.45\textwidth}
  \begin{subfigure}[b]{0.45\textwidth}
    \includegraphics[width=\linewidth]{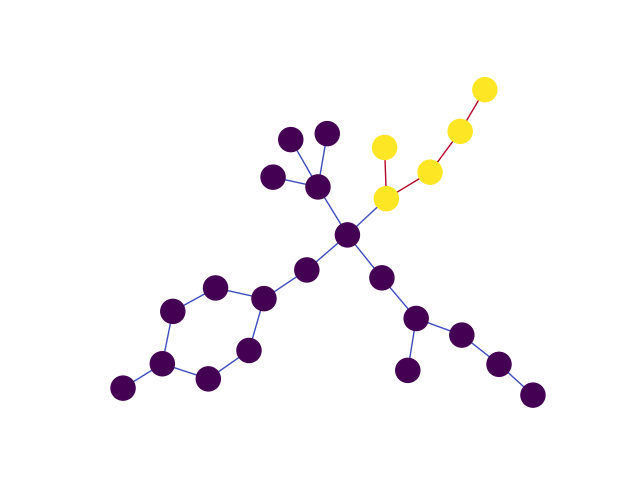}
    \end{subfigure}
    \begin{subfigure}[b]{0.45\textwidth}
    \includegraphics[width=\linewidth]{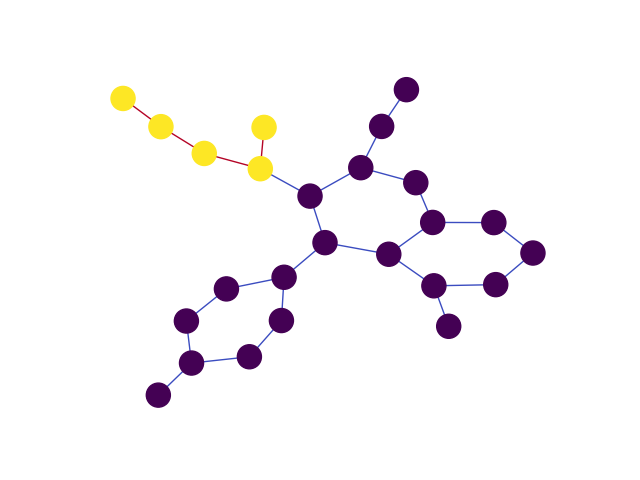}
    \end{subfigure}
  \end{subfigure}
  \caption{Explanation results (subgraph containing the {\color{yellow} yellow} nodes) by our {\name} on real-world graphs. The left and right two graphs are in Benzene and F.C., respectively. 
  }
\label{fig:vis_real}
\end{figure}

\begin{figure}[!t]
  \centering
    \begin{subfigure}{0.45\linewidth}
      \includegraphics[width=\linewidth]{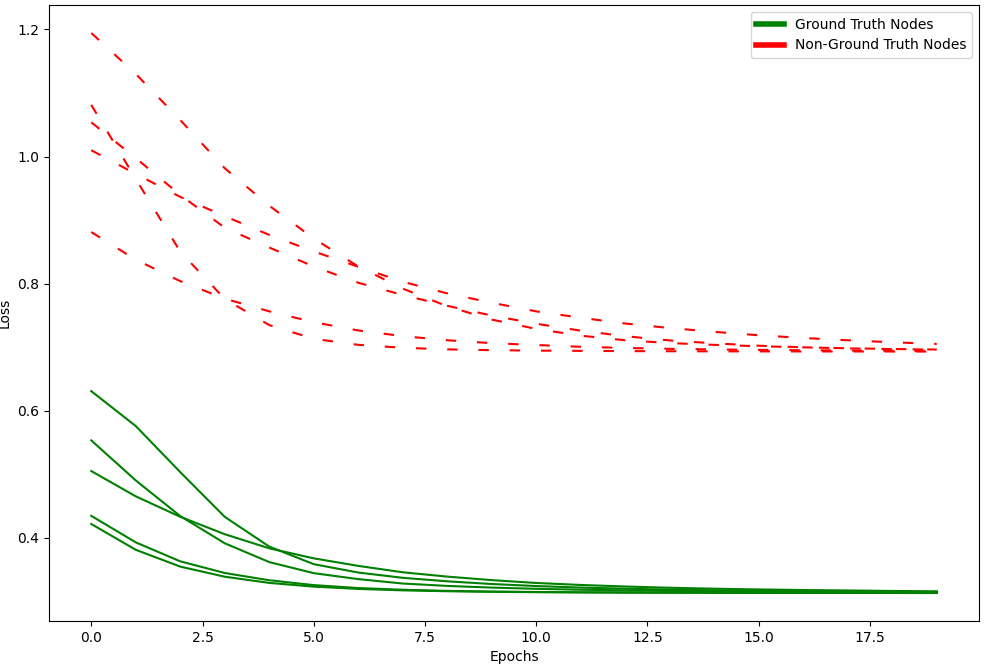}
    \end{subfigure}
    \begin{subfigure}{0.45\linewidth}
      \includegraphics[width=\linewidth]{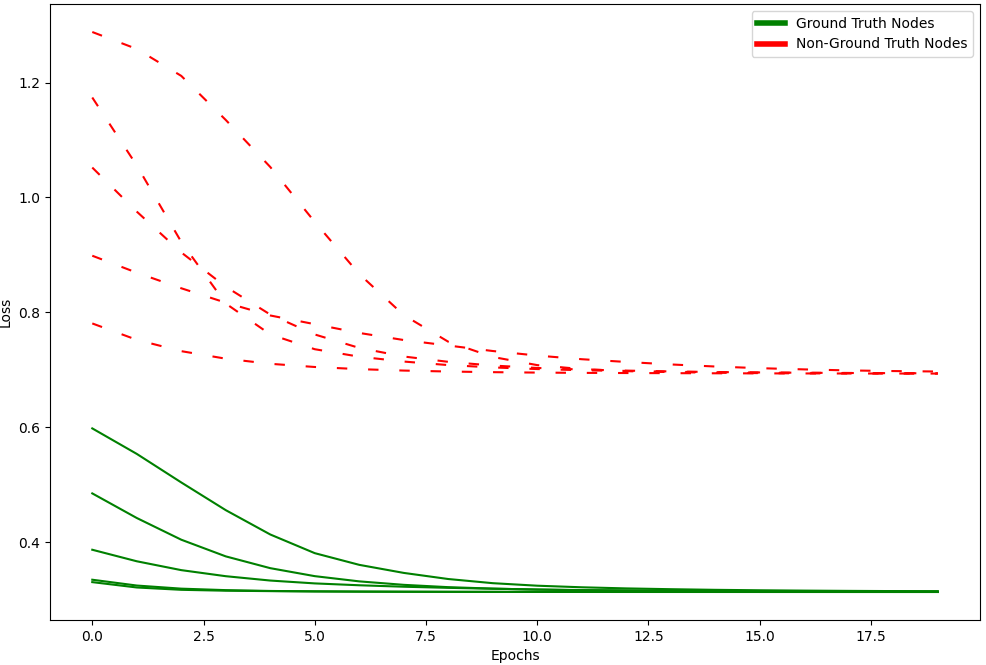}
    \end{subfigure}
    \caption{Loss curves of training the GNN-NCMs on the groundtruth nodes ({\color{green} green curves}) and non-groundtruth ones ({\color{red} red curves}) on two graphs from the two real-world datasets, respectively. More examples are shown in Appendix~\ref{app:exp}.}
        \label{fig:loss_real}
\end{figure}

\begin{figure}[!t]
  \centering
  \begin{subfigure}[b]{0.45\textwidth}
    \includegraphics[width=\linewidth]{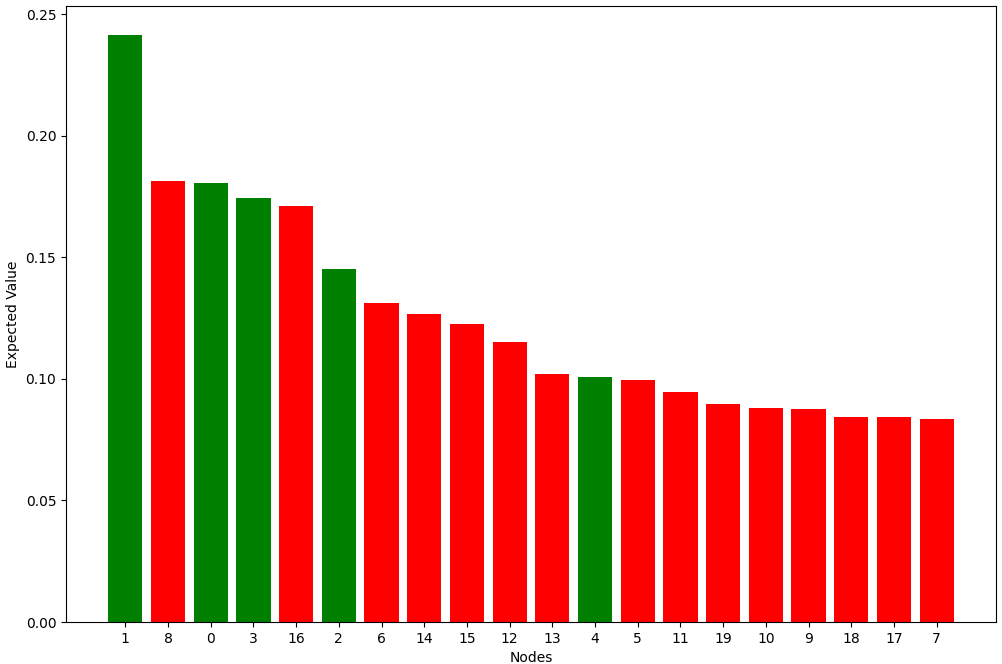}
  \end{subfigure}
  \begin{subfigure}[b]{0.45\textwidth}
    \includegraphics[width=\linewidth]{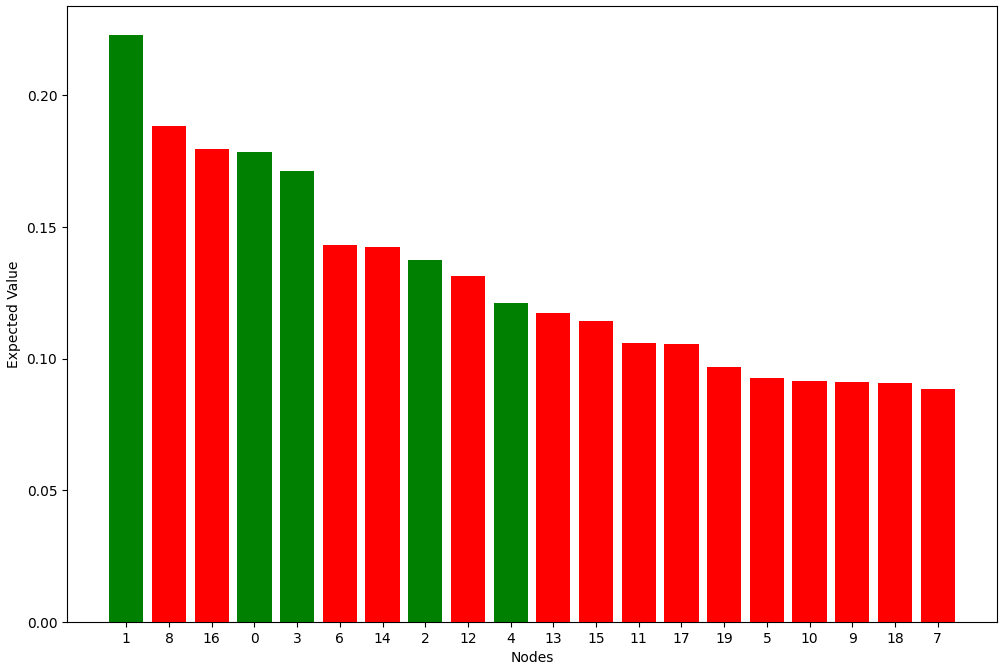}
  \end{subfigure}
   \caption{Node expressivity distributions on two unsuccessful graphs from the real-world datasets, respectively. {\color{green} Green bars} correspond to nodes that are in the groundtruth, while {\color{red} red bars} correspond to nodes that are not. More examples are in Appendix~\ref{app:exp}.}
    \label{fig:exp_dist_real}
\end{figure}

\subsection{Results on Real-World Datasets}

\noindent {\bf Comparison results:}
Table~\ref{tab:real_comparison} shows the results of all the compared explainers on the real-world datasets and three metrics. We have similar observations as those in Table~\ref{tab:synthetic_comparison}. 
Especially, no existing explainers can even find one exactly matched groundtruth. Particularly, the explanation subgraphs produced by the two causality-inspired baselines can cover the majority or almost all groundtruth in synthetic datasets (hence high accuracy), and the sizes of the explanation subgraphs are slightly larger than those of the groundtruth (hence relatively large recall). However, the causality-inspired baselines are not good at exactly matching the groundtruth, i.e., {groundtruth match accuracy is low overall}. Note also that the exact match of {\name} is also largely reduced (about $30\%$). One possible reason is that the groundtruth explanation in real-world graphs is not easy to define or even inaccurate. For instance, in MUTAG, both $NO_2$ and $NH_2$ motifs are considered as the "mutagenic" groundtruth in the literature. However, \cite{lin2021generative} found 32\% of non-mutagenic graphs contain $NO_2$ or $NH_2$, implying inaccurate groundtruth.
Here, we propose to also use an approximate groundtruth match accuracy, where we require the estimated subgraph to be a subset and its size is no less than 60\% of the groundtruth. With this new alternative metric, its value is much larger (i.e., 67\% and 75\%) on the two datasets. 

\noindent {\bf Visualization results:} Figure~\ref{fig:vis_real} visualizes the explanation results of some graphs in the real-world datasets. We observe the explanatory subgraphs found by {\name} approximately/exactly match the groundtruth.

\noindent {\bf Loss curve:} Figure~\ref{fig:loss_real} shows the loss curves to train our GNN-NCM on a set of groundtruth and non-groundtruth ones. Similarly, the loss decreases stably for groundtruth nodes, while not for non-groundtruth ones. Again, this implies our GNN-NCM tends to find the causal subgraph.     

\noindent {\bf Node expressivity distribution:}
We randomly select some unsuccessful graphs in real-world datasets and plot their distributions on the node expressivity in Figure~\ref{fig:exp_dist_real}. Still, though the groundtruth nodes do not always achieve the best expressivity, they are at the top.

\section{Conclusion}
GNN explanation, i.e., identifying the informative subgraph that ensures a GNN makes a particular prediction for a graph, is an important research problem. Though various GNN explainers have been proposed, they are shown to be prone to spurious correlations. 
We propose a \emph{causal} GNN explainer based on the fact that a graph often consists of a causal subgraph and fulfills the goal via causal inference. We then propose to train GNN neural causal models to uncover the causal explanatory subgraph. In future work, we will study the robustness of our CXGNN under the adversarial graph perturbation attacks~\cite{wang2019attacking,mu2021a,wang2022bandits,wang2023turning,wang2021certified,yanggnncert,li2024graph}. 

\section*{Acknowledgements}
We thank all the anonymous reviewers for their valuable feedback and constructive comments. Behnam and Wang are partially supported by the Amazon Research Award, the National Science Foundation (NSF) under Grant Nos. 2216926, 2241713, 2331302, and 2339686.  
Any opinions, findings conclusions, and recommendations expressed in this material are those of the author(s) and do not necessarily reflect the views of the funding agencies.

%
%
\bibliographystyle{splncs04}
\bibliography{main}
\newpage 
\appendix
\section*{Appendix}

\section{A GNN-SCM Example} \label{app:scm-example}

The causal structure of a graph is a subgraph centering on a reference node and accepts the SCM structure via Definition~\ref{causal network}. The goal of causality in a graph is to identify the subgraph with the maximum explainable node expressivity as Theorem~\ref{thm:nodeexp} that causally explains GNN predictions.
Below, we use a toy example graph to show how our explainer captures the causality in this graph.

\begin{wrapfigure}{r}{0.3\textwidth}
  \centering
    \begin{tikzpicture}
      \node[circle, draw] (A) at (0,0) {A};
      \node[circle, draw] (B) at (2,0) {B};
      \node[circle, draw] (C) at (1,-1) {C};
      \node[circle, draw] (D) at (3,0) {D};
      \draw[-] (A) -- (B);
      \draw[-] (A) -- (C);
      \draw[-] (B) -- (D);
    \end{tikzpicture}
  \caption{A toy example}
  \label{fig2}
\end{wrapfigure}
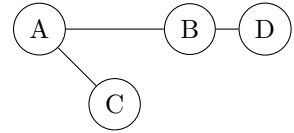
We use a toy example to demonstrate SCMs and the intervention process in GNNs. Figure \ref{fig2} shows a graph $G$ that contains four nodes $A$, $B$, $C$, and $D$, and three edges $A$-$B$,  $A$-$C$, and $B$-$D$. In GNNs, these edges contain messages passing between two nodes. For example, the message between two nodes A and B in the l-th layer of the GNN is $m^l_{A,B}$ = MSG($h^{l-1}_A$, $h^{l-1}_B$, $e_{A,B})$. If there exists an edge, it means that there is an interaction between nodes that has a specific value in each layer $l$. If there is no edge, two nodes don't share a message. If we consider node $A$ as the reference node $v$, the nodes $B$, and $C$ are in the 1-hop neighbors $\mathcal{N}_{\leq 1}(v)$, and node $D$ is in the 2-hop neighbors $\mathcal{N}_{\leq 2}(v)$. 
This GNN-SCM induces its causal structure $\mathcal{G}$ from Graph G, as discussed in Definition \ref{causal network}.

\subsection{GNN-SCM construction}
Following~\cite{pearl2009causality}, 
we build a GNN-SCM $\mathcal{M}(\mathcal{G})$ that learns from the causal structure $\mathcal{G}$. The endogenous variables are node labels $\{y_v; v \in A,B,C,D\}$. The exogenous variables in $\mathcal{G}$ are reference node $A$'s states: $u_{A_1}, u_{A_2}$ (in this example we consider  binary states $A_1$ and $A_2$), edges effects on reference node $A$: $u_{A,B}, u_{A,C}$, and neighbor nodes' effects on reference node $A$: $u_{B}, u_{C}, u_{D}$. All of these latent variables are assumed to accept the same probability $P({\bf U})$ as a probability function defined over the domain of U since we don't want to input new specific information. $\mathcal{F}$ is a set of functions based on the observable and latent variables discussed above. One should consider that $u_{v_{i}}$ and $u_{v_{j}}$ are not independent.

As discussed in Theorem~\ref{thm:GNN-SCM}, we construct the GNN-SCM $\mathcal{M}(\mathcal{G})$ based on graph G as $\mathcal{M}(\mathcal{G}) =: (U,V,F, P({\bf U}))$, where:

{
\footnotesize
\begin{align}\label{eq:example SCM}
    \mathcal{M}(\mathcal{G}) = 
        \begin{cases}
            U := \begin{cases}u_{A_1}, u_{A_2}\\
            u_{A,B}, u_{A,C}\\
            u_{B}, u_{C}, u_{D}\end{cases} , {D_{u}} = \{0,1\}\\
            V := \{A,B,C,D\}\\
            \mathcal{F} := \begin{cases}
            f_A(B, C, D, u_{A_1},u_{A_2}) = f_A(B,u_{A_1},u_{A_2}) \land f_A(C,u_{A_1},u_{A_2}) \land f_A(D)\\
            f_A(B,u_{A_1},u_{A_2}) = (((\neg B \oplus U_{A_1}) \lor U_{A,B}) \oplus u_{A_2})\\
            f_A(C,u_{A_1},U_{A_2}) = (((\neg C \oplus U_{A_1}) \lor U_{A,C}) \oplus u_{A_2})\\
            f_A(D) = \neg D\\
            f_B(u_{B}, u_{A,B}) = \neg u_{B} \land \neg u_{A,B}\\
            f_C(u_{C}, u_{A,C}) = \neg u_{C} \land \neg u_{A,C}\\
            f_D(u_{D}) = \neg u_{D} \end{cases}\\
            P({\bf U}) := \begin{cases}P(u_{A_1}) = P(u_{A_2}) = P(u_{A,B}) = P(u_{A,C})\\ = P(u_{A,D}) = P(u_{B}) = P(u_{C}) = P(u_{D})\end{cases} = 1/8\\
        \end{cases}
\end{align}
}%

\begin{center}
\begin{figure}[h]
\centering
\begin{tikzpicture}[scale=0.8]
[level distance=15mm,
   every node/.style={fill=white,rectangle,align=center},
   level 1/.style={sibling distance=60mm},
   level 2/.style={sibling distance=50mm},
   level 3/.style={sibling distance=40mm},
   level 4/.style={sibling distance=30cm},
   level 5/.style={sibling distance=20cm}]
   \node {A, B, C, P({\bf U})}
    child {node{$U_{B}=0$}
        child {node{$U_{C}=0$}
            child {node{$U_{A,B}=0$}
                child {node{$U_{A,C}=0$}
                    child {node {$U_{A_1}=0$}
                        child {node{$U_{A_2}=0$}}
                        child {node{$U_{A_2}=1$}}
                    }
                    child {node {$U_{A_1}=1$}
                        child {node{$U_{A_2}=0$}}
                        child {node{$U_{A_2}=1$}}
                    }
                }
                child {node{$U_{A,C}=1$}
                    child {node{$U_{A_1}=0$}
                        child {node{$U_{A_2}=0$}}
                        child {node{$U_{A_2}=1$}}
                    }
                    child {node{$U_{A_1}=1$}
                        child {node{$U_{A_2}=0$}}
                        child {node{$U_{A_2}=1$}}
                    }
                }
            }
            child {node{$U_{A,B}=1$}
                child {node{$U_{A,C}=0$}
                    child {node{$U_{A_1}=0$}
                        child {node{$U_{A_2}=0$}}
                        child {node{$U_{A_2}=1$}}
                    }
                    child {node{$U_{A_1}=1$}
                        child {node{$U_{A_2}=0$}}
                        child {node{$U_{A_2}=1$}}
                    }
                }
                child {node{$U_{A,C}=1$}
                    child {node{$U_{A_1}=0$}
                        child {node{$U_{A_2}=0$}}
                        child {node{$U_{A_2}=1$}}
                    }
                    child {node{$U_{A_1}=1$}
                        child {node{$U_{A_2}=0$}}
                        child {node{$U_{A_2}=1$}}
                    }
                }
            }
        }
        child {node{$U_{C}=1$}
            child {node{$U_{A,B}=0$}
                child {node{$U_{A,C}=0$}
                    child {node {$U_{A_1}=0$}
                        child {node{$U_{A_2}=0$}}
                        child {node{$U_{A_2}=1$}}
                    }
                    child {node {$U_{A_1}=1$}
                        child {node{$U_{A_2}=0$}}
                        child {node{$U_{A_2}=1$}}
                    }
                }
                child {node{$U_{A,C}=1$}
                    child {node{$U_{A_1}=0$}
                        child {node{$U_{A_2}=0$}}
                        child {node{$U_{A_2}=1$}}
                    }
                    child {node{$U_{A_1}=1$}
                        child {node{$U_{A_2}=0$}}
                        child {node{$U_{A_2}=1$}}
                    }
                }
            }
            child {node{$U_{A,B}=1$}
                child {node{$U_{A,C}=0$}
                    child {node{$U_{A_1}=0$}
                        child {node{$U_{A_2}=0$}}
                        child {node{$U_{A_2}=1$}}
                    }
                    child {node{$U_{A_1}=1$}
                        child {node{$U_{A_2}=0$}}
                        child {node{$U_{A_2}=1$}}
                    }
                }
                child {node{$U_{A,C}=1$}
                    child {node{$U_{A_1}=0$}
                        child {node{$U_{A_2}=0$}}
                        child {node{$U_{A_2}=1$}}
                    }
                    child {node{$U_{A_1}=1$}
                        child {node{$U_{A_2}=0$}}
                        child {node{$U_{A_2}=1$}}
                    }
                }
            }
        }
    }
    child {node{$U_{B}=1$}
        child {node{$U_{C}=0$}
            child {node{$U_{A,B}=0$}
                child {node{$U_{A,C}=0$}
                    child {node {$U_{A_1}=0$}
                        child {node{$U_{A_2}=0$}}
                        child {node{$U_{A_2}=1$}}
                    }
                    child {node {$U_{A_1}=1$}
                        child {node{$U_{A_2}=0$}}
                        child {node{$U_{A_2}=1$}}
                    }
                }
                child {node{$U_{A,C}=1$}
                    child {node{$U_{A_1}=0$}
                        child {node{$U_{A_2}=0$}}
                        child {node{$U_{A_2}=1$}}
                    }
                    child {node{$U_{A_1}=1$}
                        child {node{$U_{A_2}=0$}}
                        child {node{$U_{A_2}=1$}}
                    }
                }
            }
            child {node{$U_{A,B}=1$}
                child {node{$U_{A,C}=0$}
                    child {node{$U_{A_1}=0$}
                        child {node{$U_{A_2}=0$}}
                        child {node{$U_{A_2}=1$}}
                    }
                    child {node{$U_{A_1}=1$}
                        child {node{$U_{A_2}=0$}}
                        child {node{$U_{A_2}=1$}}
                    }
                }
                child {node{$U_{A,C}=1$}
                    child {node{$U_{A_1}=0$}
                        child {node{$U_{A_2}=0$}}
                        child {node{$U_{A_2}=1$}}
                    }
                    child {node{$U_{A_1}=1$}
                        child {node{$U_{A_2}=0$}}
                        child {node{$U_{A_2}=1$}}
                    }
                }
            }
        }
        child {node{$U_{C}=1$}
            child {node{$U_{A,B}=0$}
                child {node{$U_{A,C}=0$}
                    child {node {$U_{A_1}=0$}
                        child {node{$U_{A_2}=0$}}
                        child {node{$U_{A_2}=1$}}
                    }
                    child {node {$U_{A_1}=1$}
                        child {node{$U_{A_2}=0$}}
                        child {node{$U_{A_2}=1$}}
                    }
                }
                child {node{$U_{A,C}=1$}
                    child {node{$U_{A_1}=0$}
                        child {node{$U_{A_2}=0$}}
                        child {node{$U_{A_2}=1$}}
                    }
                    child {node{$U_{A_1}=1$}
                        child {node{$U_{A_2}=0$}}
                        child {node{$U_{A_2}=1$}}
                    }
                }
            }
            child {node{$U_{A,B}=1$}
                child {node{$U_{A,C}=0$}
                    child {node{$U_{A_1}=0$}
                        child {node{$U_{A_2}=0$}}
                        child {node{$U_{A_2}=1$}}
                    }
                    child {node{$U_{A_1}=1$}
                        child {node{$U_{A_2}=0$}}
                        child {node{$U_{A_2}=1$}}
                    }
                }
                child {node{$U_{A,C}=1$}
                    child {node{$U_{A_1}=0$}
                        child {node{$U_{A_2}=0$}}
                        child {node{$U_{A_2}=1$}}
                    }
                    child {node{$U_{A_1}=1$}
                        child {node{$U_{A_2}=0$}}
                        child {node{$U_{A_2}=1$}}
                    }
                }
            }
        }
    };
\end{tikzpicture}
\caption{Logic tree for example's SCM}
\label{fig:logic-tree}
\end{figure}
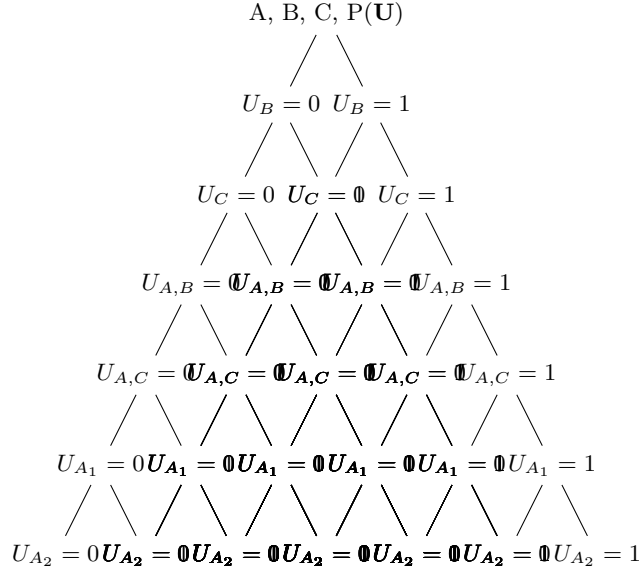
\end{center}

In the GNN-SCM $\mathcal{M}(\mathcal{G})$, the set of functions $\mathcal{F}$ should be the exact form of the interactions between variables. The argument of each function is the input that this function will do one of the logical operations OR, AND, NOT, or XOR  on them. For example, $f_A(B, C, D, u_{A_1},u_{A_2}) = f_A(B,u_{A_1},u_{A_2}) \land f_A(C,u_{A_1},u_{A_2}) \land f_A(D)$ calculates the value of effects on reference observable variable $A$ when we assume that all observable variables B, C, D, and each state $A_1$, or $A_2$ are cause of reference observable variable A to accept a specific value $A_1$, or $A_2$. In this setting, all of these variables should be feasible and have value. \emph{For simplicity, we consider all observable variables as binary, and the probability of these states as uniform distribution.}

In the graph provided, intervention $do(C=1)$ means forcing the value of node $C$'s label to be 1, and the probability $P(A = 1 | do(C = 1))$ determines the respective causal effect for a treatment $A=1$. In the GNN-SCM $\mathcal{M}(\mathcal{G})$, this effect is denoted as $u_{C}$. Since $C$ has an edge to the reference node $A$, there is also a latent variable $u_{A,C}$, and its causal effect can be calculated by $P(do(C=1) | e_{A,C} = 1)$. In this example, there is one node $D$ that does not have an edge to $A$. So, here we just calculate its causal effect on node $A$ by latent variable $u_{D}$.

For the causal effect calculation, we need a truth table showing induced values of $\mathcal{M}(\mathcal{G})$. The logic tree we used for this table is shown in Figure \ref{fig:logic-tree}, which has seven layers since there are seven distinct latent variables($u_{B}$, $u_{C}$, $u_{A,B}$, $u_{A,C}$, $u_{A_1}$, $u_{A_2}$) each accepting the value of 0 (it's not the cause) or 1 (it is the cause).

\subsection{SCM tables}
By interpreting the variables and their values from the logic tree, the truth table will be in four different states. If none of the 1-hop neighborhood nodes' observable variable affects reference node A, if one of them (node B or C) has an effect, or otherwise both of them affect the reference node. Probabilities in P({\bf U}) are labeled from $p_0$ to $p_{63}$ for convenience, which are 7 binary variables(layers in the logic tree).

\label{example truth tables}
\begin{table}[!t]
    \centering
    \begin{tabular}{|c|c|c|c|c|c|c||c|c|c|c||c|}
        \hline
        $u_{B}$ & $u_{C}$ & $u_{D}$ & $u_{A,B}$ & $u_{A,C}$ & $u_{A_1}$ & $u_{A_2}$ & \textbf{B} & \textbf{C} &\textbf{D} & \textbf{A} &\textbf{P({\bf U})} \\ 
        \hline
        0 & 0 & 0 & 0 & 0 & 0 & 0 & 1 & 1 & 1 & $\neg{(B \land C)}$ & $p_0$ \\
        0 & 0 & 0 & 0 & 0 & 1 & 0 & 1 & 1 & 1 & $(B \land C)$ & $p_1$ \\
        0 & 0 & 0 & 0 & 0 & 0 & 1 & 1 & 1 & 1 & $(B \land C)$ & $p_2$ \\
        0 & 0 & 0 & 0 & 0 & 1 & 1 & 1 & 1 & 1 & $\neg{(B \land C)}$ & $p_3$ \\
        \hline
        0 & 0 & 0 & 1 & 0 & 0 & 0 & 0 & 1 & 1 & 1 & $p_4$ \\
        0 & 0 & 0 & 1 & 0 & 1 & 0 & 0 & 1 & 1 & 1 & $p_5$ \\
        0 & 0 & 0 & 1 & 0 & 0 & 1 & 0 & 1 & 1 & 0 & $p_6$ \\
        0 & 0 & 0 & 1 & 0 & 1 & 1 & 0 & 1 & 1 & 0 & $p_7$ \\
        \hline
        0 & 0 & 0 & 0 & 1 & 0 & 0 & 1 & 0 & 1 & 1 & $p_8$ \\
        0 & 0 & 0 & 0 & 1 & 1 & 0 & 1 & 0 & 1 & 1 & $p_9$ \\
        0 & 0 & 0 & 0 & 1 & 0 & 1 & 1 & 0 & 1 & 0 & $p_{10}$ \\
        0 & 0 & 0 & 0 & 1 & 1 & 1 & 1 & 0 & 1 & 0 & $p_{11}$ \\
        \hline
        0 & 0 & 0 & 1 & 1 & 0 & 0 & 0 & 0 & 1 & 1 & $p_{12}$ \\
        0 & 0 & 0 & 1 & 1 & 1 & 0 & 0 & 0 & 1 & 1 & $p_{13}$ \\
        0 & 0 & 0 & 1 & 1 & 0 & 1 & 0 & 0 & 1 & 0 & $p_{14}$ \\
        0 & 0 & 0 & 1 & 1 & 1 & 1 & 0 & 0 & 1 & 0 & $p_{15}$ \\
        \hline
    \end{tabular}
    \caption{Example's SCM truth table where $u_{B} = 0$, and $u_{C} = 0$ (node $B$ and $C$ are the cause of node $A$ to get a specific value). Probabilities in $P(U)$ are labeled from $p_{0}$ to $p_{15}$ for convenience.}
    \label{tab:example1}
\end{table}

\begin{table}[!t]
    \centering
    \begin{tabular}{|c|c|c|c|c|c|c||c|c|c|c||c|}
        \hline
        $u_{B}$ & $u_{C}$ & $u_{D}$ & $u_{A,B}$ & $u_{A,C}$ & $u_{A_1}$ & $u_{A_2}$ & \textbf{B} & \textbf{C} &\textbf{D} & \textbf{A} &\textbf{P({\bf U})} \\ 
        \hline
        1 & 0 & 0 & 0 & 0 & 0 & 0 & 0 & 1 & 1 & $\neg{(B \land C)}$ & $p_{16}$ \\
        1 & 0 & 0 & 0 & 0 & 1 & 0 & 0 & 1 & 1 & $(B \land C)$ & $p_{17}$ \\
        1 & 0 & 0 & 0 & 0 & 0 & 1 & 0 & 1 & 1 & $(B \land C)$ & $p_{18}$ \\
        1 & 0 & 0 & 0 & 0 & 1 & 1 & 0 & 1 & 1 & $\neg{(B \land C)}$ & $p_{19}$ \\
        \hline
        1 & 0 & 0 & 1 & 0 & 0 & 0 & 0 & 1 & 1 & 1 & $p_{20}$ \\
        1 & 0 & 0 & 1 & 0 & 1 & 0 & 0 & 1 & 1 & 1 & $p_{21}$ \\
        1 & 0 & 0 & 1 & 0 & 0 & 1 & 0 & 1 & 1 & 0 & $p_{22}$ \\
        1 & 0 & 0 & 1 & 0 & 1 & 1 & 0 & 1 & 1 & 0 & $p_{23}$ \\
        \hline
        1 & 0 & 0 & 0 & 1 & 0 & 0 & 0 & 0 & 1 & 1 & $p_{24}$ \\
        1 & 0 & 0 & 0 & 1 & 1 & 0 & 0 & 0 & 1 & 1 & $p_{25}$ \\
        1 & 0 & 0 & 0 & 1 & 0 & 1 & 0 & 0 & 1 & 0 & $p_{26}$ \\
        1 & 0 & 0 & 0 & 1 & 1 & 1 & 0 & 0 & 1 & 0 & $p_{27}$ \\
        \hline
        1 & 0 & 0 & 1 & 1 & 0 & 0 & 0 & 0 & 1 & 1 & $p_{28}$ \\
        1 & 0 & 0 & 1 & 1 & 1 & 0 & 0 & 0 & 1 & 1 & $p_{29}$ \\
        1 & 0 & 0 & 1 & 1 & 0 & 1 & 0 & 0 & 1 & 0 & $p_{30}$ \\
        1 & 0 & 0 & 1 & 1 & 1 & 1 & 0 & 0 & 1 & 0 & $p_{31}$ \\
        \hline
    \end{tabular}
    \caption{Example's SCM truth table where $u_{B} = 1$, and $u_{C} = 0$ (only node $C$ is not the cause of node $A$ to get a specific value). Probabilities in $P(U)$ are labeled from $p_{16}$ to $p_{31}$ for convenience.}
    \label{tab:example2}
\end{table}

\begin{table}[!t]
    \centering
    \begin{tabular}{|c|c|c|c|c|c|c||c|c|c|c||c|}
        \hline
        $u_{B}$ & $u_{C}$ & $u_{D}$ & $u_{A,B}$ & $u_{A,C}$ & $u_{A_1}$ & $u_{A_2}$ & \textbf{B} & \textbf{C} &\textbf{D} & \textbf{A} &\textbf{P({\bf U})} \\ 
        \hline
        0 & 1 & 0 & 0 & 0 & 0 & 0 & 1 & 0 & 1 & $\neg{(B \land C)}$ & $p_{32}$ \\
        0 & 1 & 0 & 0 & 0 & 1 & 0 & 1 & 0 & 1 & $(B \land C)$ & $p_{33}$ \\
        0 & 1 & 0 & 0 & 0 & 0 & 1 & 1 & 0 & 1 & $(B \land C)$ & $p_{34}$ \\
        0 & 1 & 0 & 0 & 0 & 1 & 1 & 1 & 0 & 1 & $\neg{(B \land C)}$ & $p_{35}$ \\
        \hline
        0 & 1 & 0 & 1 & 0 & 0 & 0 & 0 & 0 & 1 & 1 & $p_{36}$ \\
        0 & 1 & 0 & 1 & 0 & 1 & 0 & 0 & 0 & 1 & 1 & $p_{37}$ \\
        0 & 1 & 0 & 1 & 0 & 0 & 1 & 0 & 0 & 1 & 0 & $p_{38}$ \\
        0 & 1 & 0 & 1 & 0 & 1 & 1 & 0 & 0 & 1 & 0 & $p_{39}$ \\
        \hline
        0 & 1 & 0 & 0 & 1 & 0 & 0 & 1 & 0 & 1 & 1 & $p_{40}$ \\
        0 & 1 & 0 & 0 & 1 & 1 & 0 & 1 & 0 & 1 & 1 & $p_{41}$ \\
        0 & 1 & 0 & 0 & 1 & 0 & 1 & 1 & 0 & 1 & 0 & $p_{42}$ \\
        0 & 1 & 0 & 0 & 1 & 1 & 1 & 1 & 0 & 1 & 0 & $p_{43}$ \\
        \hline
        0 & 1 & 0 & 1 & 1 & 0 & 0 & 0 & 0 & 1 & 1 & $p_{44}$ \\
        0 & 1 & 0 & 1 & 1 & 1 & 0 & 0 & 0 & 1 & 1 & $p_{45}$ \\
        0 & 1 & 0 & 1 & 1 & 0 & 1 & 0 & 0 & 1 & 0 & $p_{46}$ \\
        0 & 1 & 0 & 1 & 1 & 1 & 1 & 0 & 0 & 1 & 0 & $p_{47}$ \\
        \hline
    \end{tabular}
    \caption{Example's SCM truth table where $u_{B} = 0$, and $u_{C} = 1$ (only node $B$ is not the cause of node $A$ to get a specific value). Probabilities in $P(U)$ are labeled from $p_{32}$ to $p_{47}$ for convenience.}
    \label{tab:example3}
\end{table}

\begin{table}[!t]
    \centering
    \begin{tabular}{|c|c|c|c|c|c|c||c|c|c|c||c|}
        \hline
        $u_{B}$ & $u_{C}$ & $u_{D}$ & $u_{A,B}$ & $u_{A,C}$ & $u_{A_1}$ & $u_{A_2}$ & \textbf{B} & \textbf{C} &\textbf{D} & \textbf{A} &\textbf{P({\bf U})} \\ 
        \hline
        1 & 1 & 0 & 0 & 0 & 0 & 0 & 0 & 0 & 1 & $\neg{(B \land C)}$ & $p_{48}$ \\
        1 & 1 & 0 & 0 & 0 & 1 & 0 & 0 & 0 & 1 & $(B \land C)$ & $p_{49}$ \\
        1 & 1 & 0 & 0 & 0 & 0 & 1 & 0 & 0 & 1 & $(B \land C)$ & $p_{50}$ \\
        1 & 1 & 0 & 0 & 0 & 1 & 1 & 0 & 0 & 1 & $\neg{(B \land C)}$ & $p_{51}$ \\
        \hline
        1 & 1 & 0 & 1 & 0 & 0 & 0 & 0 & 0 & 1& 1 & $p_{52}$ \\
        1 & 1 & 0 & 1 & 0 & 1 & 0 & 0 & 0 & 1 & 1 & $p_{53}$ \\
        1 & 1 & 0 & 1 & 0 & 0 & 1 & 0 & 0 & 1 & 0 & $p_{54}$ \\
        1 & 1 & 0 & 1 & 0 & 1 & 1 & 0 & 0 & 1 & 0 & $p_{55}$ \\
        \hline
        1 & 1 & 0 & 0 & 1 & 0 & 0 & 0 & 0 & 1 & 1 & $p_{56}$ \\
        1 & 1 & 0 & 0 & 1 & 1 & 0 & 0 & 0 & 1 & 1 & $p_{57}$ \\
        1 & 1 & 0 & 0 & 1 & 0 & 1 & 0 & 0 & 1 & 0 & $p_{58}$ \\
        1 & 1 & 0 & 0 & 1 & 1 & 1 & 0 & 0 & 1 & 0 & $p_{59}$ \\
        \hline
        1 & 1 & 0 & 1 & 1 & 0 & 0 & 0 & 0 & 1 & 1 & $p_{60}$ \\
        1 & 1 & 0 & 1 & 1 & 1 & 0 & 0 & 0 & 1 & 1 & $p_{61}$ \\
        1 & 1 & 0 & 1 & 1 & 0 & 1 & 0 & 0 & 1 & 0 & $p_{62}$ \\
        1 & 1 & 0 & 1 & 1 & 1 & 1 & 0 & 0 & 1 & 0 & $p_{63}$ \\
        \hline
    \end{tabular}
    \caption{Example's SCM truth table where $u_{B} = 1$, and $u_{C} = 1$ (node $B$ and $C$ are not the cause of node $A$ to get a specific value). Probabilities in $P(U)$ are labeled from $p_{48}$ to $p_{63}$ for convenience.}
    \label{tab:example4}
\end{table}
In all provided table rows, $u_D = 0$ and $D=1$, meaning  D is not the cause and the latent variable of it is 1. However, there is no edge between nodes A and D, we need to mention this in our calculations. For simplicity, we just showed the cases that $u_D = 0$, but there are the same truth tables with $u_D = 1$ and $D = 0$ by probabilities $p_{64}$ to $p_{127}$. Given the probabilities from the truth tables, we can define them as follows:
\begin{align*}
    P({\bf U}) :=  \text{Unif}(0,1) \implies p_0 = p_1 = p_2 = \ldots = p_{63} = \ldots = p_{127} = {1}/{128}
\end{align*}
\subsection{GNN-SCM results}
The capability of the tables shows that our specified GNN-SCM $\mathcal{M}(\mathcal{G})$ can calculate all queries from each PCH layer~\cite{pearl2018book}. In continue, we will calculate an example for each layer:

An association layer query such as $P(A = 1 | C = 1)$ which is the probability of observable variable A to be 1 given observable variable C to be 1, can be computed as:

{
\scriptsize
\begin{align}\label{eq:level1 example}
& P(A = 1 | C = 1) = \frac{P(A = 1 , C = 1)}{P(C = 1)} \nonumber \\
& = \frac{p1 + p2 + p4 + p5 + p17 + p18 + p20 + p21}{p0 + p1 + p2 + p3+ p4 + p5 + p6 + p7 + p16 + p17 + p18 + p19 + p20 + p21 + p22 + p23}   = 0.5  
\end{align}
}%

An intervention layer query such as $P(A = 1 | do(C = 1))$ which is the probability of A being 1 after this intervention on C. It seeks to understand the causal effect of setting C to 1 on outcome A . $do(C = 1)$ represents an intervention where the variable C is actively set to 1 to control the variable directly, essentially breaking its usual causal edges with other variables. This query can be computed as:

{
\scriptsize
\begin{align}\label{eq:level2 example}
 &  P(A = 1 | do(C = 1)) = P(A = (B \land C) | C = 1) \lor P(A = 1 | C = 0)  \nonumber \\ 
 & = \frac{p1 + p2 + p17 + p18}{p0 + p1 + p2 + p3+ p4 + p5 + p6 + p7 + p16 + p17 + p18 + p19 + p20 + p21 + p22 + p23}  \nonumber \\ 
& + \frac{p24 + p25 + p28 + p29 + p36 + p37 + p40 + p41 + p44 + p45 + p52 + p53 + p56 + p57 + p60 + p61}
{p8 + p9  + ... 
+ p14 + p15 + p24 + p25 + p26 + p27 + p28 + p29 + p30 + p31 + p32 + ... + p63} \nonumber \\ 
&  = 0.25 + 0.33 = 0.58
\end{align}
}%

For the implementation, we need the observed data generated from the graph. One generated data can be: \{$A: True, B: False, C: False, D: True, u_{A_1}: 1, u_{A_2}: 0, u_{A,B}: 1, u_{A,C}: 1, u_B: 0, u_C: 1, u_D: 0, A|C : 1$\}. The Fraction of samples that satisfy each function of the GNN-SCM $\mathcal{M}(\mathcal{G})$: \{$f_A: 0.156458, f_B:  0.249319, f_C: 0.25009, f_D: 0.50059$\}. These values show the range and central tendency of the probabilities across the 500 trials, indicating a degree of variability in the outcomes based on the random generation of the probabilities.

\subsection{Example's GNN-NCM}
With respect to literature, an NCM is as expressive as an SCM, and all NCMs are SCMs. The causal diagram constraints are the bias between SCMs and NCMs. In our specified GNN-SCM $\widehat{\mathcal{M}}(\mathcal{G})$, and GNN-NCM $\widehat{\mathcal{M}}(\mathcal{G}, \theta))$, all causal information(observable, and latent variables) comes from the causal structure defined in Definition. \ref{causal network}. The respective GNN-NCM $\widehat{\mathcal{M}}(\mathcal{G}, \theta)$ is constructed as a proxy of the exact GNN-SCM $\mathcal{M}(\mathcal{G})$. Based on Equation. \ref{eq:gnn-ncm}, reference node A is chosen as the target node for causal structure $\mathcal{G}$. This GNN-NCM $\widehat{\mathcal{M}}(\mathcal{G}, \theta)$ is an inductive bias type of the GNN-SCM as $\widehat{\mathcal{M}}(\mathcal{G}, \theta)$=:$\langle \widehat{\bf U},V,\mathcal{\widehat{F}},P(\widehat{\bf U})\rangle $. The construction of the corresponding GNN-NCM that induces the same distributions for our example dataset is as follows:

{
\footnotesize
\begin{align}\label{eq:example NCM}
    \widehat{\mathcal{M}}(\mathcal{G}, \theta))= 
        \begin{cases}
            \widehat{\bf U} := \{\widehat{\bf U}\}, D_{\widehat{\bf U}} = [0,1]\\
            V := \{A,B,C,D\}\\
            \widehat{\mathcal{F}} := \begin{cases}
            \widehat{f}_A(\widehat{\bf U}) = ?\\
            \widehat{f}_B(A,\widehat{\bf U}) = ?\\
            \widehat{f}_C(A,\widehat{\bf U}) = ?\\
            \widehat{f}_D(A,\widehat{\bf U}) = ?\end{cases}\\
            P(\widehat{\bf U}) :=  ?\\
        \end{cases}
\end{align}
}

\begin{figure}[!t]
\centering
    \includegraphics[width=0.4\linewidth]{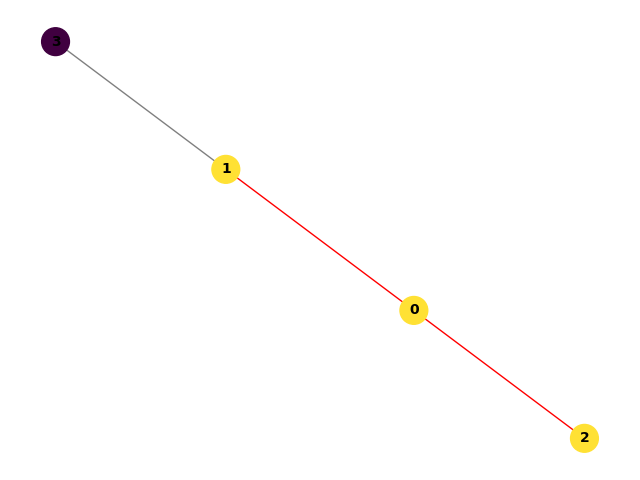}
  \caption{Result of GNN-NCM $\widehat{\mathcal{M}}(\mathcal{G}, \theta))$ for toy example}
\label{fig:vis_toy}
\end{figure}

We know the observable variables $V$ values, but there is no clue about the exact values of latent variables U, so we have to estimate them by functions $\widehat{\mathcal{F}}$ based on the information in the given causal structure $\mathcal{G}$. First, we have to build the causal structure $\mathcal{G}$ given example graph G. 

{
\footnotesize
\begin{align}
\mathcal{G}(G) &= \big\{ {\bf V}_{v} = \{A, B, C, D\} \}, \notag \\
& \quad  {\bf U}_{v} = \big\{ {\bf U}_{v_i}: \{u_{B}, u_{C}, u_{D}, u_{A_1}, u_{A_2}\} \cup \{{\bf U}_{v,v_i}: \{u_{A,B}, u_{A,C}\}\big\}
\end{align}
}%
Second, given the probabilities from the truth tables, we have: 
\begin{align*}
    P(\widehat{\bf U}) :=  \text{Unif}(0,1) \implies p_0 = p_1 = p_2 = \ldots = p_{63} = {1}/{256}
\end{align*}
At last, based on Algorithm \ref{Algorithm1}, we train the GNN-NCM to find the functions. The GNN-NCM extends the GNN-SCM to utilize the power of neural networks in capturing complex patterns in the dataset. The process begins with sampling from a prior distribution to simulate the unobserved confounders in the causal process. 
Since the reference node in the GNN-SCM was node $A$, we assign value 1 to it and want to see how GNN-NCM finds the causal effects on this node in graph G. The result of the Algorithm \ref{Algorithm2}, are:

When the target node is $A$, the expected probability is 0.24082797765731812

When the target node is $B$, the expected probability is 0.2353899081548055

When the target node is $C$, the expected probability is 0.18209974467754364

When the target node is $D$, the expected probability is 0.12958189845085144

Hence, the final causal explanatory subgraph $\Gamma$ based on the results is the GNN-NCM $\widehat{\mathcal{M}}(\mathcal{G}, \theta))$ with target node $A$ is:
{
\begin{align}
& \mathcal{G}(G) = \big\{ {\bf V}_{v} = \{A = (B \land C), B = 1, C = 1, D = 0\} \}, \nonumber \\ 
& {\bf U}_{v} = \big\{ {\bf U}_{v_i}: \{u_{B}=0, u_{C}=0, u_{D}=1, u_{A_1}=0, u_{A_2}=1\} \notag \\ 
& \quad \quad  \cup \{{\bf U}_{v,v_i}: \{u_{A,B}=1, u_{A,C}=1\}\big\}
\end{align}
}

The GNN-NCM $\widehat{\mathcal{M}}(\mathcal{G}, \theta))$ structure is:
\begin{center}
\small
\begin{equation}\label{eq:example NCM}
    \widehat{\mathcal{M}}(\mathcal{G}, \theta))= 
        \begin{cases}
            \widehat{\bf U} := \Big\{ {\bf U}_{v_i}: \{u_{B}=0, u_{C}=0, u_{D}=1, u_{A_1}=0, u_{A_2}=1\} \\ \qquad \cup \{{\bf U}_{v,v_i}: \{u_{A,B}=1, u_{A,C}=1\}\Big\}\\
            V := \{A = (B \land C), B = 1, C = 1, D = 0\}\\
            \widehat{\mathcal{F}} := \begin{cases}
            \widehat{f}_A(\widehat{\bf U}) = \widehat{f}_{A}(\widehat{\bf u}_{A_1} = 0, \widehat{\bf u}_{A_2} = 1,\widehat{\bf u}_{{A},{B}} = 1,\widehat{\bf u}_{{A},{C}} = 1)\\
            \widehat{f}_B(A,\widehat{\bf U}) = \widehat{f}_{B}(\widehat{\bf u}_{B} = 0,\widehat{\bf u}_{{A},{B}} = 1)\\
            \widehat{f}_C(A,\widehat{\bf U}) = \widehat{f}_{C}(\widehat{\bf u}_{C} = 0,\widehat{\bf u}_{{A},{C}} = 1)\\
            \widehat{f}_D(A,\widehat{\bf U}) = \widehat{f}_{D}(\widehat{\bf u}_{D} = 1)\end{cases}\\
            P(\widehat{\bf U}) :=  \begin{cases}P(u_{A_1}) = P(u_{A_2}) = P(u_{A,B}) = P(u_{A,C})\\ = P(u_{A,D}) = P(u_{B}) = P(u_{C}) = P(u_{D})\end{cases} = 1/8\\
        \end{cases}
\end{equation}
\end{center}

\section{More Background on Causality}
\label{app:background}

According to the literature, causality interprets the information by the Pearl Causal Hierarchy (PCH) layers~\cite{pearl2018book}.

\begin{definition} [PCH layers] \label{PCH}
The PCH layers $\mathrm{L}_i$ for $i$ $\in$ {1,2,3} are: $\mathrm{L}_1$ association layer, $\mathrm{L}_2$ intervention layer, and $\mathrm{L}_3$ counterfactual layer. 
\end{definition}

\begin{definition}[$\mathcal{G}$-Consistency] \label{def:g-con}
  Let $\mathcal{G}$ be the causal structure induced by SCM $\mathcal{M}^*$. For any SCM $\mathcal{M}$, we say $\mathcal{M}$ is $\mathcal{G}$-consistent w.r.t $\mathcal{M}^*$ if $\mathcal{M}$ imposes the same constraints over the interventional distributions as the true $\mathcal{M}^*$.  
\end{definition} 

\begin{definition}[$\mathcal{G}$-Constrained NCM]
    \label{def:g-cons-nscm}
    Let $\mathcal{G}$ be the causal structure induced by  SCM $\mathcal{M}^*$.
    We can construct NCM $\widehat{\mathcal{M}}$ as follows: 1) Choose $\widehat{\bf U}$ s.t. $\widehat{U}_{\bf C} \in \widehat{\bf U}$, where any pair $(V_i, V_j) \in {\bf C}$ is connected with a bidirected arrow in $\mathcal{G}$ and is maximal; 
    2) For each $V_i \in {\bf V}$, choose $\textbf{Pa}(V_i) \subseteq {\bf V}$ s.t. for every $V_j \in {\bf V}$, $V_j \in \textbf{Pa}(V_i)$ iff there is a directed arrow from $V_j$ to $V_i$ in $\mathcal{G}$.
    Any NCM in this family is said to be $\mathcal{G}$-constrained. 
\end{definition} 

\begin{theorem}
\label{thm:nscm-g-cons}
Any $\mathcal{G}$-constrained NCM $\widehat{\mathcal{M}}(\bm{\theta})$ is $\mathcal{G}$-consistent. 
\end{theorem}

\section{Proofs}
\label{app:proofs}
In this section, we provide proofs of the theorems in the main body of the paper.

\subsection{Proof of Theorem~\ref{thm:GNN-SCM}}
\label{app:gnn-scm}
\thmgnnscm*

\begin{proof} 
\label{proof1}
A GNN is a neural network operating on a graph $G=(\mathcal{V}, \mathcal{E})$ including set of nodes $\mathcal{V}$  and set of edges $\mathcal{E}$. Recall the GNN background in Section~\ref{sec:prelim}, the GNN learning mechanism for a node $v$ in the $l$-th layer can be summarized as:
\begin{equation}
\text{GNN}(v) \equiv \begin{cases}
            \text{node embeddings } h^{l-1}_v  \text{ from previous layer } l-1, \text{ for }  u \in \{v\} \cup \mathcal{N}(v) \\
            \text{message } m^l_{u,v} = \text{MSG}(h^{l-1}_u, h^{l-1}_v, e_{u,v}) \text{ for current layer } l\\
            \text{aggregated message } h^{l}_v = \text{AGG}(m^l_{u,v} |  u \in \mathcal{N}(v)) \text{ for current layer } l\\
        \end{cases}
    \end{equation}
where the above process is iteratively performed $k$ times for a $k$-layer GNN. In doing so, each node $v$ will leverage the information from all its within $k$ neighborhoods. 
We denote the $k$-layer GNN learning for $v$ as:
$$\text{GNN}(G) = \big \{ \text{node embed: } \{ h_v \} \cup \{ h_u: u \in \mathcal{N}_{\leq k} (v)\}, \text{message: } \{m_{u,v}, u \in \mathcal{N}_{\leq k} (v)\} \big \},$$
where  $h_v$ is $v$' node feature and we omit the dependence on node $v$ for notation simplicity.  

By definition from literature, an SCM $\mathcal{M}$ is a four-tuple $\mathcal{M} \equiv ({\bf U}, {\bf V},\mathcal{F},P({\bf U}))$. In this specification of ours, an SCM $\mathcal{M}(\mathcal{G})$ is a $\mathcal{G}$-consistent four-tuple model based on a set of observable variables ${\bf V}$, a set of latent variables ${\bf U}$, a set of functions $\mathcal{F}$, and the probability of latent variables $P({\bf U})$.  
Recall in Definition \ref{causal network}, where the causal structure $\mathcal{G}(G)$ of a given graph $G$ (centered on a reference node $v$) was defined. Now, the correspondence of the GNN learning on $G$ and the causal structure $\mathcal{G}$ centered on a node $v$ can be written as: 
{
\scriptsize
\begin{equation}
\mathcal{G}(G)\equiv \text{GNN}(G)
\iff \begin{cases}
  \text{obs. vars: } \{y_{v}\} \cup \{y_{v_i}: 
    v_i \in  \mathcal{N}_{\leq k}(v) \}
  \equiv \text{node embed. } \{h_{v} \} \cup \{ h_u: u \in \mathcal{N}_{\leq k} (v)\} \\
    \text{lat. vars:  } \{{\bf U}_{v_i}: v_i \in \mathcal{N}_{\leq k}(v) \} \cup \{{\bf U}_{v,v_i}:{e_{v,v_i} \in \mathcal{E}}\}
    \equiv \text{ msg: } \{m_{u,v}, u \in \mathcal{N}_{\leq k} (v)\} \\
    \text{probability of latent variables } P({\bf U}) \equiv \text{Dom}(\{{h_v}\})\\
\end{cases}
 \end{equation}
}

Hence, there exists a GNN-SCM $\mathcal{M}(\mathcal{G})$ that induces the causal structure $\mathcal{G}$ of $G$, as below:
{
\scriptsize
\begin{align}
  \begin{array}{r@{\mskip\thickmuskip}l}
 \mathcal{G}(G) =    \begin{cases}
 \text{node set } \mathcal{V}(\mathcal{G})\\
 \text{edge set } \mathcal{E(G)}\\
            \text{node effect } = \{{\bf U}_{v_i}\}\\
            \text{edge effect } = \{{\bf U}_{v, v_i}\} 
            \end{cases}
  \end{array} 
  \iff
   \mathcal{M}(\mathcal{G}) 
   \equiv 
  \begin{cases}
    {\bf U} \text{ : }\{{\bf U}_{v_i}: v_i \in \mathcal{N}_{\leq k}(v) \} \cup \{{\bf U}_{v,v_i}:{e_{v,v_i} \in \mathcal{E}}\}\\
    {\bf V} \text{ : } \{y_{v}\} \cup \{y_{v_i};
    v_i \in  \mathcal{N}_{\leq k}(v) \}\\
    \mathcal{F} \text{ : } \{f_1, f_2, \ldots, f_n\} \in \mathcal{F} \text{; } f({\bf U}) \rightarrow {\bf V}\\
    P({\bf U}) \text{ : }P({\bf U}_{v_i}), P({\bf U}_{v,v_i}) 
    \\
    \end{cases}
    \label{eqn:gnn-ncm}
\end{align}
}
\end{proof}

\subsection{Proof of Theorem~\ref{thm:thmgnnncm}}
\label{app:thmgnnncm}
\thmgnnncm*
\begin{proof} \label{proof2}

From the literature, we know there exists a SCM $\mathcal{M}$ that includes exact values of observable and latent variables through studying the causes and effects within the SCM structure. First, we show a lemma that demonstrates the inheritance of neural causal models (NCMs) (see its definition in Section~\ref{app:background}) from SCMs, which are built upon Definition~\ref{def:g-con}, Definition~\ref{def:g-cons-nscm} and Theorem~\ref{thm:nscm-g-cons}. 

\begin{lemma}[\cite{xia2021causal}] \label{SCM-NCM}
All  NCMs $\widehat{\mathcal{M}}(\theta)$ (parameterized by $\theta$) are SCMs (i.e., $\widehat{\mathcal{M}}(\theta) \prec M$). Further, any $\mathcal{G}$-constrained $\widehat{\mathcal{M}}(\theta)$ (see Definition~\ref{def:g-cons-nscm}) has the same empirical observations as the SCM $\mathcal{M}$, which means $\mathcal{G}$-constrained NCMs can be used for generating any distribution associated with the SCMs.
\end{lemma}

By Lemma~\ref{SCM-NCM}, we know a $\mathcal{G}$-constrained NCM $\widehat{\mathcal{M}}(\theta)$ inherits all properties of the respective SCM $\mathcal{M}$ and ensures  causal inferences via $\mathcal{G}$-constrained NCM. In our context, we need to build the corresponding $\mathcal{G}$-constrained GNN-NCM $\widehat{\mathcal{M}}(\mathcal{G},\theta)$ for the GNN-SCM defined in Equation~\ref{eqn:gnn-ncm}. With it, we ensure all $\mathcal{G}$-constrained GNN-NCMs $\widehat{\mathcal{M}}(\mathcal{G},\theta)$ are GNN-SCMs ($\widehat{\mathcal{M}}(\mathcal{G},\theta) \prec \mathcal{M}(\mathcal{G})$), meaning these GNN-NCMs can be used for performing causal inferences on the causal structure $\mathcal{G}(G)$. First, based to the four-tuple SCM $\mathcal{M} \equiv ({\bf U},{\bf V},\mathcal{F},P({\bf U}))$, a $\mathcal{G}$-constrained NCM  $\widehat{\mathcal{M}}(\theta) = (\widehat{\bf U},{\bf V},\mathcal{\widehat{F}(\theta)},P(\widehat{\bf U}))$ can be defined. 
In our scenario, we can define the set of functions $\mathcal{\widehat{F}(\theta)}$ of the $\mathcal{G}$-constrained GNN-NCM as:
\begin{center} 
$\mathcal{\widehat{F}(\theta)} = \{\widehat{f}_{v_i}(\widehat{\bf u}_{v_i},\widehat{\bf u}_{{v_i},{v_j}}; \theta_{v_i}):v_i \in \mathcal{V}(\mathcal{G})\} \approx \{f_1, f_2, \ldots, f_n\} \in \mathcal{F} \text{; } f({\bf U}_{v_i}) \rightarrow v_i$,
\end{center}

From Theorem~\ref{thm:GNN-SCM}, there exists a GNN-SCM $\mathcal{M}(\mathcal{G})$ for a GNN operating on a graph $G$. Also the causal structure $\mathcal{G}(G)$ in Definition \ref{causal network} naturally satisfies Definition~\ref{def:g-cons-nscm}. 
Then, with respect to Equation~\ref{eq:gnn-ncm},  our $\mathcal{G}$-constrained GNN-NCM $\widehat{\mathcal{M}}(\mathcal{G},\theta)$ based on underlying GNN-SCM $\mathcal{M}(\mathcal{G})$ is defined as: 
{
\small
\begin{align*}
 \mathcal{M}(\mathcal{G}) \iff 
\widehat{\mathcal{M}}(\mathcal{G},\theta) 
& \equiv
\begin{cases}
    \widehat{\bf U} := \{{\widehat{\bf U}}_{v_i}: v_i \in \mathcal{N}_{\leq k}(v) \} \cup \{{\widehat{\bf U}}_{v,v_i}:{e_{v,v_i} \in \mathcal{E}}\}\\
    {\bf V} := \{y_{v}\} \cup \{y_{v_i};
    v_i \in  \mathcal{N}_{\leq k}(v) \}\\
    \mathcal{\widehat{F}(\theta)}:= \{\widehat{f}_{v_i}(\widehat{\bf u}_{v_i},\widehat{\bf u}_{{v_i},{v_j}})(\theta_{v_i}):v_i \in V\}\\
    P(\widehat{\bf U}) := \{\widehat{\bf U}_{v_i} \sim \text{Unif}(0,1): v_i \in {\bf V} \}  \cup \{ T_{k, v_i} \sim \mathcal{N}(0,1)\}\\
    \end{cases}
\end{align*}
}%
\end{proof}

\subsection{Proof of Theorem~\ref{thm:nodeex}}
\label{app:nodeex}
\nodeex*
\begin{proof} \label{proof3}
In a graph classification task, GNN predicts a graph label $\widehat{y}_G$ for a graph $G$ with label $y_G$. The GNN explanation measures how accurately did the GNN classify the graph by finding the groundtruth explanation $\Gamma_{G}$ in the graph $G$. 
In other words, the graph explanation demands the nodes in $\Gamma_{G}$
should be as accurate as possible. That is, 
if $y_v = \widehat{y}_v; \forall v \in \Gamma_{G}$, then GNN explanation explained $G$'s prediction accurately.

Based on Theorem~\ref{thm:thmgnnncm}$, \mathcal{G}$-constrained GNN-NCM $\widehat{\mathcal{M}}(\mathcal{G},\theta)$ induces causal structure $\mathcal{G}$ based on the reference node $v \in \mathcal{V}(\mathcal{G})$. So, the trained $\mathcal{G}$-constrained GNN-NCM estimates node effect ${\bf U}_{v_i}$, and edge effect ${\bf U}_{v, v_i}$ defined in Equation~\ref{causal net}. 
Note that all the effects are respective to the reference node $v$, and if $v$ changes, the causal structure $\mathcal{G}$ will be also changed, and as a result $\mathcal{G}$-constrained GNN-NCM will be completely different.

According to Equation. \ref{eq:NM6}, a $\mathcal{G}$-constrained GNN-NCM $\widehat{\mathcal{M}}(\mathcal{G},\theta)$ calculates $p^{\widehat{\mathcal{M}}(\mathcal{G},\theta)}(y_{v})$ as the expected value of all the causal effects from the neighbor nodes $v_i$, i.e. $\text{do}(v_i)$, on the reference node $v$: 
\begin{equation}
    \forall v \in \mathcal{V}(G), \left( \exists v_i \in \mathcal{N}_{\leq k}(v) \right) \implies \left( p^{\widehat{\mathcal{M}}(\mathcal{G}(G),\theta)}(y_{v}) \geq 0 \right)
\end{equation}
As the expected value was calculated for $v$, we can explain $v$'s node label based on the outcome of $p^{\widehat{\mathcal{M}}(\mathcal{G}(G),\theta)}(y_{v})$.
\end{proof}

\subsection{Proof of Theorem~\ref{thm:nodeexp}}
\label{app:nodeexp}
\nodeexp*
\begin{proof} \label{proof4}
We know there is a value for the expected effect on an explainable node $v \in \mathcal{V}(G)$ as $p^{\widehat{\mathcal{M}}(\mathcal{G}(G),\theta)}(y_{v})$. This probability was calculated by the trained $\mathcal{G}$-constrained GNN-NCM $\widehat{\mathcal{M}}(\mathcal{G},\theta)$. 

For our purpose, we only consider the association and intervention layers. 
The association layer is about the observable information provided by the data, while the intervention layer in this paper is an explanation via doing interventions. 
\begin{equation}
\begin{aligned}
    & \text{Association Layer } \mathrm{L}_1: \quad G=(\mathcal{V},\mathcal{E}), \; y_G \in \mathcal{Y}, \; \Gamma_G \\
    & \text{Intervention Layer } \mathrm{L}_2: \quad \text{do}(v_i) = (y_v|y_{v_i} = x), \; \mathbf{U}_{v_i}, \; \mathbf{U}_{v,v_i} \\
\end{aligned}
\end{equation}
 
The explanation methods using information in the association layer is called association-based explanation. Instead, the $\mathcal{G}$-constrained GNN-NCM $\widehat{\mathcal{M}}(\mathcal{G}(G),\theta)$ is trained on interventions---a node $v_i$ in the neighborhood of the reference node $v$ provides the causal explanation information based on the intervention do$(v_i)$---and leveraging this interventional layer information can causality interpret the GNN predictions. To align this intrinsic explanation information from causal effects, we introduce the term \emph{expressivity}. 
Remember in probability theory, where the expected value of a random variable provides a measure of the central tendency of a probability distribution. 
For a discrete random variable \(Y\) with a probability distribution \(p(y)\), the expected value of \(Y\) , denoted \(\mathbb{E}(Y)\), is defined as:
$\mathbb{E}(Y_v) = \sum_{y} y \cdot p(y)$, 
where \(y\) ranges over all possible values of \(Y\), and \(p(y)\) is the probability that \(Y\) takes the value \(y\).  
According to Equation~\ref{eq:NM6}, $p^{\widehat{\mathcal{M}}(\mathcal{G}(G),\theta)}(y_{v})$ includes all the causal effect  from the 
neighhor nodes on $y_v$. Hence, the expected value of of the random variable node label \(Y_v\) will be defined as:
\[
\mathbb{E}_{\mathrm{L}_2}(Y_v) = \sum_{y_{v}} y_{v} \cdot p^{\widehat{\mathcal{M}}(\mathcal{G}(G),\theta)}(y_{v}),
\]
where the subscript $\mathrm{L}_2$ means the expectation leverages the interventional layer information. 
Note that this expected value is only feasible for the reference node (i.e., $v$) upon which the $\mathcal{G}$-constrained GNN-NCM $\widehat{\mathcal{M}}(\mathcal{G},\theta)$ was built. This expected value is treated as the expressivity of the explainable node $v$ that is denoted as $\text{exp}_{v}(\widehat{\mathcal{M}}(\mathcal{G},\theta))$.
\end{proof}

\begin{figure}[!tbh]
  \centering
  \includegraphics[width=0.8\textwidth]{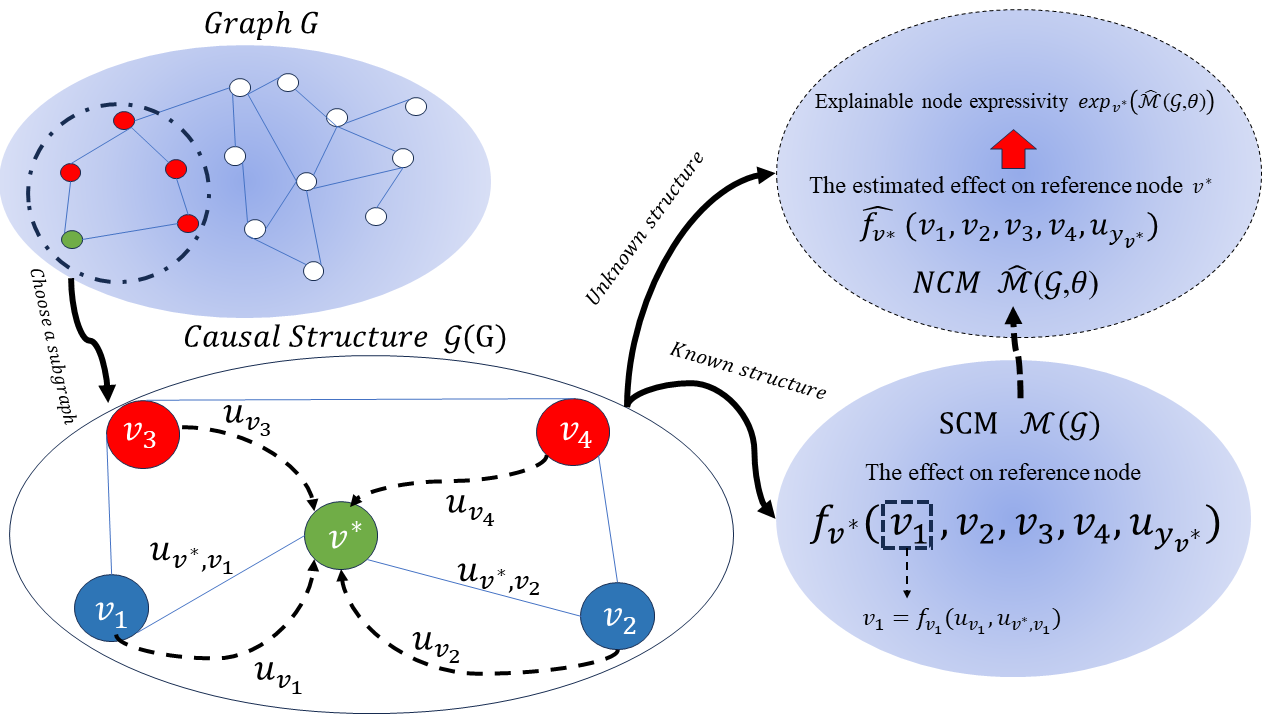}
  \caption{Overview. {\color{green} Green node} is the reference node $v$, {\color{blue} Blue nodes} and {\color{red}Red nodes} are node $v$'s $1$-hop and $2$-hop neighbors.}
  \label{fig:Causal graph explanation}
\end{figure}

\section{Experiments}
\label{app:exp}

\subsection{More experimental setup}
\label{app:setup}

\noindent {\bf CXGNN:} The hyperparameter settings were determined through a systematic search and validation process to optimize the model's performance. 
The following hyperparameters were selected based on cross-validation:
\begin{itemize}
\item {Learning rate:} 
We test learning rates—0.001, 0.01, 0.1—to find the optimal value, and 
the learning rate of 0.01 yielded the best results.

\item {Number of hidden layers:} We considered architectures with 1, 2, and 3 hidden layers. A network with 2 hidden layers outperformed the others in terms of both accuracy and convergence speed.

\item {Hidden layer size:} We tested various hidden layer sizes, including 32, 64, and 128 neurons per layer. A hidden layer size of 64 neurons struck a balance between model complexity and performance.

\item {Batch size:} We tested different batch sizes, ranging from 32 to 128. A batch size of 64 was found to be suitable. 

\item In addition, as the baseline GNN is GCN~\cite{kipf2017semi} that is a 2-layer neural network. We hence use $k=2$ in CXGNN.
\end{itemize}


\noindent {\bf GNNExplainer:} 
Its hyperparameters are detailed as follows:
\begin{itemize}
\item A dictionary to store coefficients for the entropy term and the size term for learning edge mask and node feature mask. The chosen settings are:
\begin{itemize}
        \item edge: entropy: 1.0, and size: 0.005
        \item feature: entropy: 0.1, and size: 1.0
\end{itemize}
\item The number of epochs for training the explanation model: 200.

\item Learning Rate: Used in the Adam optimizer for training, set to 0.01.

\item The number of hops to consider when explaining a node prediction. It is equal to the number of layers in the GNN.
\end{itemize}

\noindent {\bf PGMexplainer:} 
It is for explaining predictions made by GNNs using Probabilistic Graphical Models (PGMs), provides these hyperparameters:

\begin{itemize}
\item Number of perturbed graphs to generate: 10 for graph-level explanations

\item How node features are perturbed: mean, for graph-level explanations.

\item The probability that a node's features are perturbed: 0.5

\item The threshold for the chi-square independence test: 0.05

\item Threshold for the difference in predicted class probability: 0.1

\item Number of nodes to include in PGM: all nodes given by the chi-square test are kept.
\end{itemize}

\noindent {\bf Guidedbp:} 
It is a form of Guided Backpropagation for explaining graph predictions, contains several hyperparameters and method-specific parameters:

\begin{itemize}
\item The loss function used to train the model: cross entropy.
\item The number of hops for the $k$-hop subgraph, is implicitly set to the number of layers in the GNN (i.e., $k=2$ in our results).
\end{itemize}

\noindent {\bf GEM:} The GEM method has the  below hyperparameters:
\begin{itemize}
\item Optimization Parameters:
\begin{itemize}
\item Learning Rate: 0.1, Gradient Clipping: 2, Batch size: 20, Number of Epochs: 100, Optimizer: "Adam"
\end{itemize}
\item Model Parameters:

\begin{itemize}
\item Hidden Dimension: 20, Output Dimension: 20, Number of Graph Convolution Layers: 2
\end{itemize}
\item Explainer Parameters:
\begin{itemize}
\item Iterations to find alignment matrix: 1000, Number of mini-batches: 10
\end{itemize}
\end{itemize}

\noindent {\bf RCExp:} The reinforced causal explainer for graph neural networks has the  below hyperparameters:
\begin{itemize}
\item Optimization Parameters:
\begin{itemize}
\item Learning Rate: 0.01, Weight Decay: 0.005, Number of Epochs: 100, Optimizer: "Adam"
\end{itemize}
\item Model Parameters:
\begin{itemize}
\item Hidden Dimension: 64, 32, Output Dimension: 2, Number of Graph Convolution Layers: 2
\end{itemize}
\item Explainer Parameters:
\begin{itemize}
\item Output size of edge action rep generator: 64, Edge attribute dimension: 32
\end{itemize}
\end{itemize}

\noindent {\bf Orphicx:} The causality-inspired latent variable Model for interpreting graph neural networks has the  below hyperparameters:
\begin{itemize}
\item Optimization Parameters:
\begin{itemize}
\item Learning Rate: 0.0005, Weight Decay: 0.01, Number of Epochs: 100, Optimizer: "Adam", Early Stopping Patience: 20 
\end{itemize}
\item Model Parameters:
\begin{itemize}
\item Hidden Dimension: 32, Decoder Hidden Dimension: 16, Dropout rate: 0.5, Output Dimension: 108, Number of Graph Convolution Layers: 2
\end{itemize}
\item Explainer Parameters: Number of causal factors: 5
\end{itemize}

\begin{figure}[!t]
  \centering
  \begin{subfigure}[b]{0.32\textwidth}
    \includegraphics[width=\linewidth]{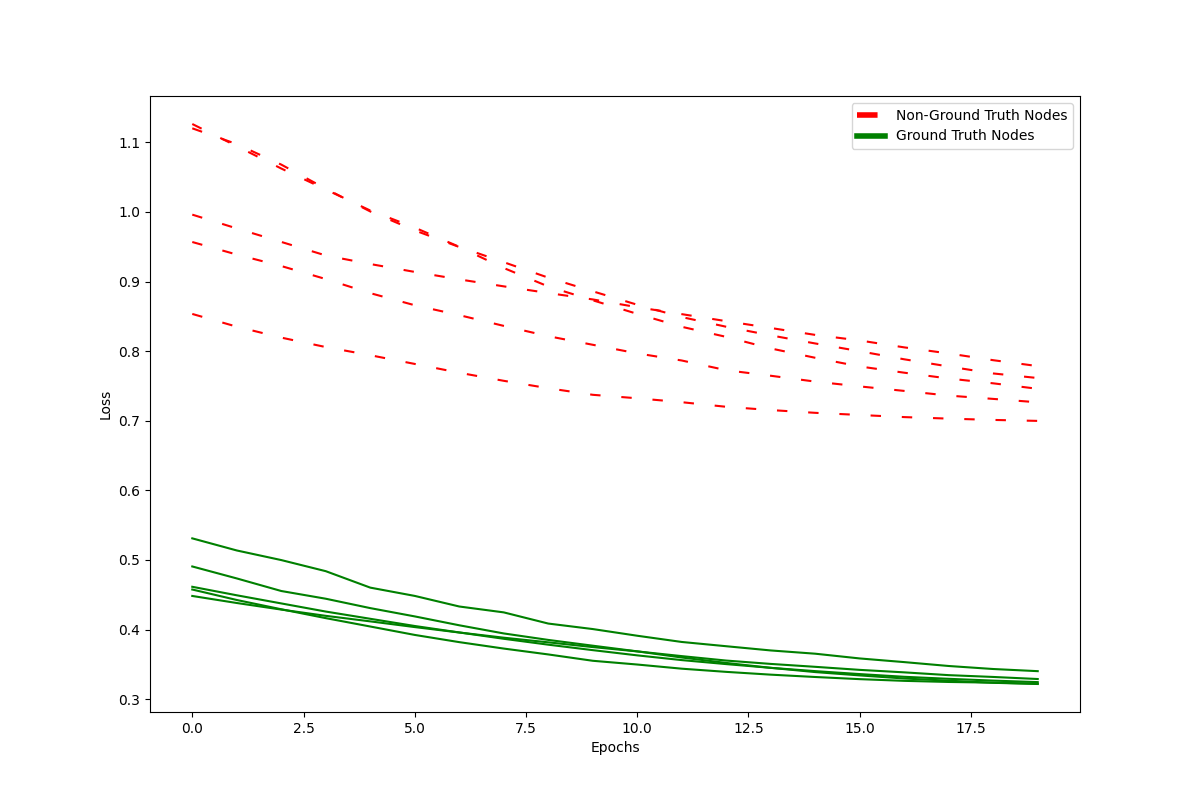}
  \end{subfigure}
  \begin{subfigure}[b]{0.32\textwidth}
    \includegraphics[width=\linewidth]{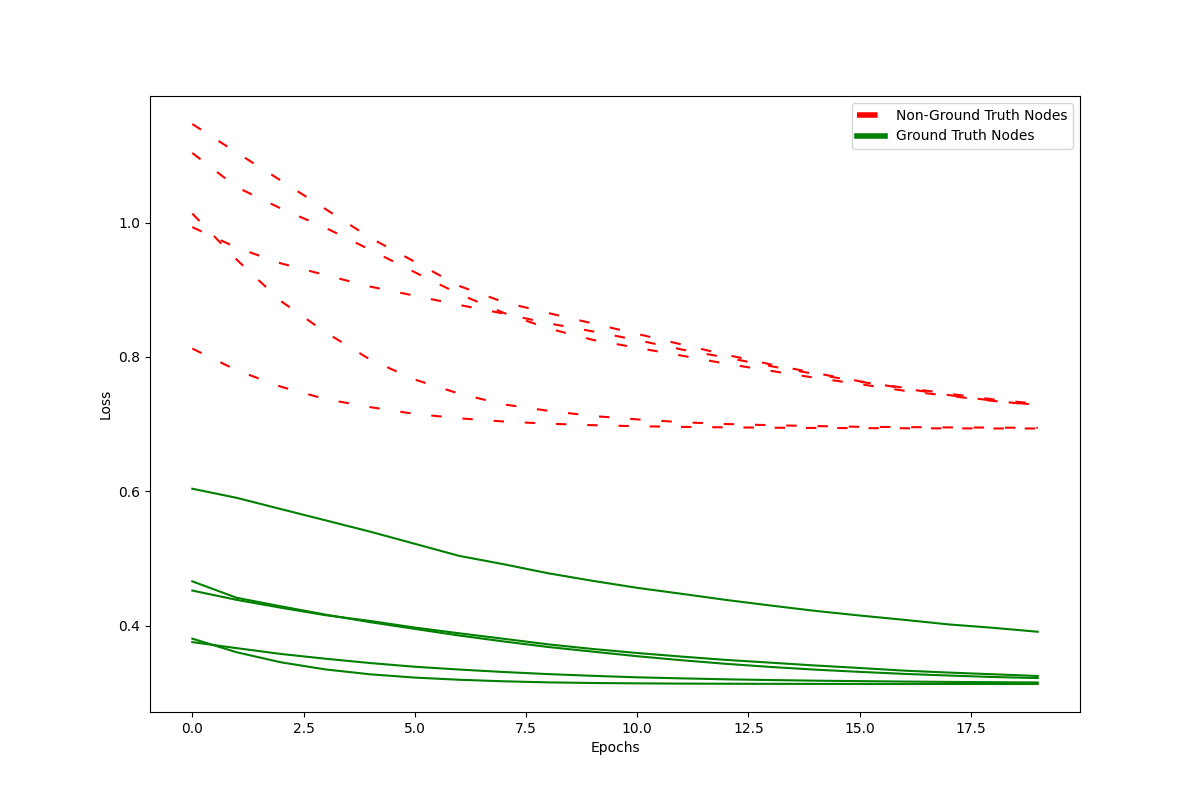}
  \end{subfigure}
  \begin{subfigure}[b]{0.32\textwidth}
    \includegraphics[width=\linewidth]{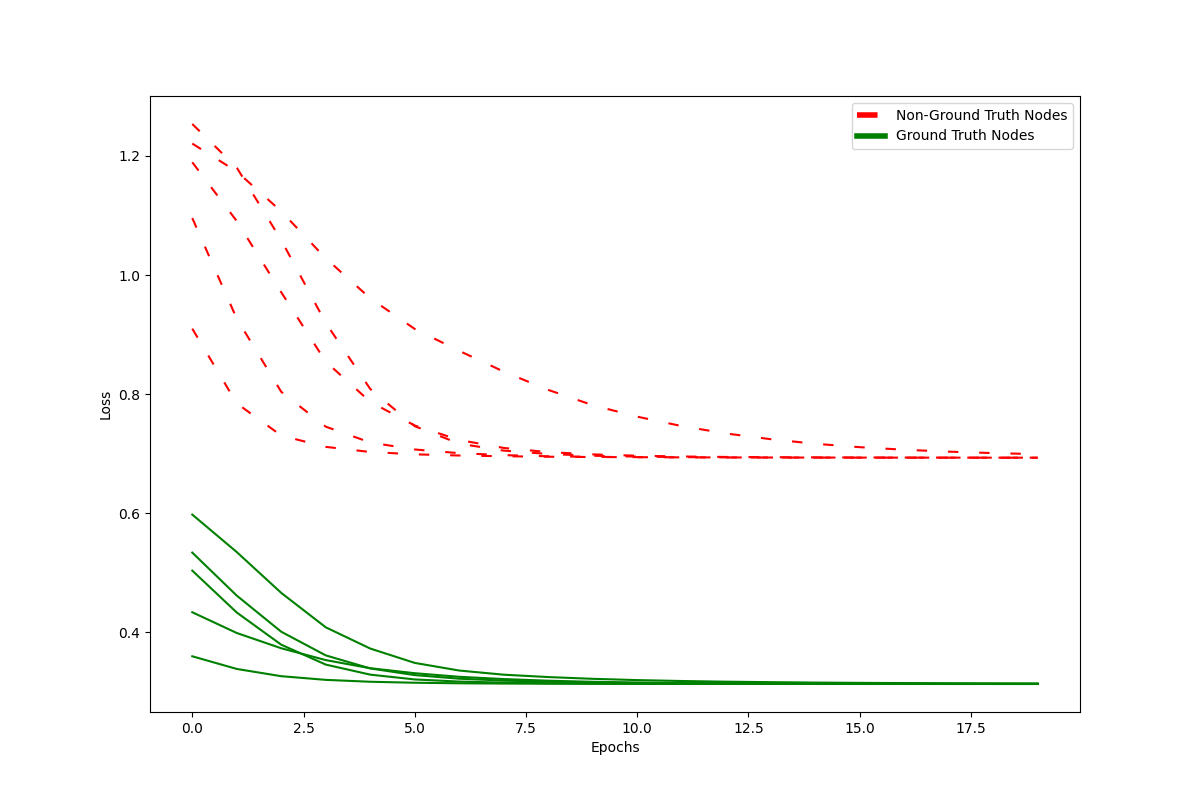}
  \end{subfigure}
    \begin{subfigure}[b]{0.32\textwidth}
    \includegraphics[width=\linewidth]{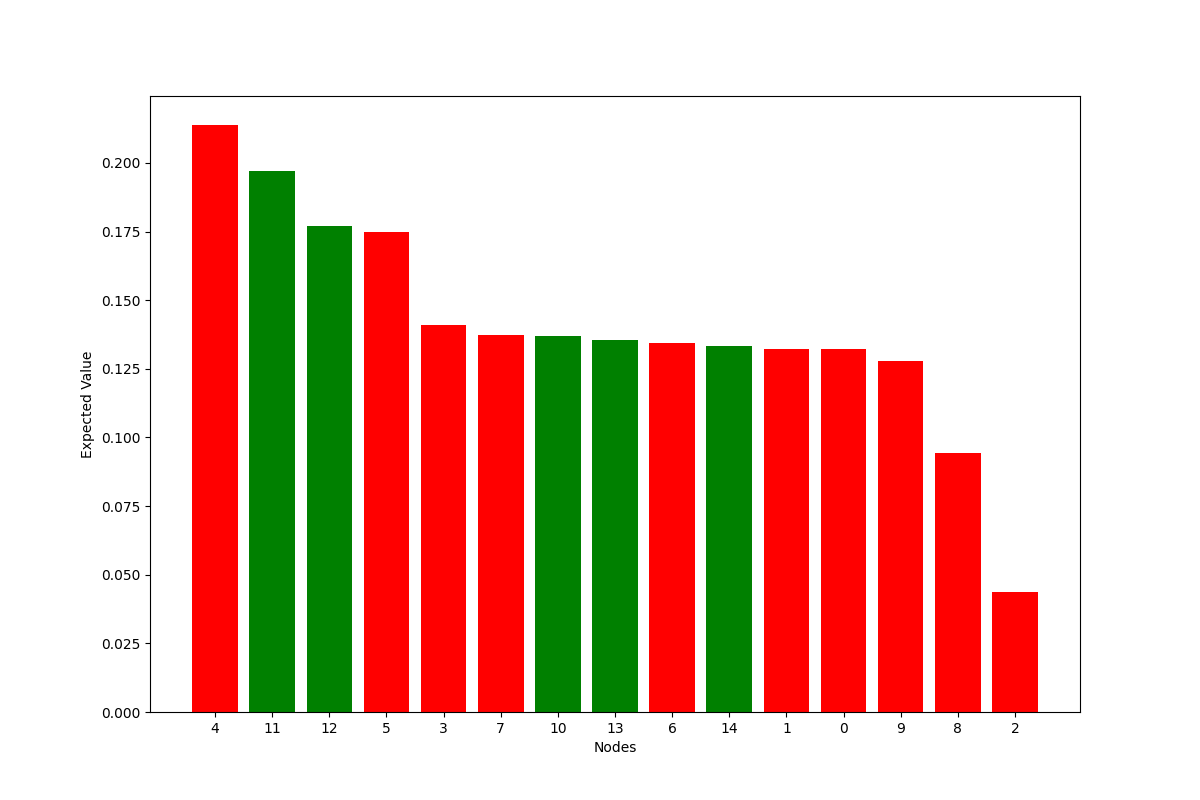}
  \end{subfigure}
  \begin{subfigure}[b]{0.32\textwidth}
    \includegraphics[width=\linewidth]{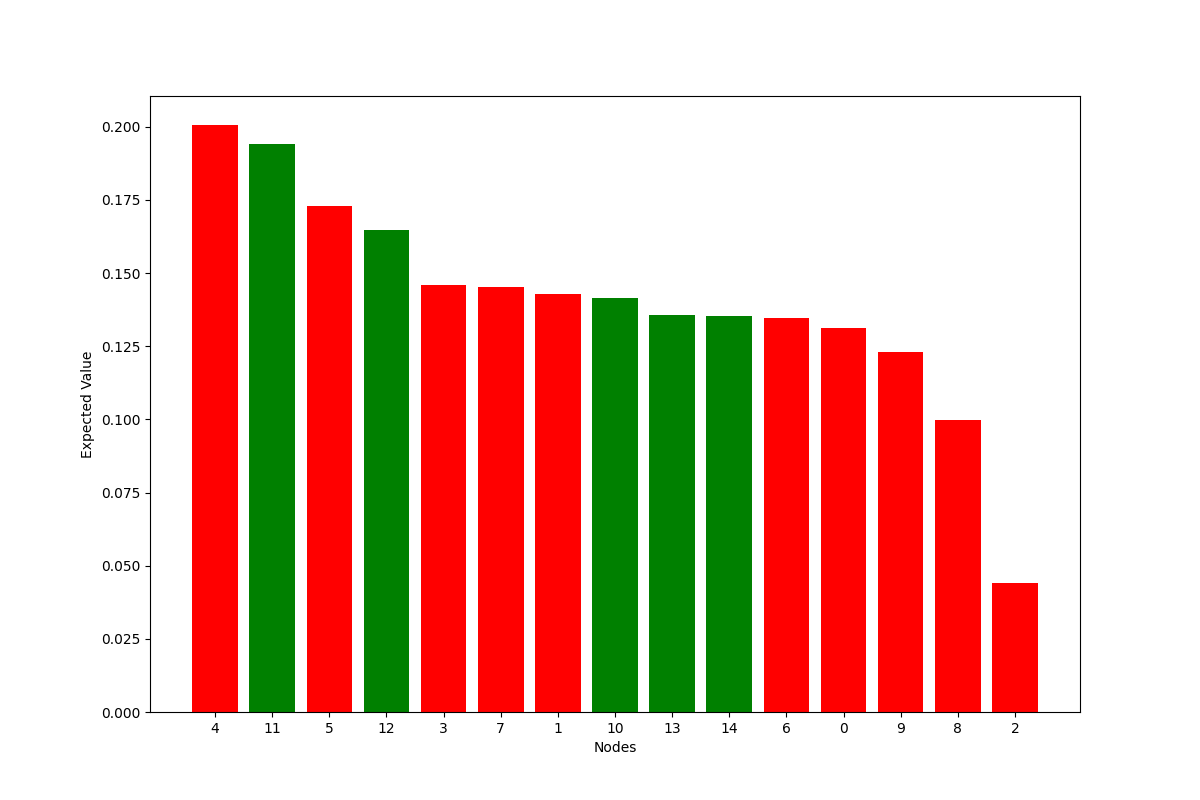}
  \end{subfigure}
  \begin{subfigure}[b]{0.32\textwidth}
    \includegraphics[width=\linewidth]{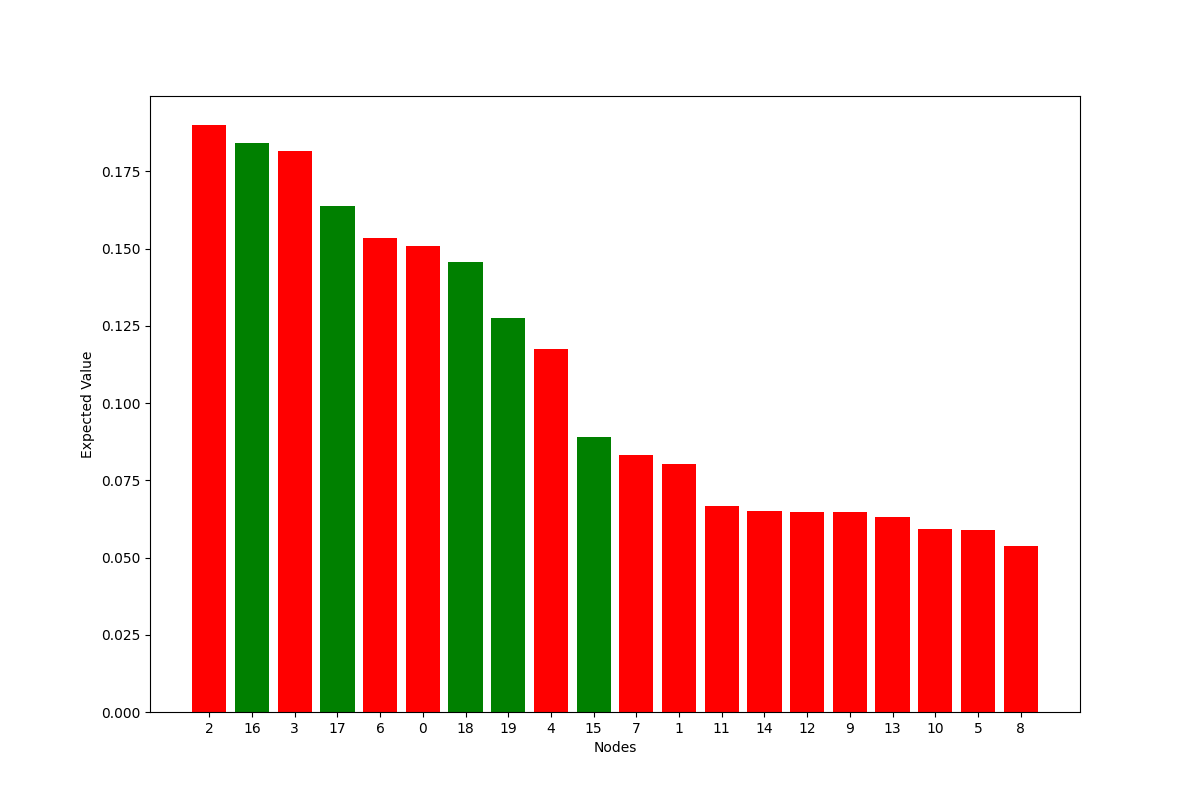}
  \end{subfigure}
\caption{More results on BA+House. Top 3 figures: Loss curves of training the GNN-NCMs on the groundtruth nodes ({\color{green} green curves}) and non-groundtruth ones ({\color{red} red curves}); Bottom 3 figures: Node expressivity distributions on three unsuccessful graphs. {\color{green} Green bars} ({\color{red} red bars}) correspond to nodes that are (NOT) in the groundtruth. Same meaning for all the below figures.
}
\label{fig:vis_bahouse}
\end{figure}

\begin{figure}
  \centering
  \begin{subfigure}[b]{0.32\textwidth}
    \includegraphics[width=\linewidth]{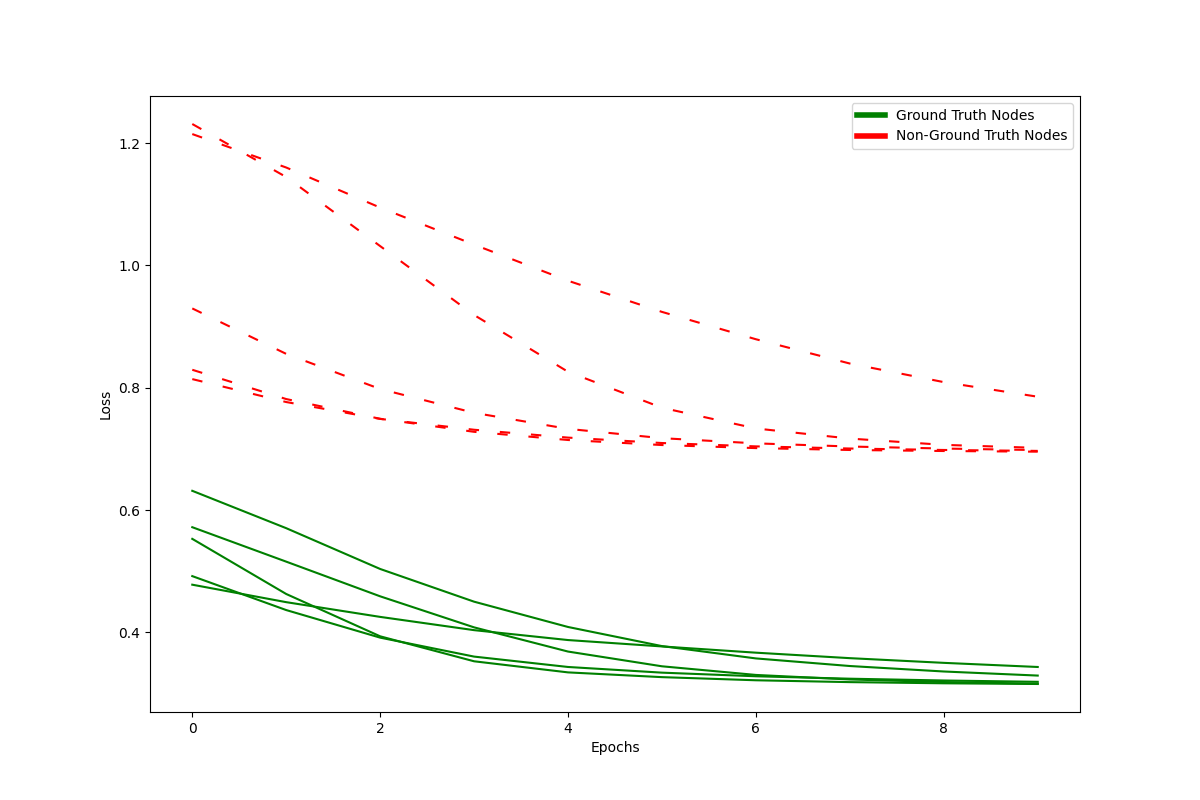}
  \end{subfigure}
  \begin{subfigure}[b]{0.32\textwidth}
    \includegraphics[width=\linewidth]{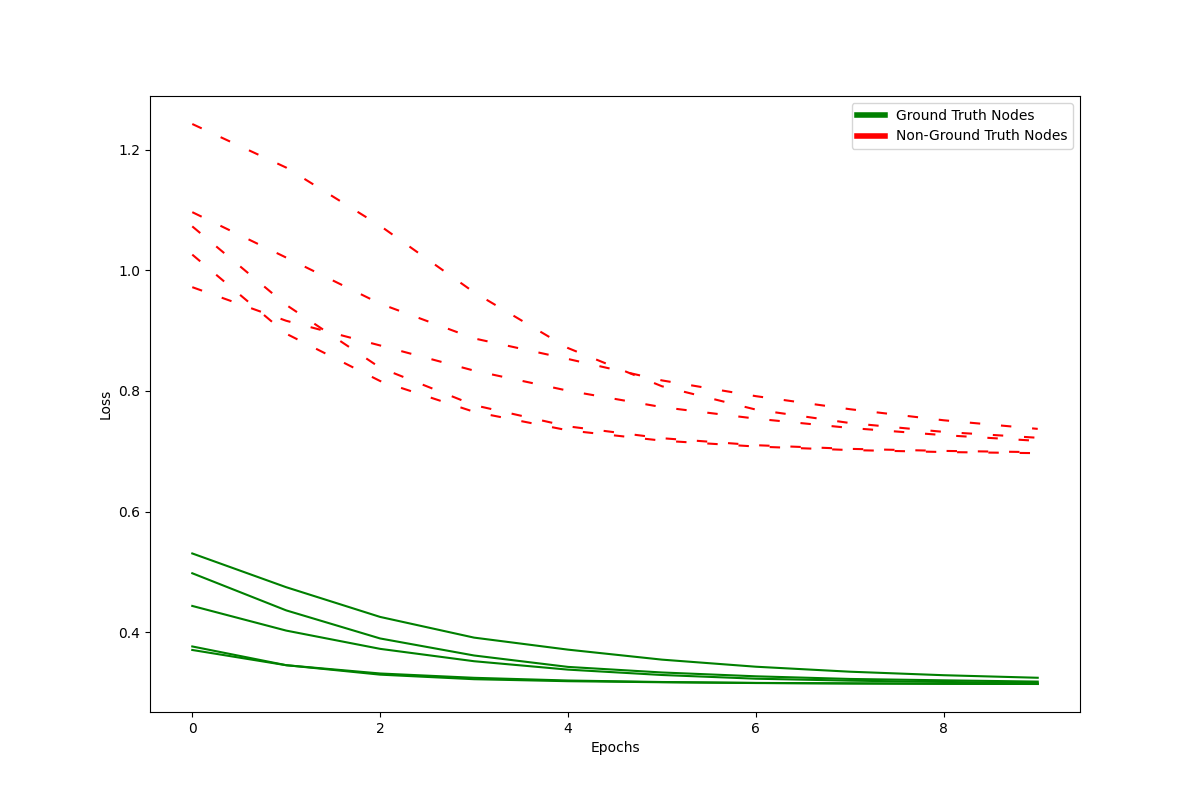}
  \end{subfigure}
    \begin{subfigure}[b]{0.32\textwidth}
    \includegraphics[width=\linewidth]{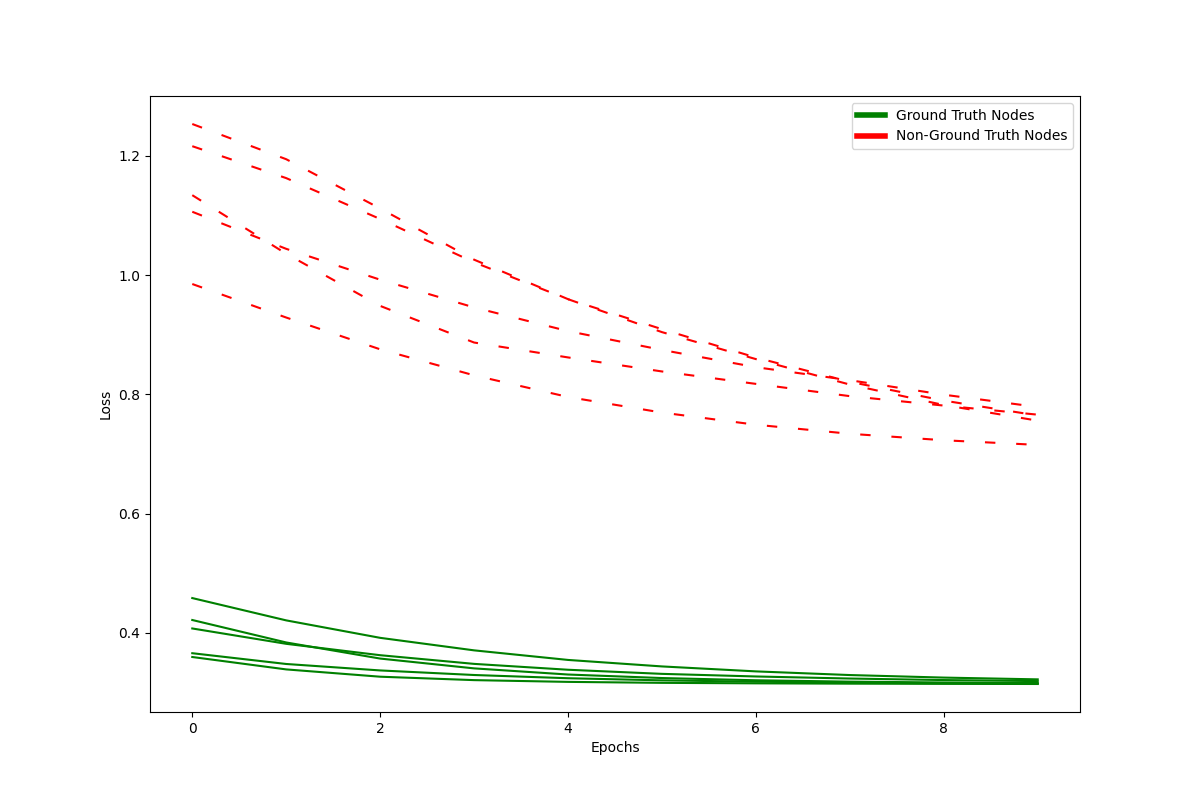}
  \end{subfigure}
   \label{fig2:sub1}
  \begin{subfigure}[b]{0.32\textwidth}
    \includegraphics[width=\linewidth]{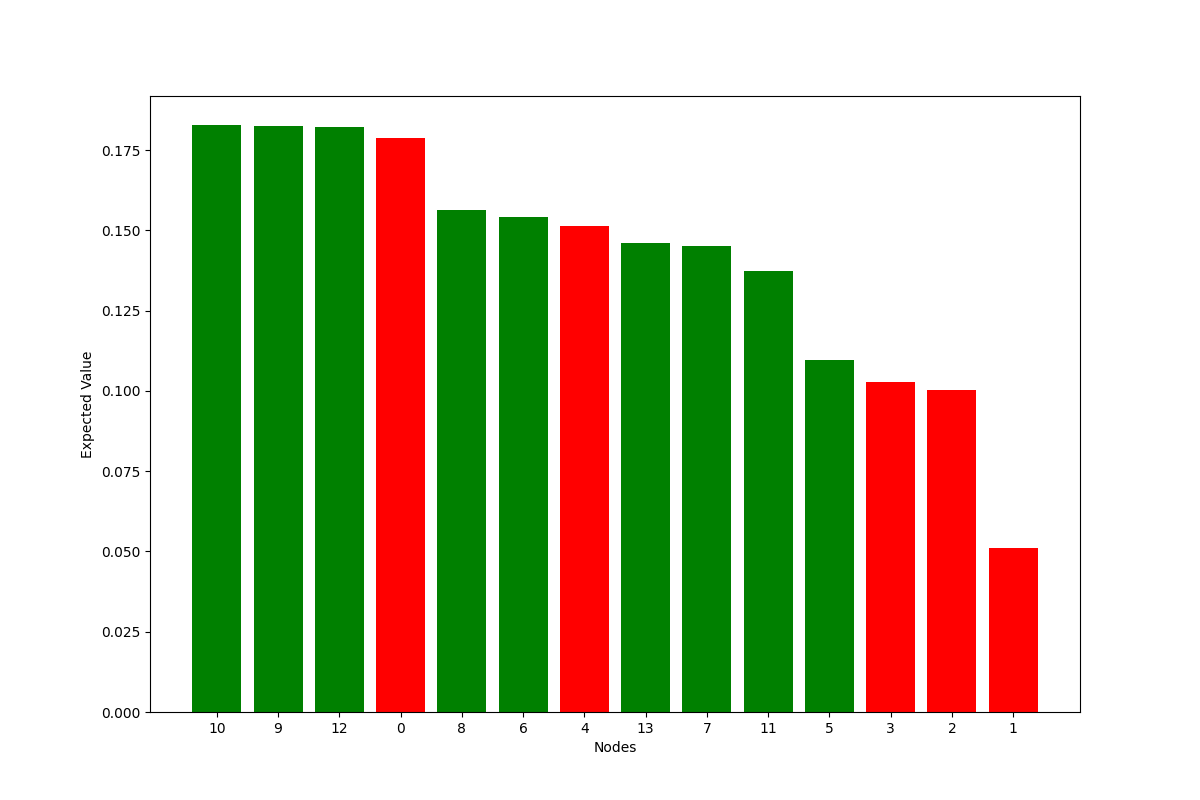}
  \end{subfigure}
  \begin{subfigure}[b]{0.32\textwidth}
    \includegraphics[width=\linewidth]{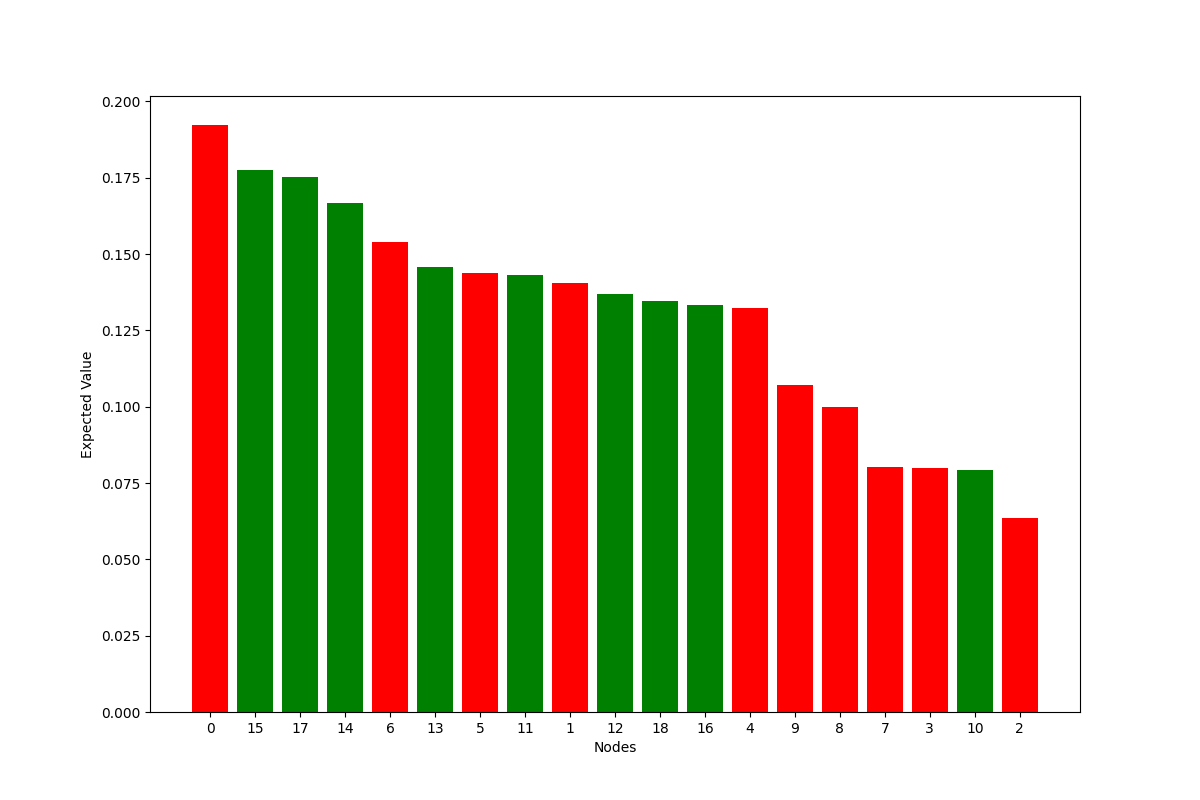}
  \end{subfigure}
  \begin{subfigure}[b]{0.32\textwidth}
    \includegraphics[width=\linewidth]{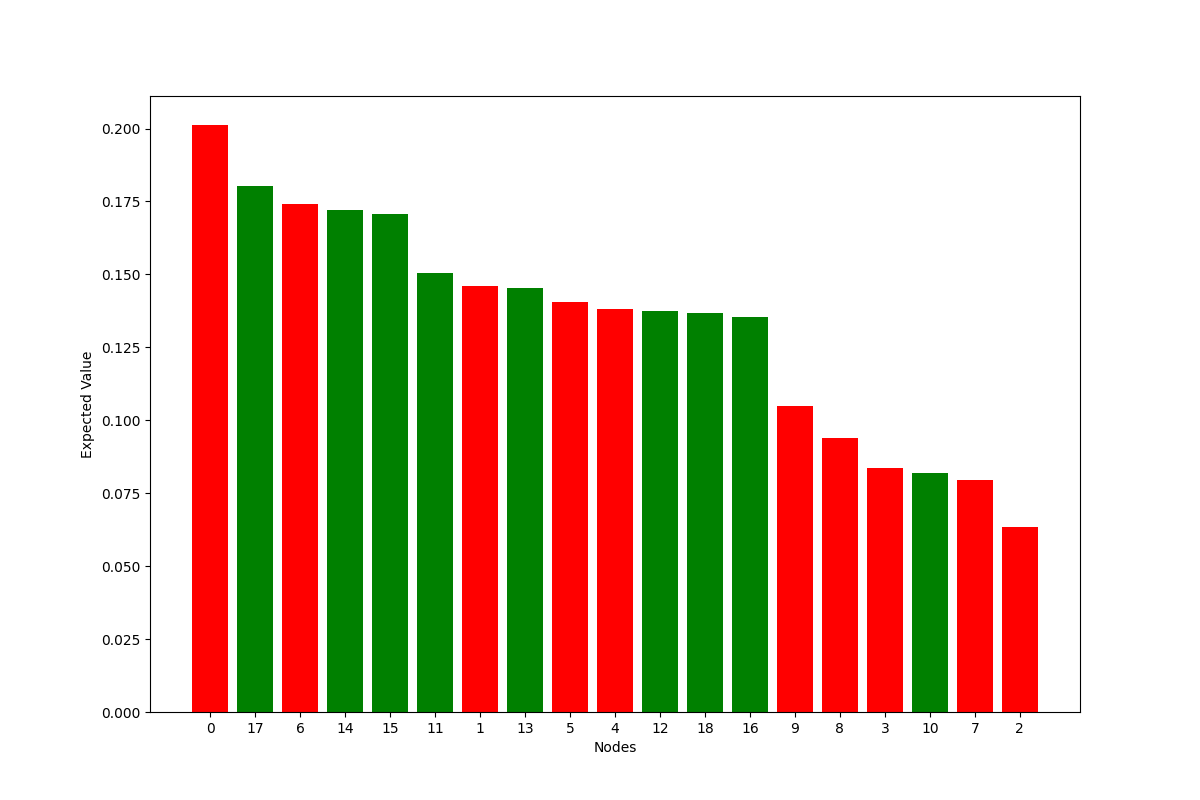}
  \end{subfigure}
\caption{More results on BA+Grid. 
}
\label{fig:vis_bagrid}
\end{figure}

\begin{figure}
  \centering
  \begin{subfigure}[b]{0.32\textwidth}
    \includegraphics[width=\linewidth]{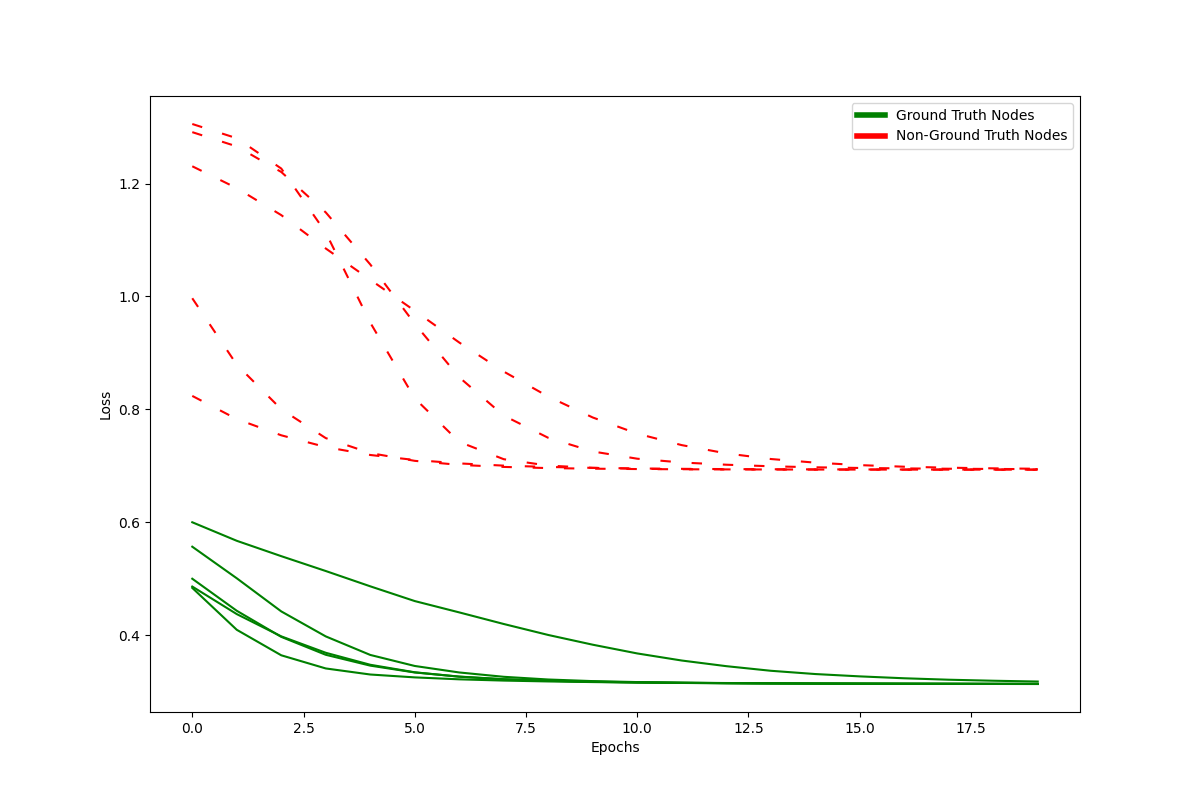}
  \end{subfigure}
  \begin{subfigure}[b]{0.32\textwidth}
    \includegraphics[width=\linewidth]{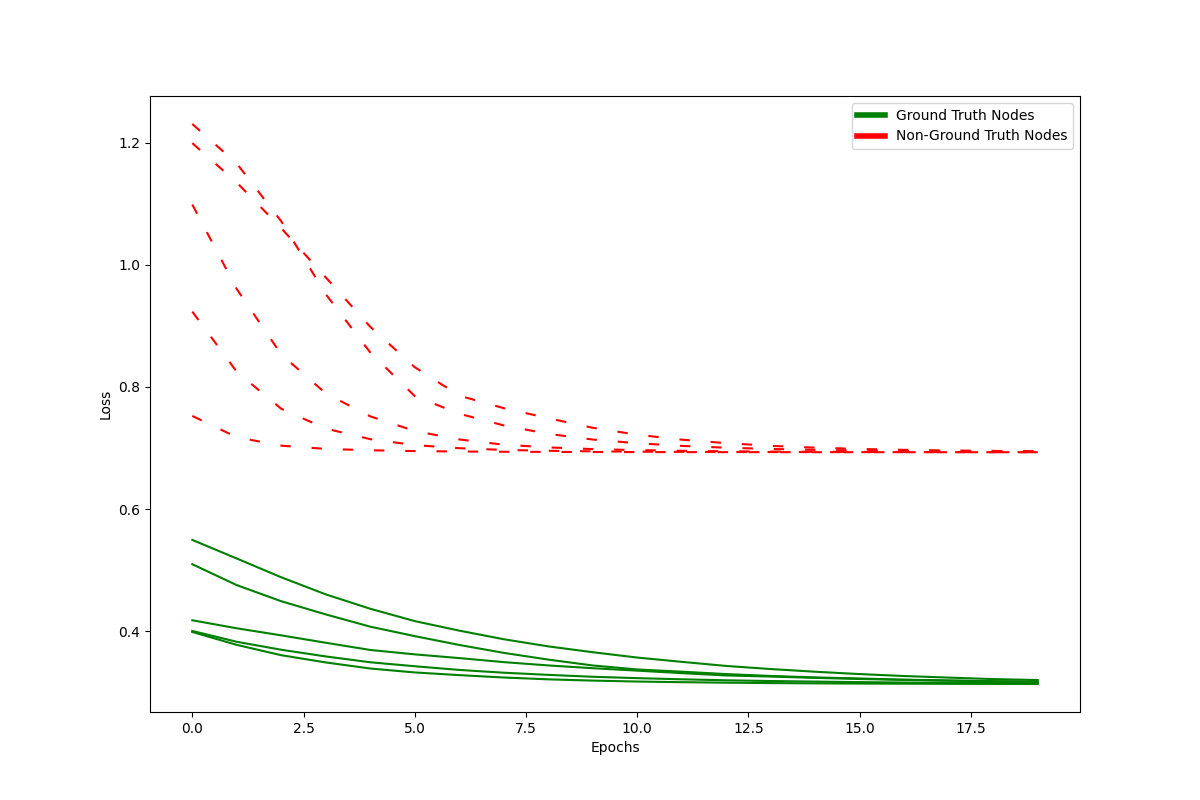}
  \end{subfigure}
    \begin{subfigure}[b]{0.32\textwidth}
    \includegraphics[width=\linewidth]{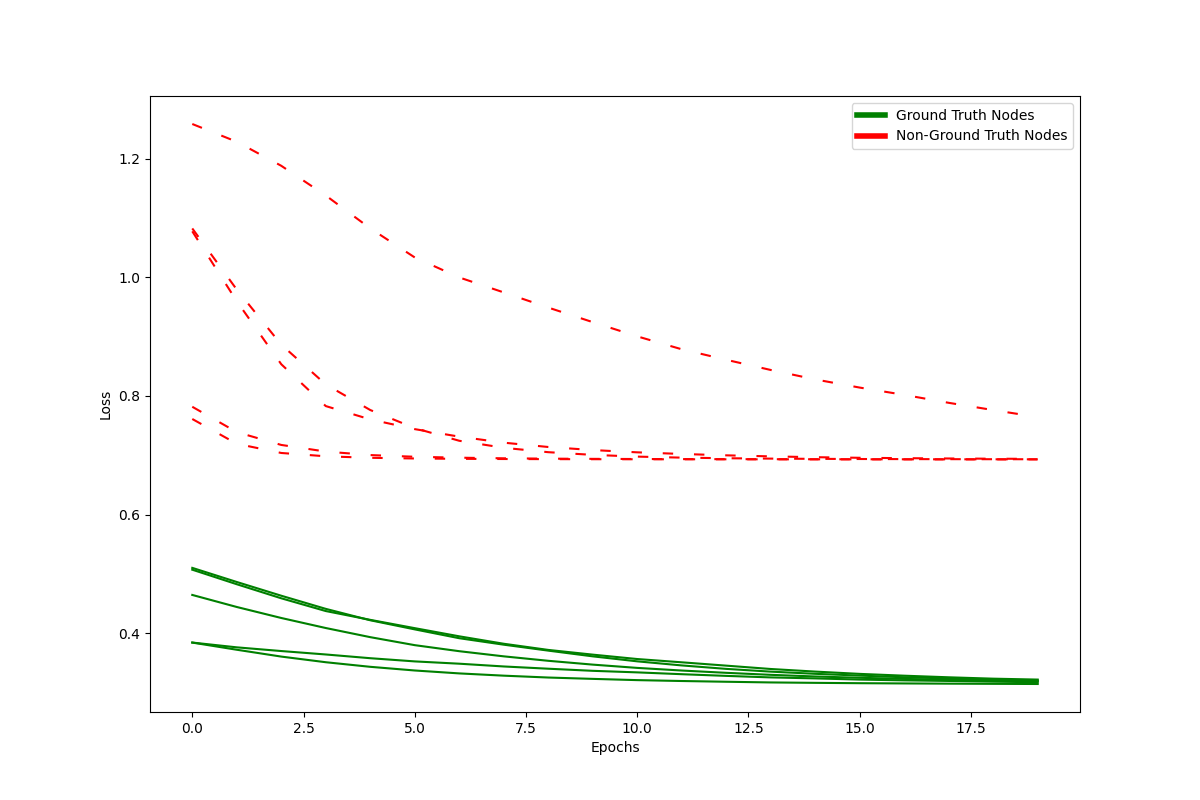}
  \end{subfigure}
  \begin{subfigure}[b]{0.32\textwidth}
    \includegraphics[width=\linewidth]{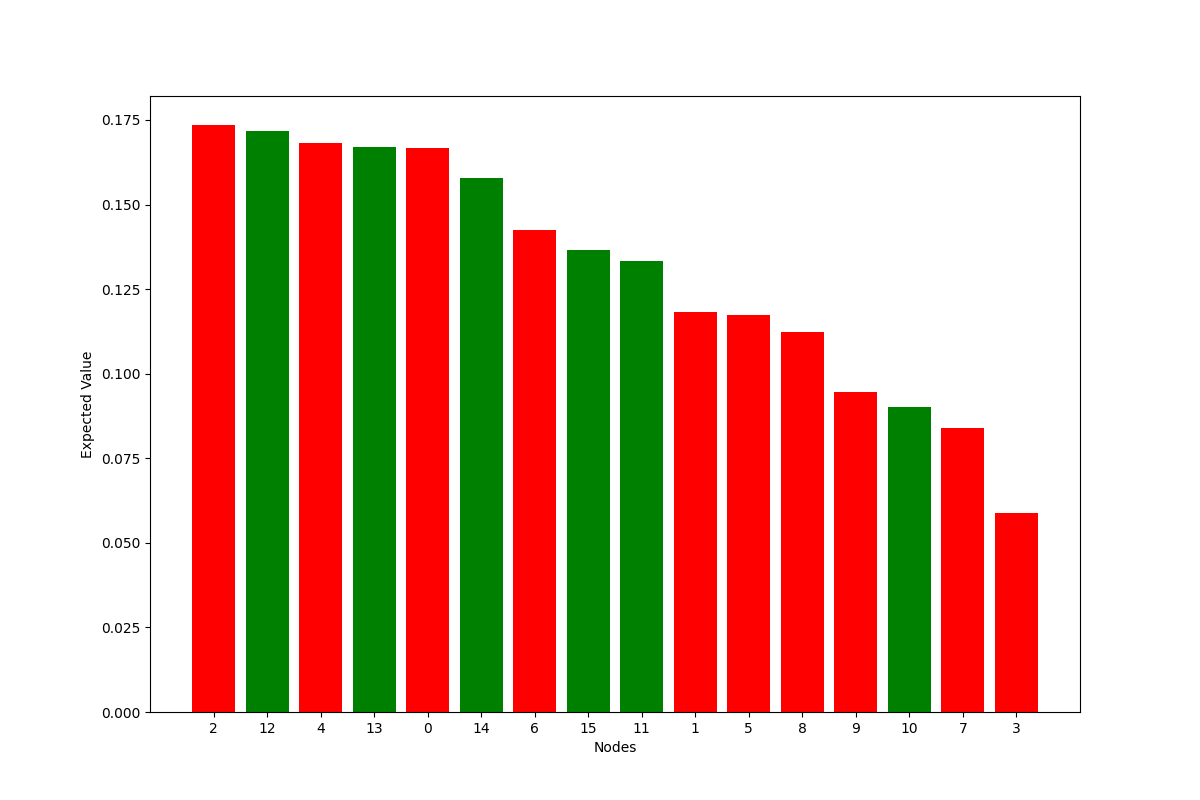}
  \end{subfigure}
  \begin{subfigure}[b]{0.32\textwidth}
    \includegraphics[width=\linewidth]{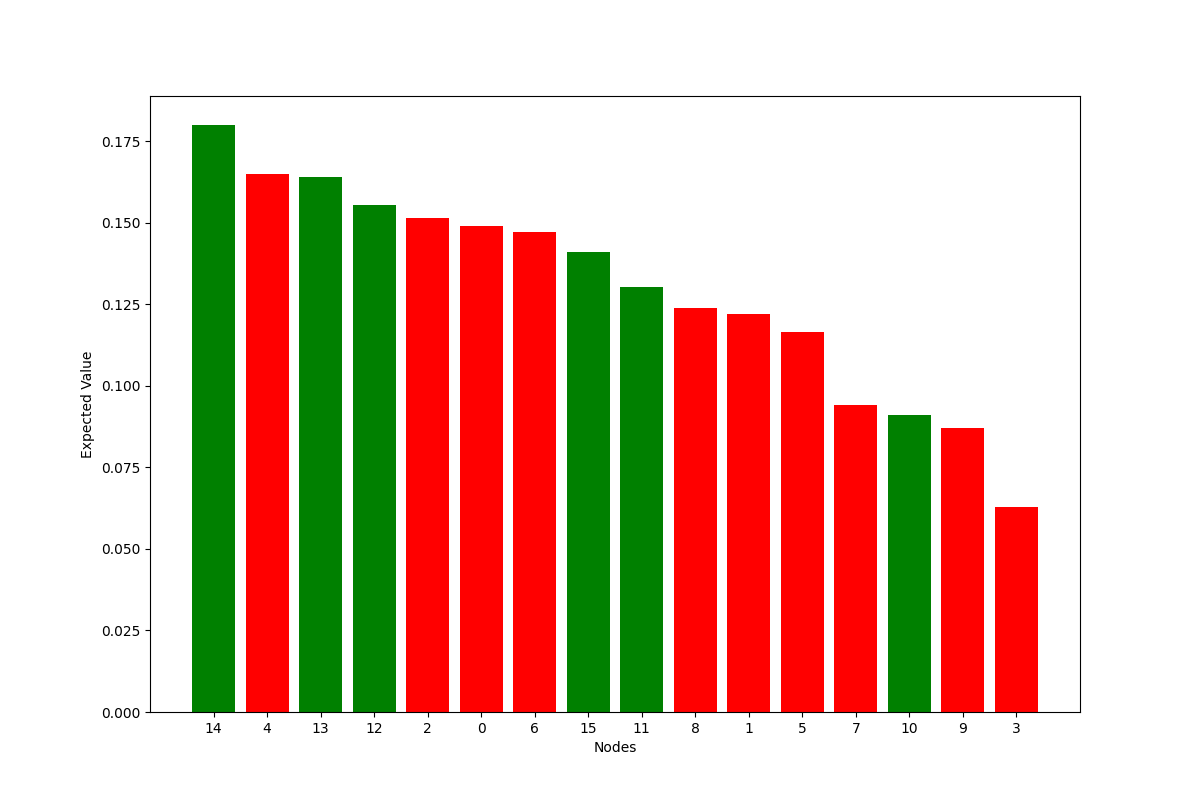}
  \end{subfigure}
  \begin{subfigure}[b]{0.32\textwidth}
    \includegraphics[width=\linewidth]{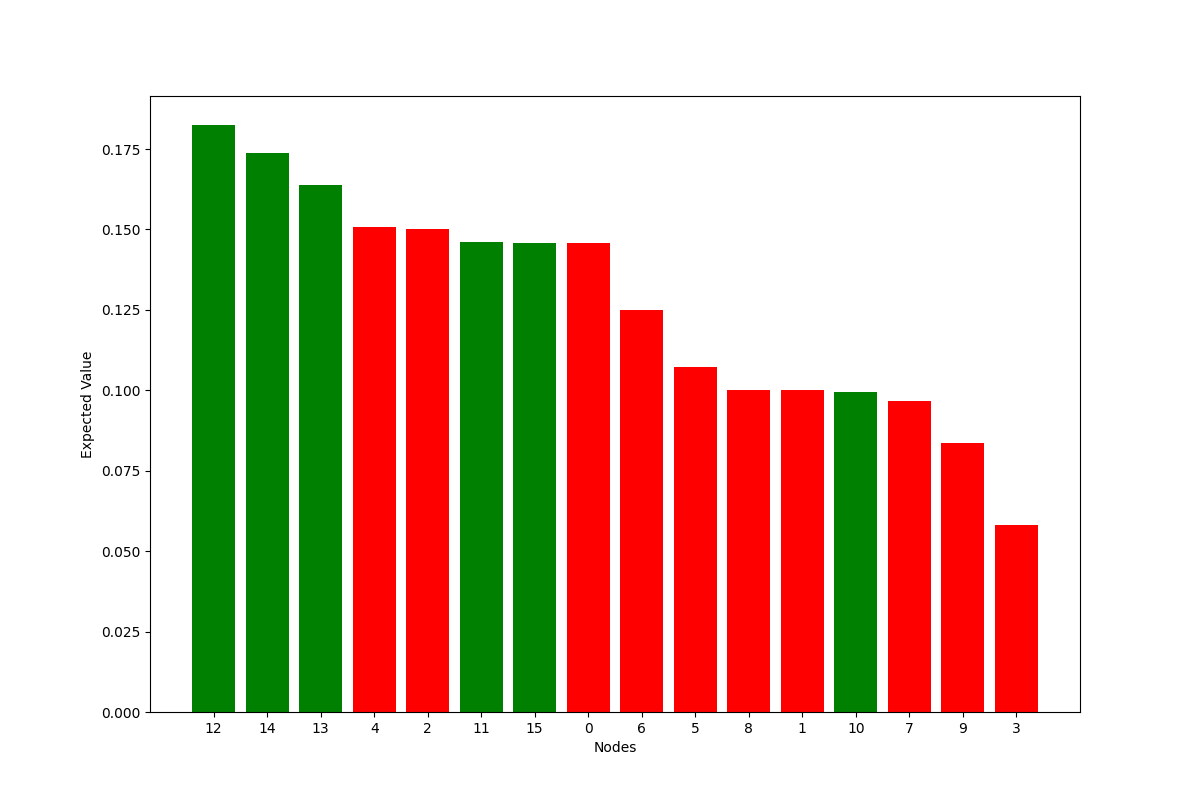}
  \end{subfigure}
\caption{More results on BA+Cycle. 
}
\end{figure}

\begin{figure}
  \centering
    \begin{subfigure}[b]{0.32\textwidth}
    \includegraphics[width=\linewidth]{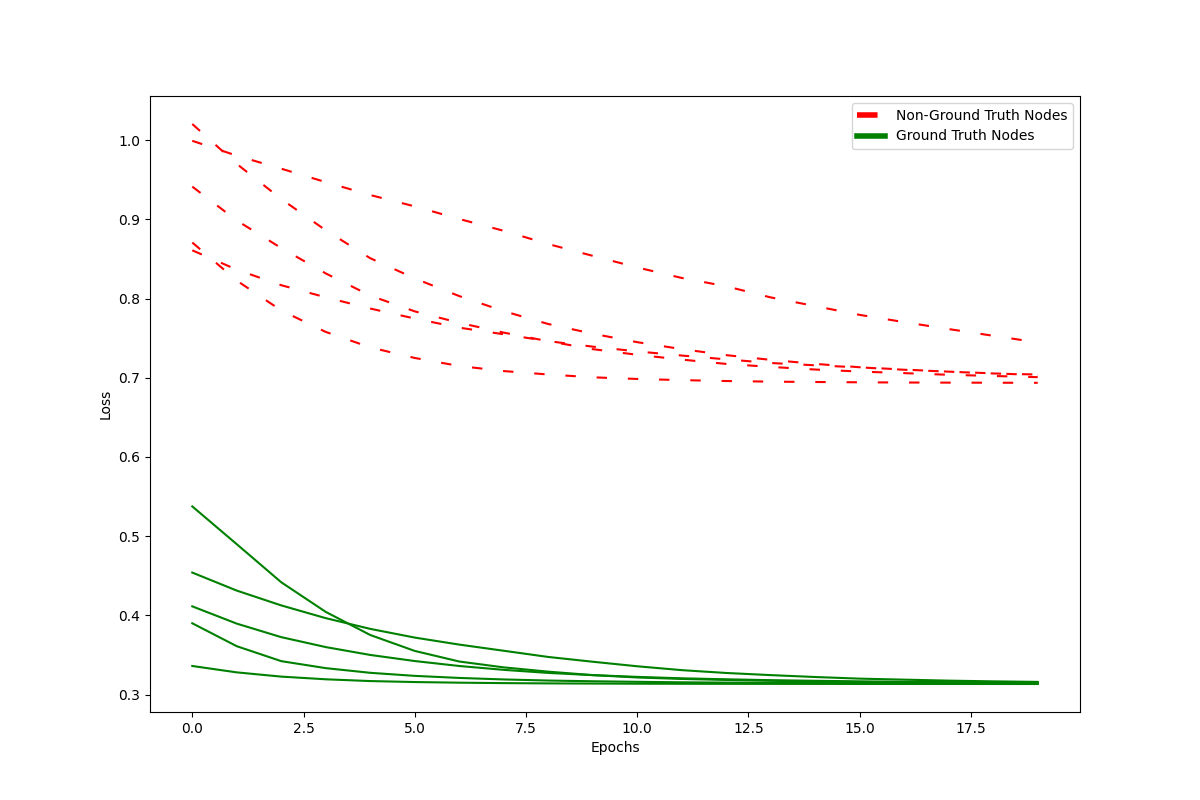}
  \end{subfigure}
  \begin{subfigure}[b]{0.32\textwidth}
    \includegraphics[width=\linewidth]{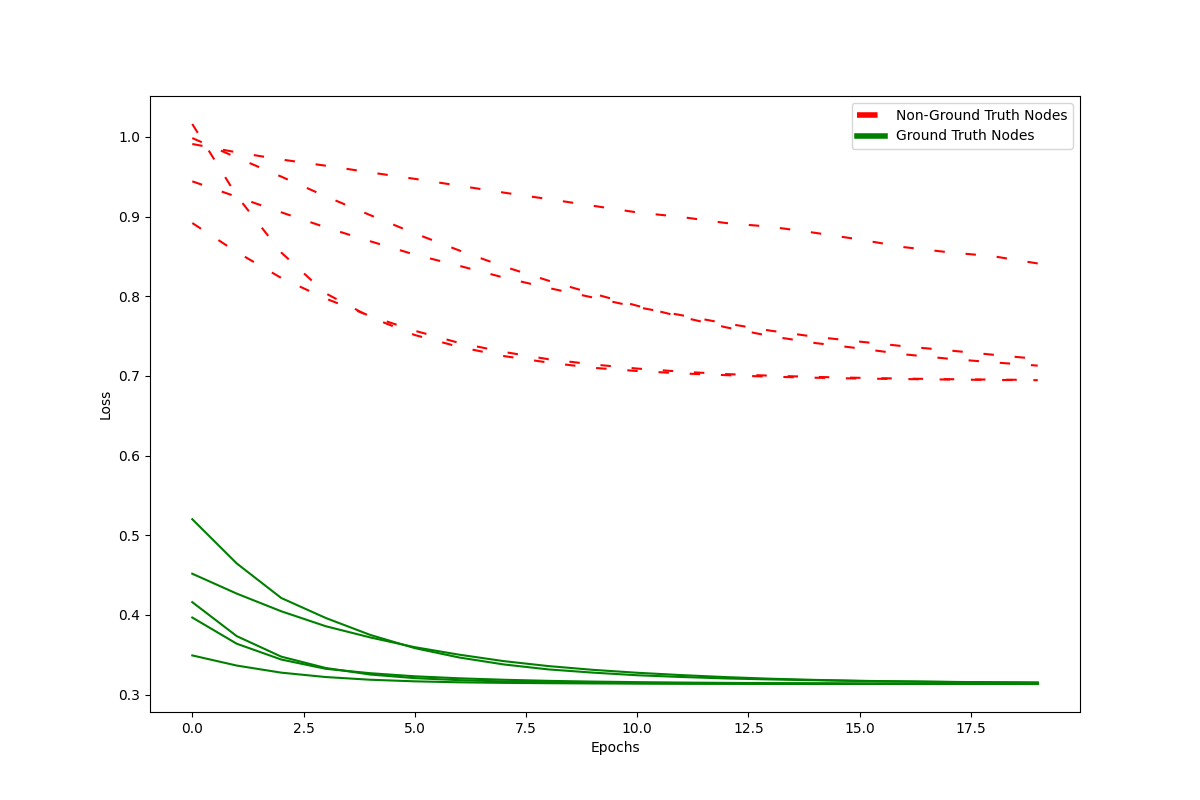}
  \end{subfigure}
  \begin{subfigure}[b]{0.32\textwidth}
    \includegraphics[width=\linewidth]{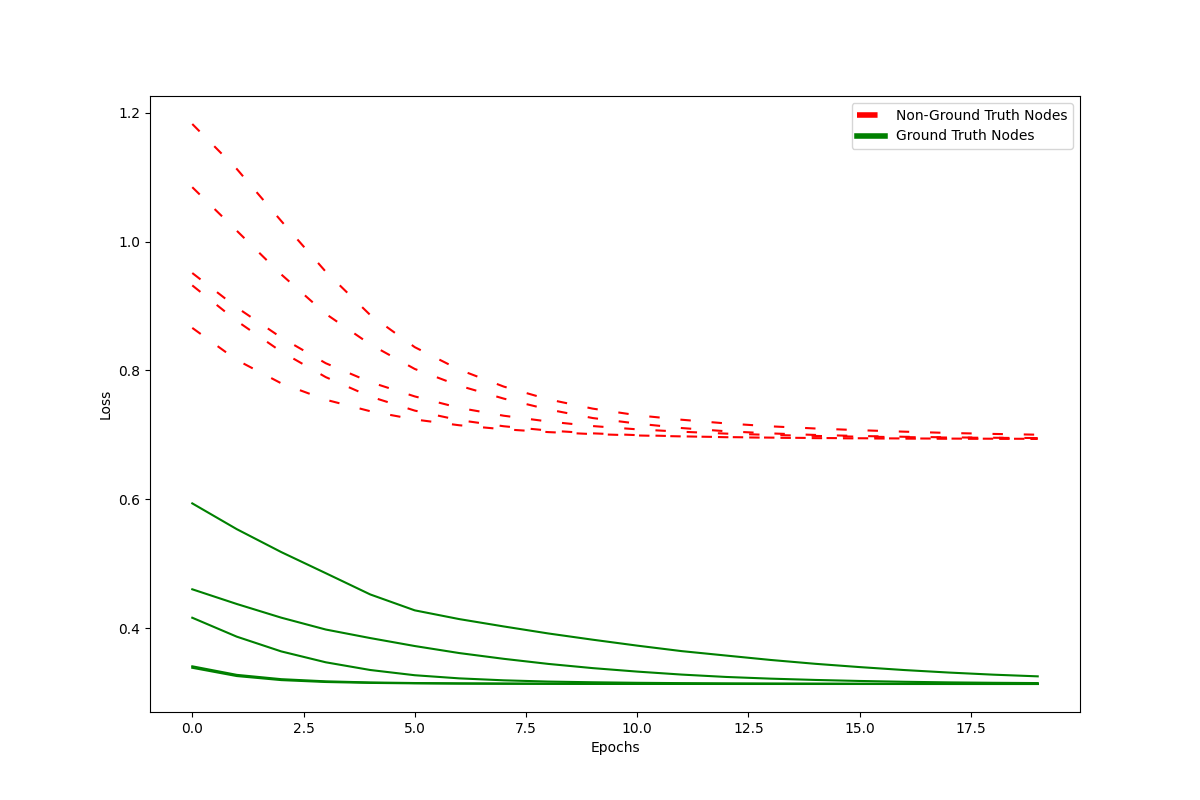}
  \end{subfigure}
    \begin{subfigure}[b]{0.32\textwidth}
    \includegraphics[width=\linewidth]{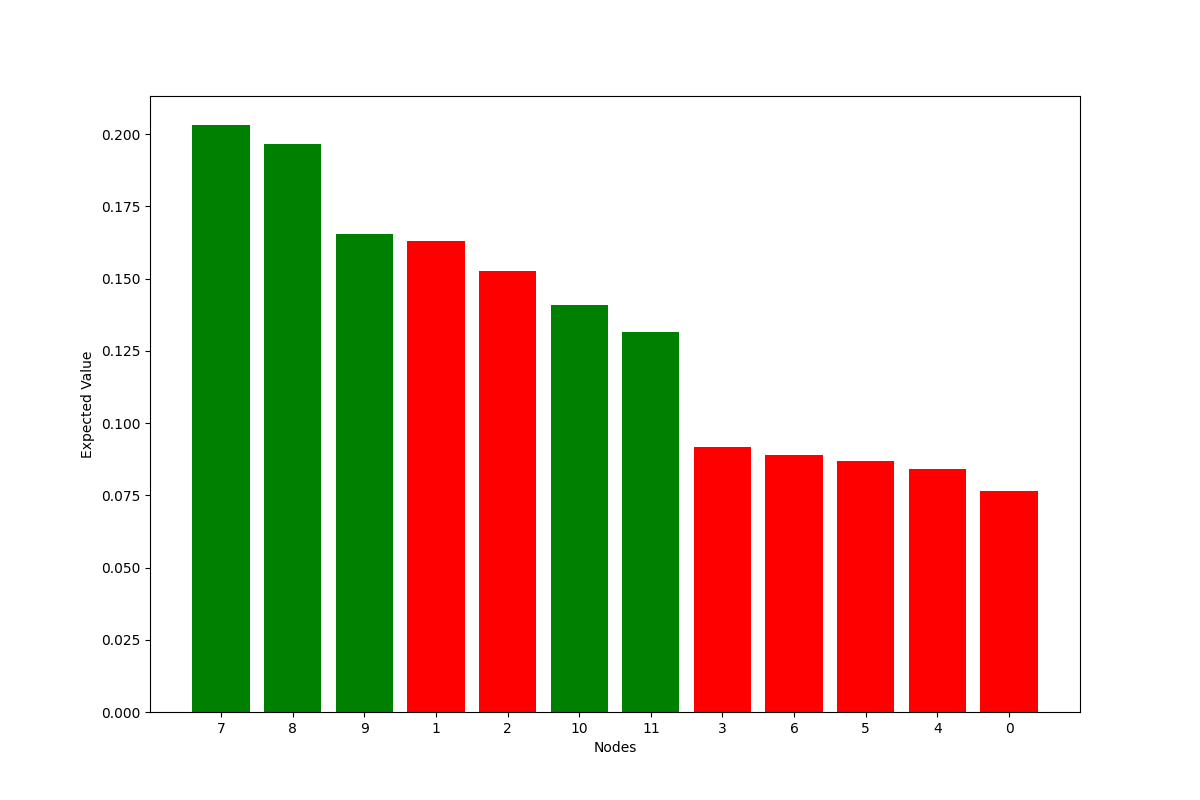}
  \end{subfigure}
  \begin{subfigure}[b]{0.32\textwidth}
    \includegraphics[width=\linewidth]{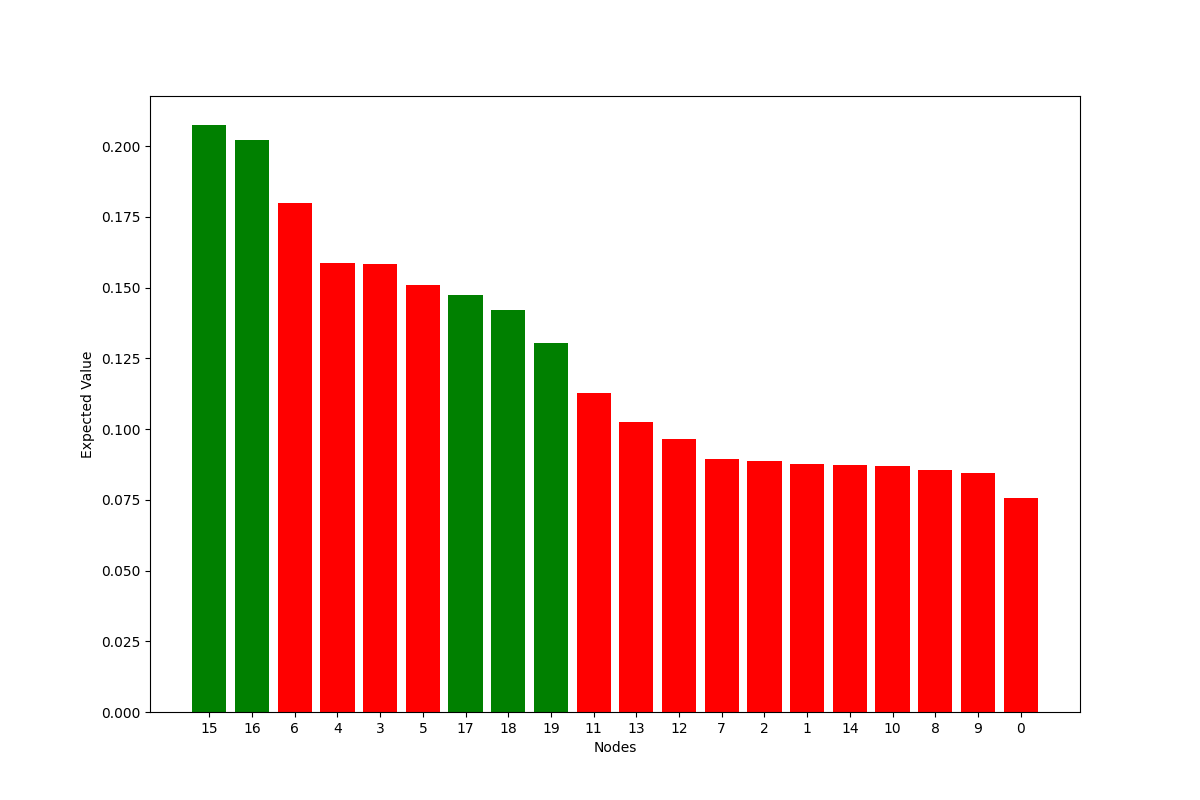}
  \end{subfigure}
  \begin{subfigure}[b]{0.32\textwidth}
    \includegraphics[width=\linewidth]{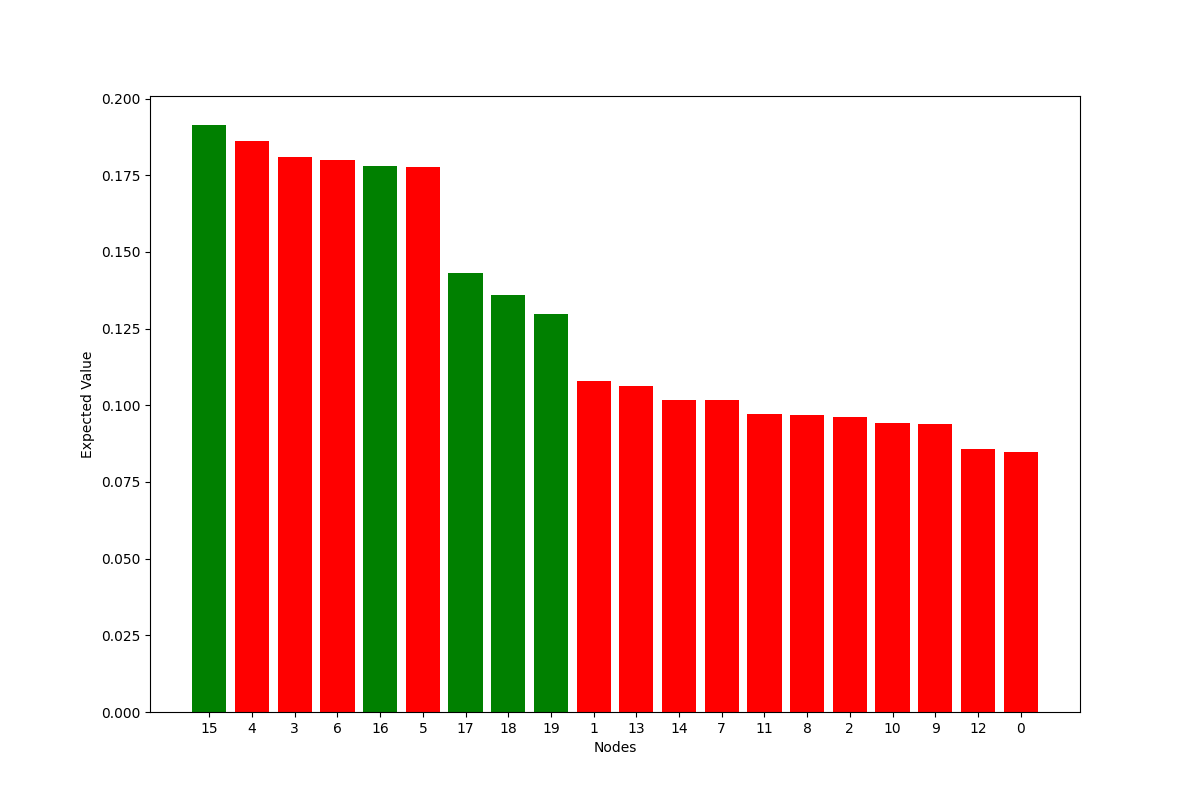}
  \end{subfigure}
\caption{More results on Tree+House. 
}
\label{fig:vis_treehouse}
\end{figure}

\begin{figure}
  \centering
  \begin{subfigure}[b]{0.32\textwidth}
    \includegraphics[width=\linewidth]{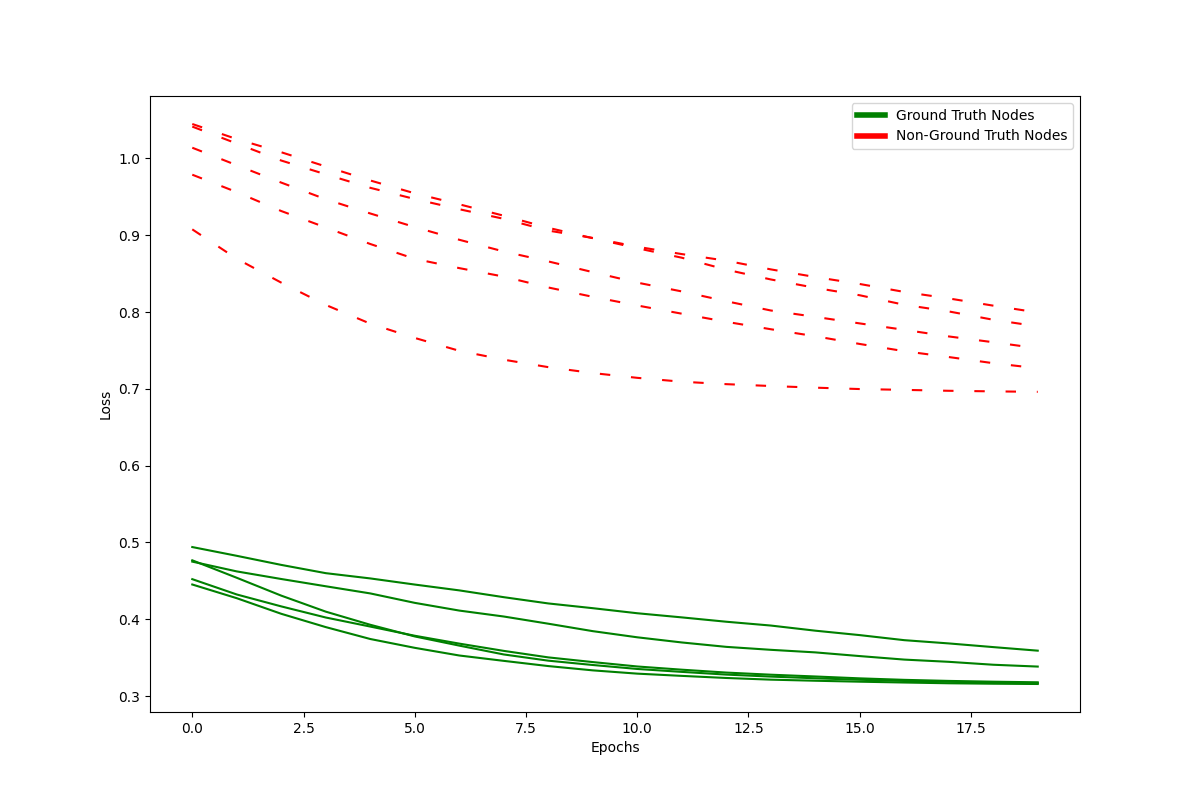}
  \end{subfigure}
  \begin{subfigure}[b]{0.32\textwidth}
    \includegraphics[width=\linewidth]{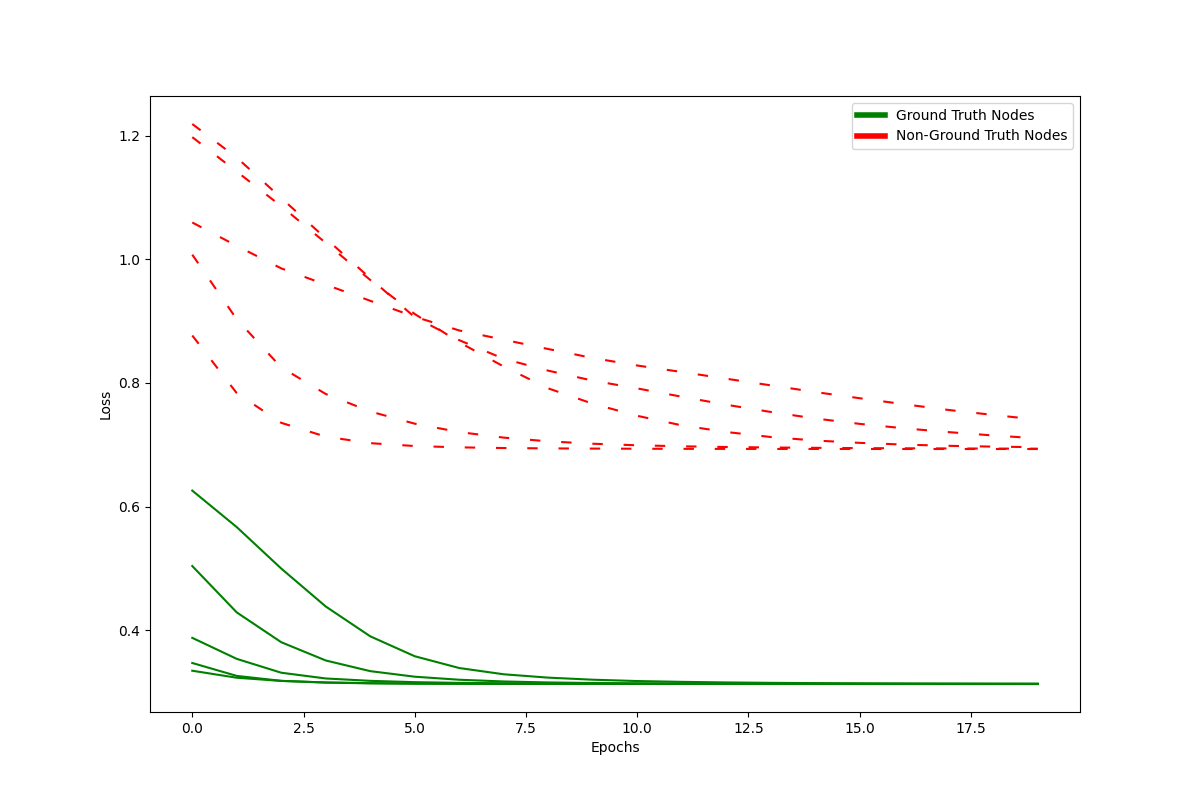}
  \end{subfigure}
    \begin{subfigure}[b]{0.32\textwidth}
    \includegraphics[width=\linewidth]{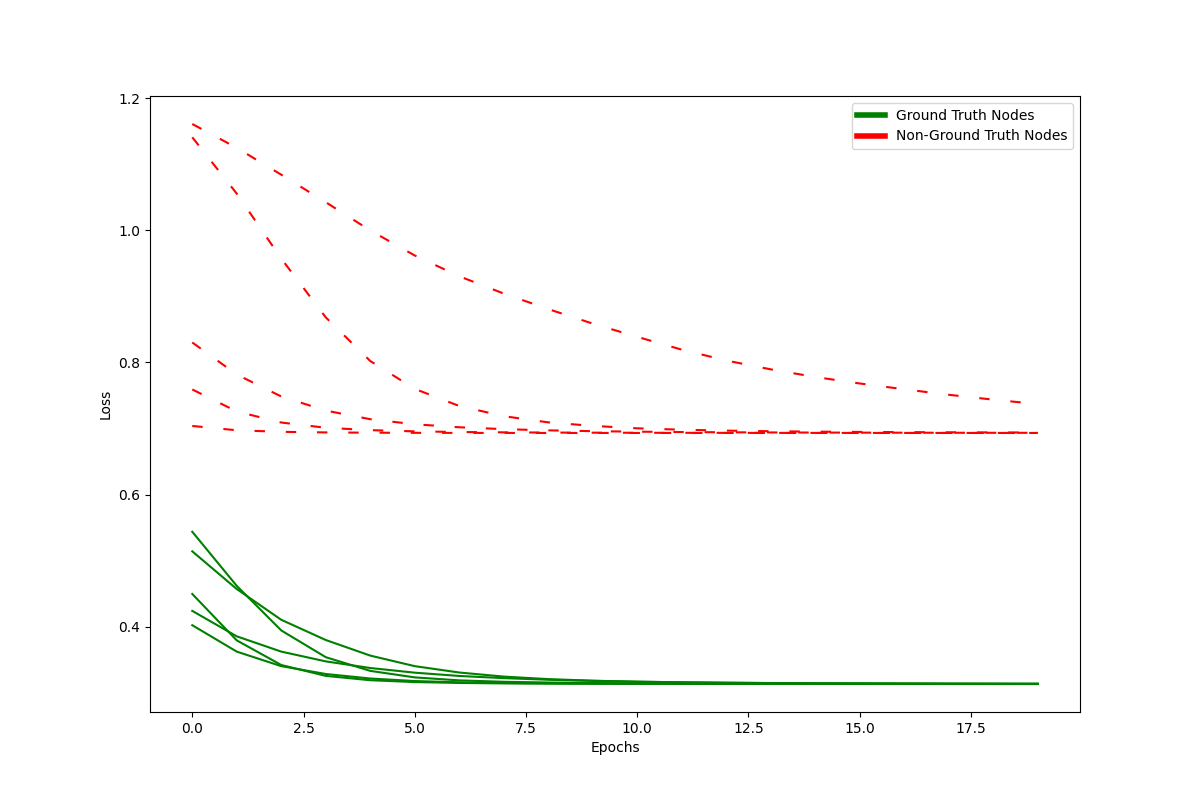}
  \end{subfigure}
  \begin{subfigure}[b]{0.32\textwidth}
    \includegraphics[width=\linewidth]{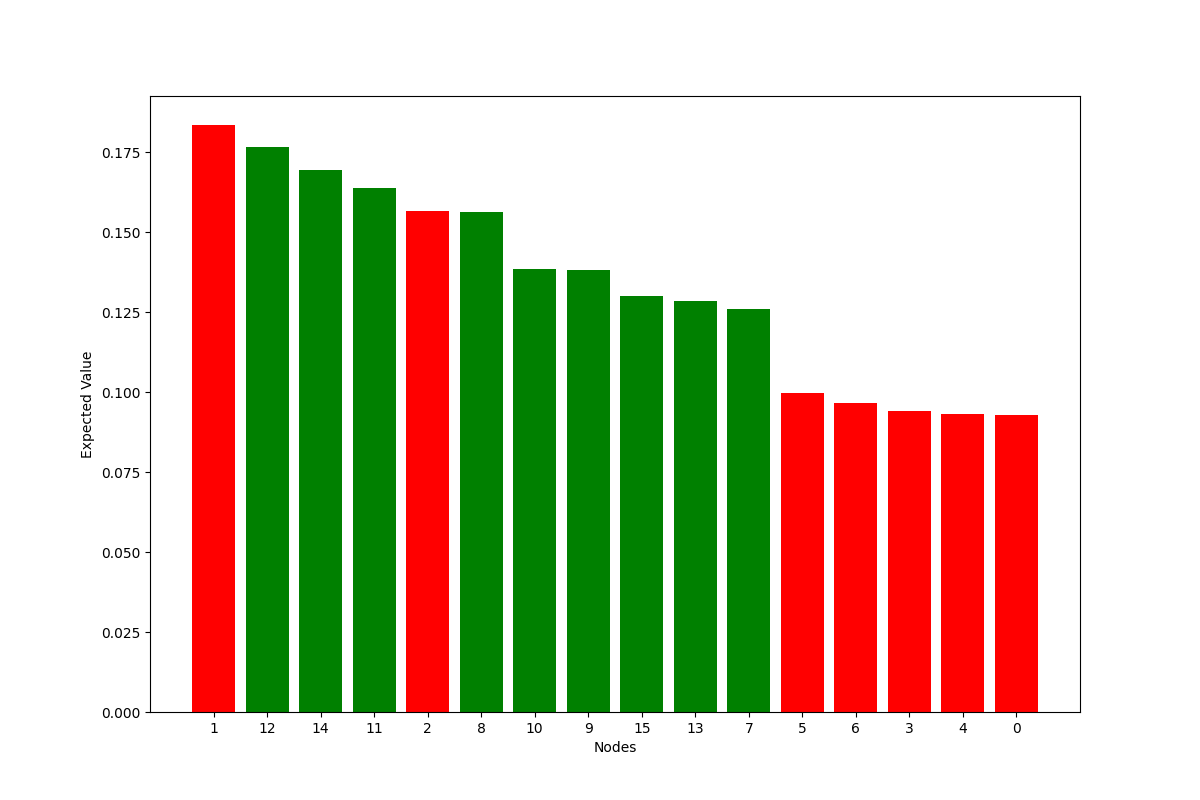}
  \end{subfigure}
  \begin{subfigure}[b]{0.32\textwidth}
    \includegraphics[width=\linewidth]{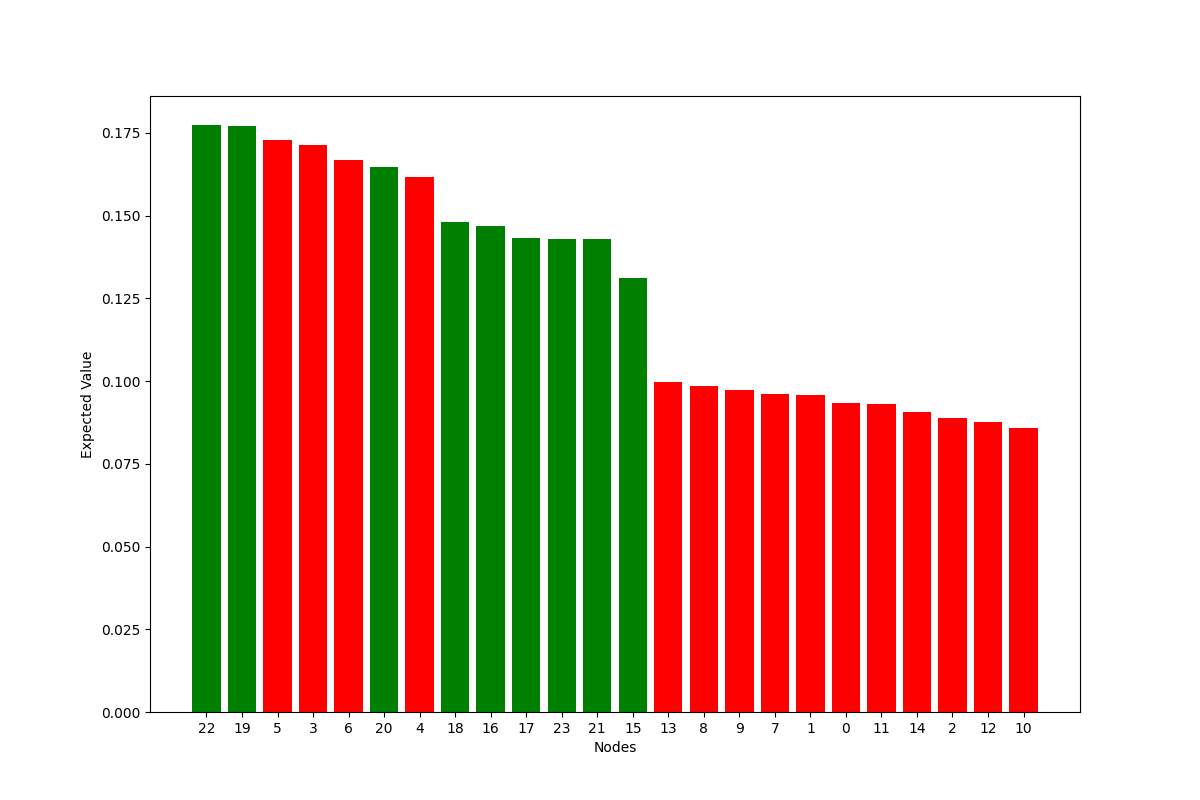}
  \end{subfigure}
  \begin{subfigure}[b]{0.32\textwidth}
    \includegraphics[width=\linewidth]{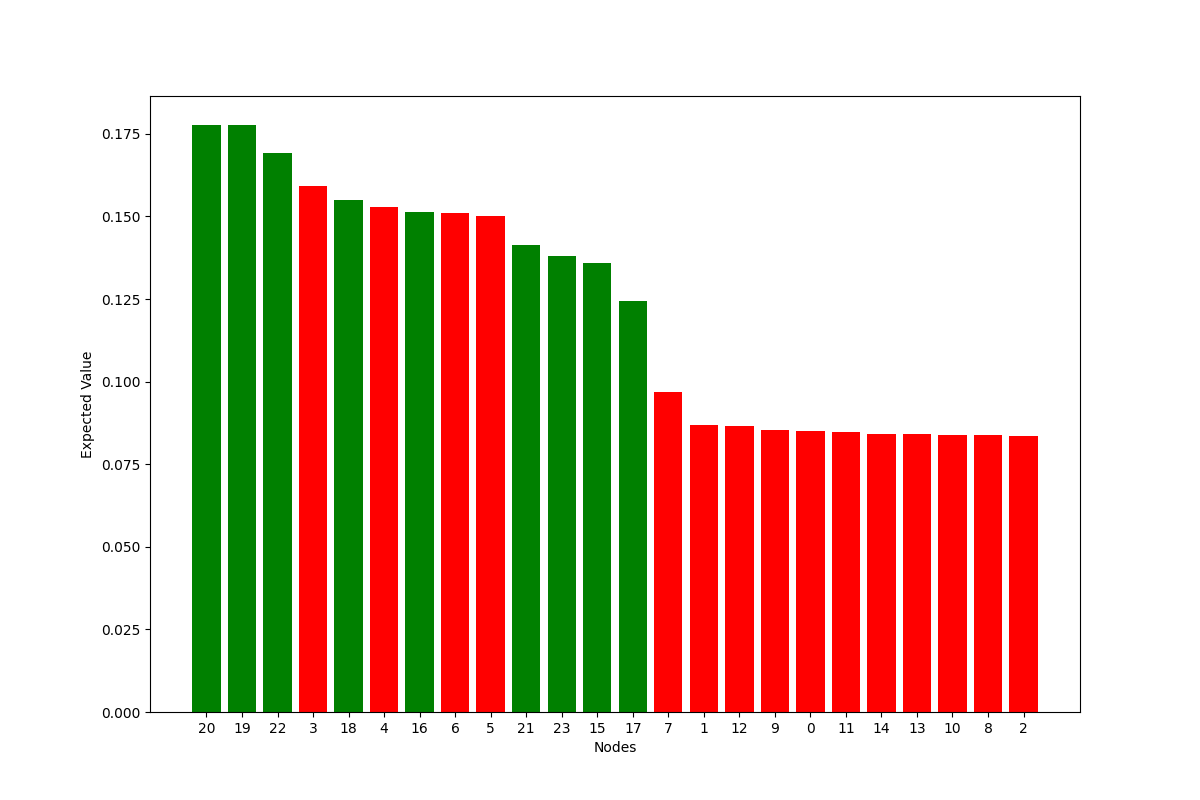}
  \end{subfigure}
\caption{More results on Tree+Grid. 
}
\label{fig:vis_treegrid}
\end{figure}

\begin{figure}
  \centering
  \begin{subfigure}[b]{0.32\textwidth}
    \includegraphics[width=\linewidth]{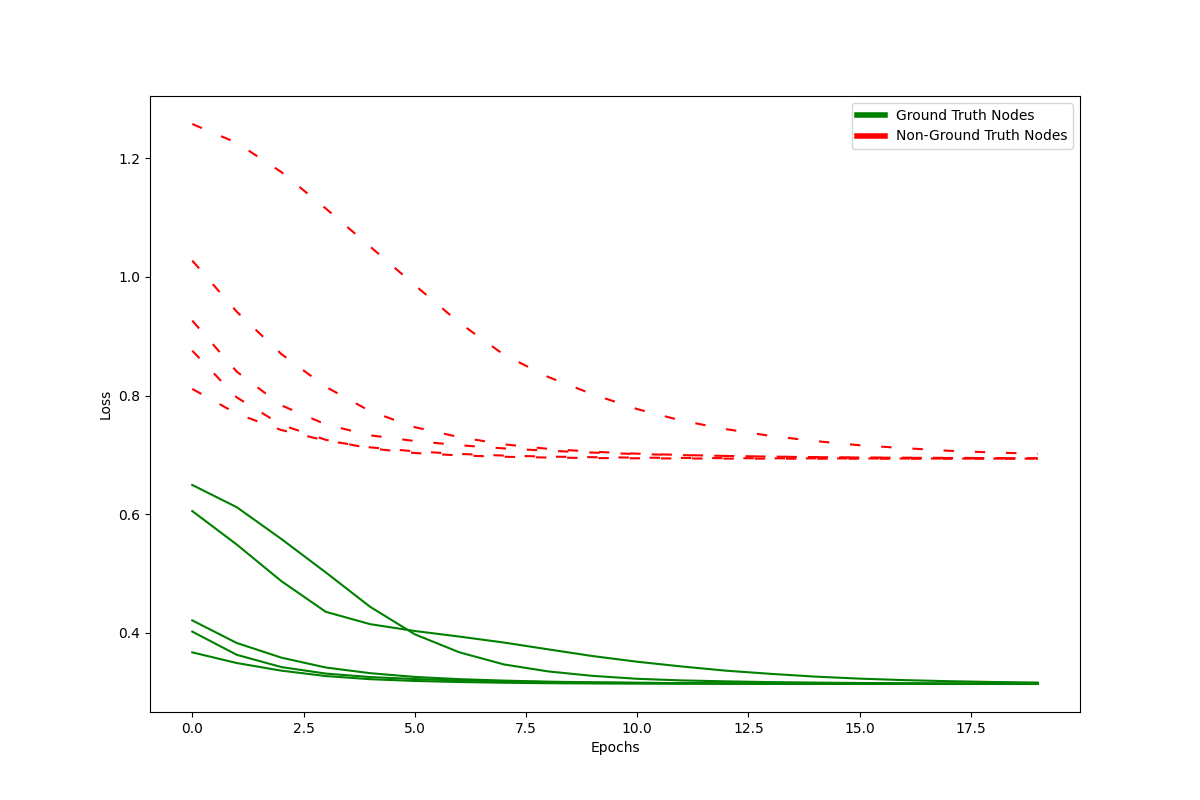}
  \end{subfigure}
  \begin{subfigure}[b]{0.32\textwidth}
    \includegraphics[width=\linewidth]{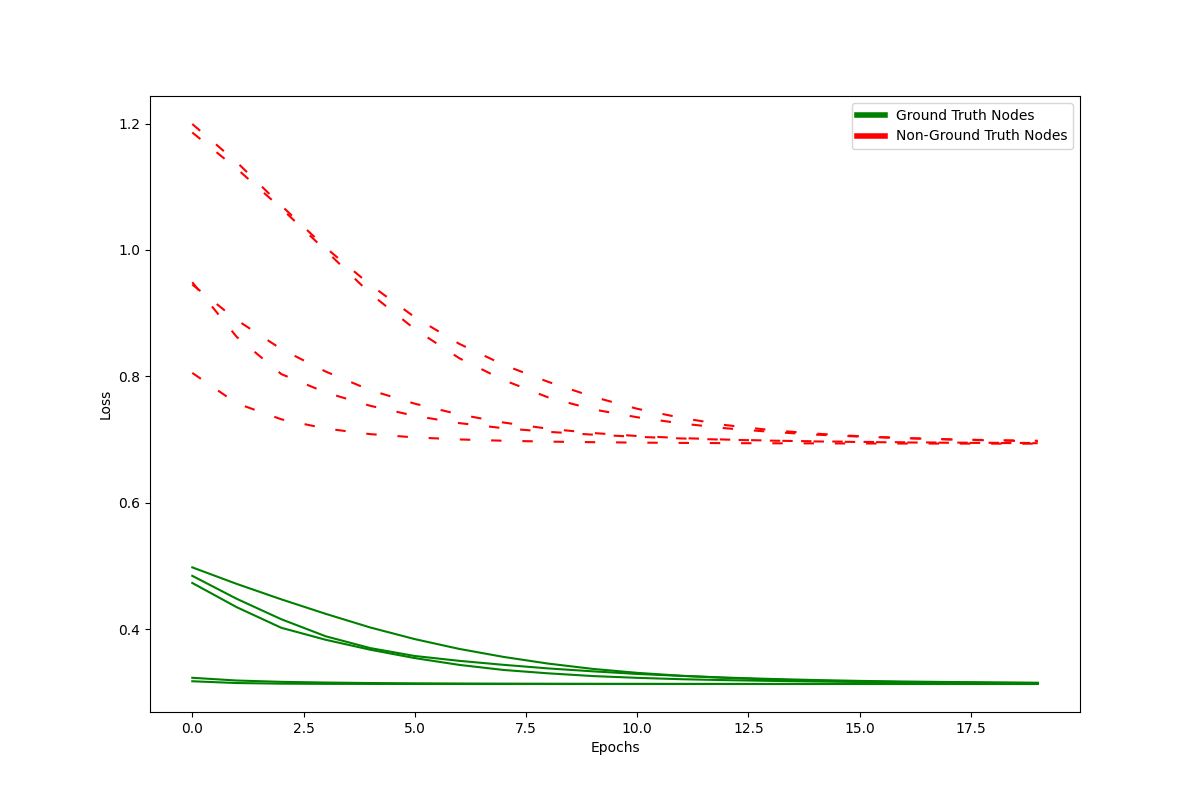}
  \end{subfigure}
    \begin{subfigure}[b]{0.32\textwidth}
    \includegraphics[width=\linewidth]{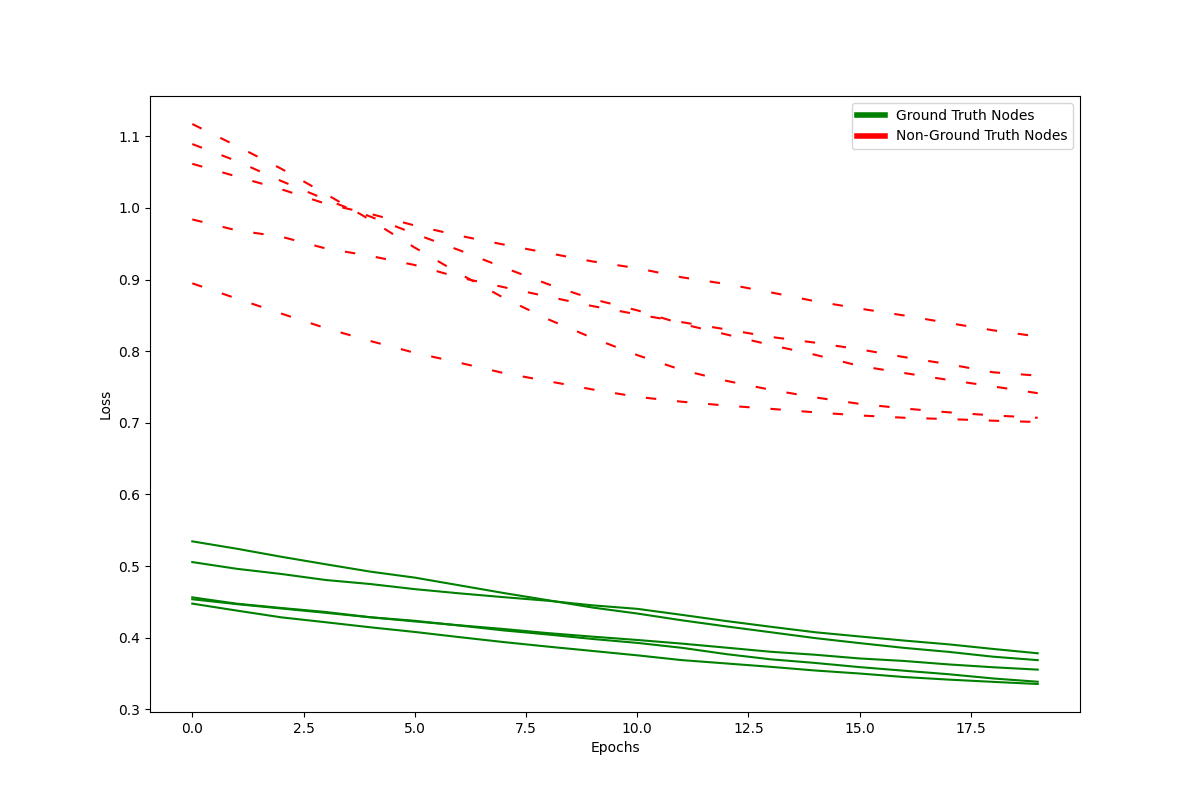}
  \end{subfigure}
  \begin{subfigure}[b]{0.32\textwidth}
    \includegraphics[width=\linewidth]{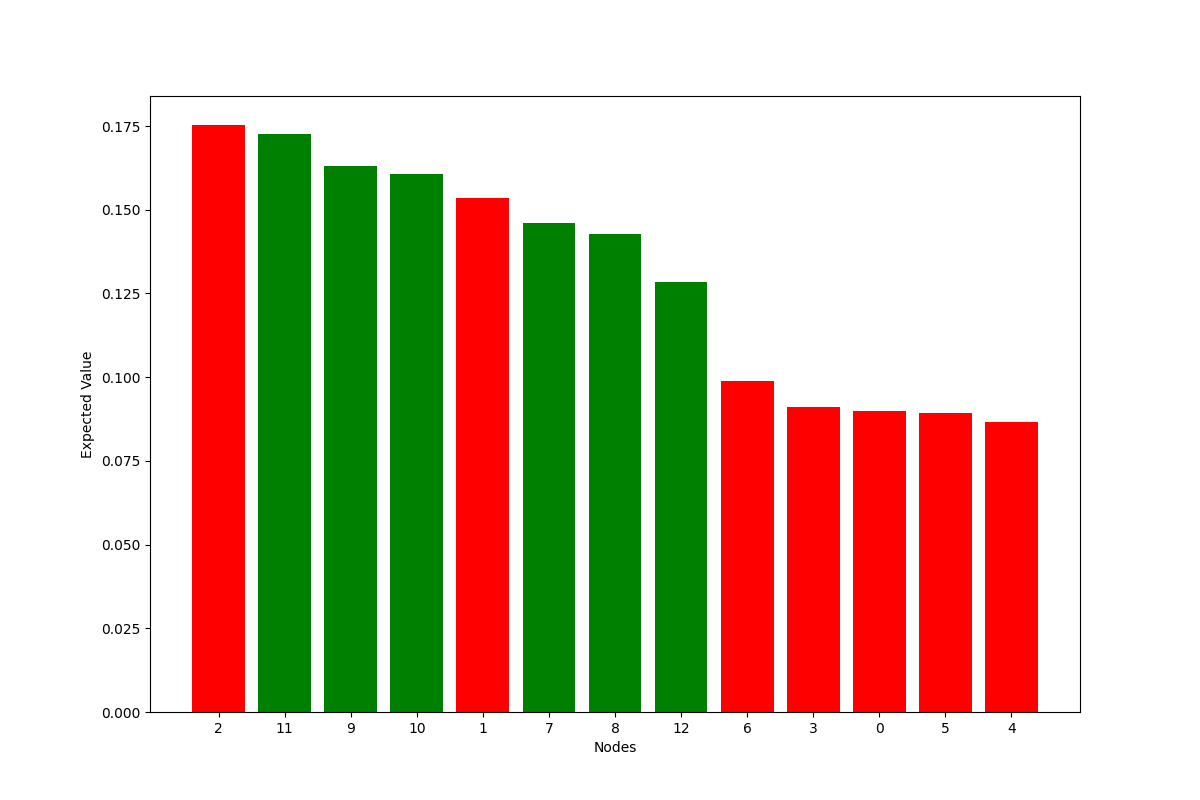}
  \end{subfigure}
  \begin{subfigure}[b]{0.32\textwidth}
    \includegraphics[width=\linewidth]{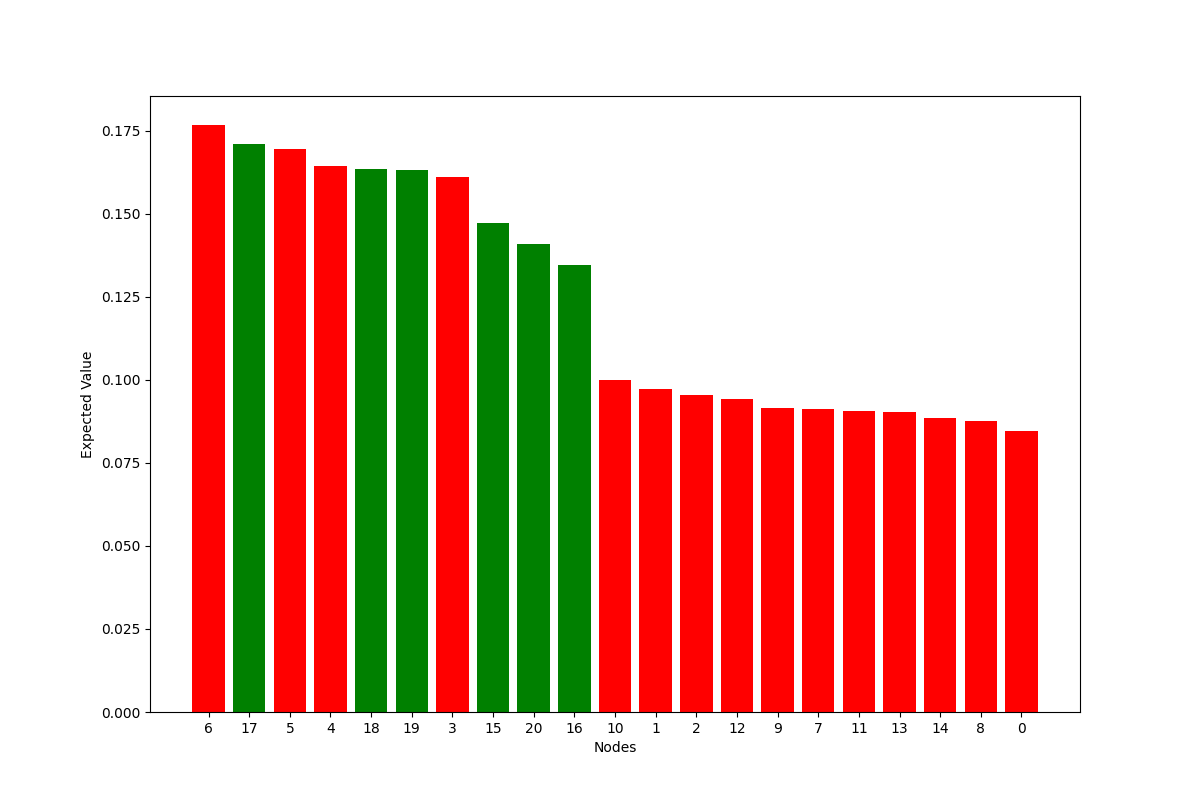}
  \end{subfigure}
  \begin{subfigure}[b]{0.32\textwidth}
    \includegraphics[width=\linewidth]{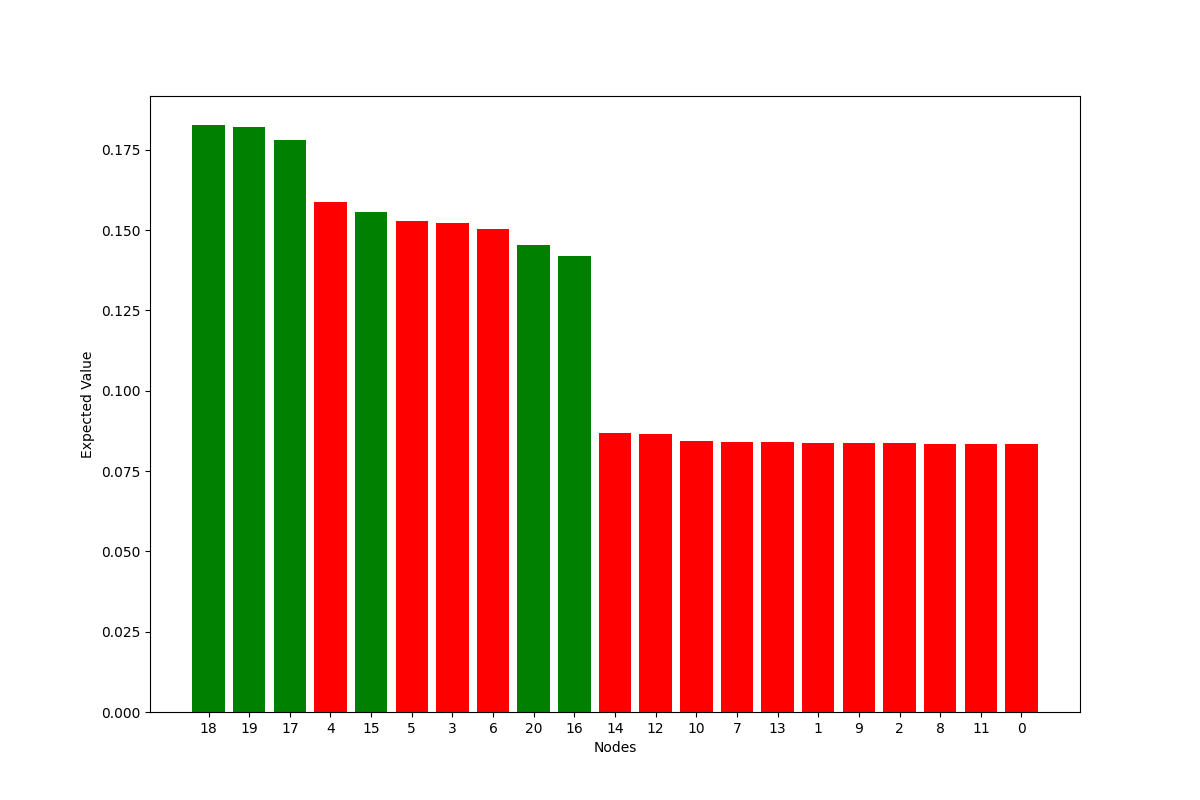}
  \end{subfigure}
\caption{More results on Tree+Cycle. 
}
\label{fig:vis_treecycle}
\end{figure}

\begin{figure}
  \centering
  \begin{subfigure}[b]{0.32\textwidth}
    \includegraphics[width=\linewidth]{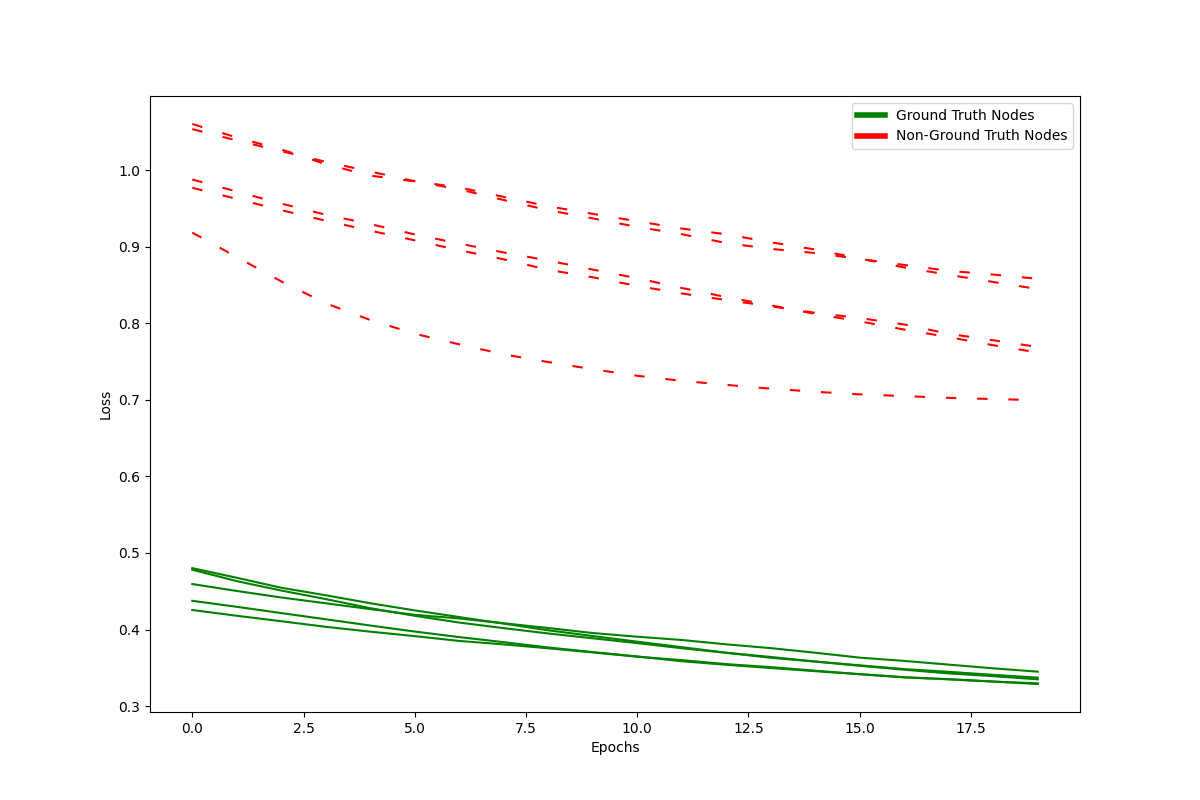}
  \end{subfigure}
  \begin{subfigure}[b]{0.32\textwidth}
    \includegraphics[width=\linewidth]{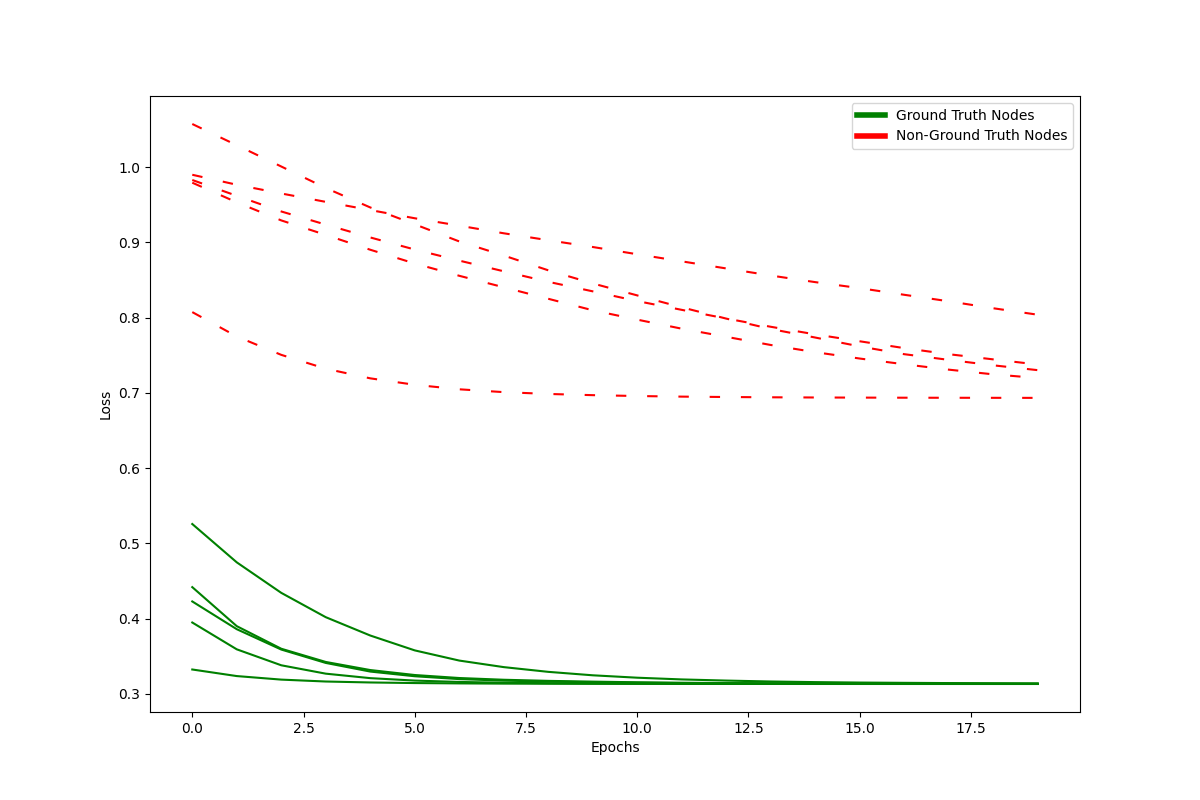}
  \end{subfigure}
    \begin{subfigure}[b]{0.32\textwidth}
    \includegraphics[width=\linewidth]{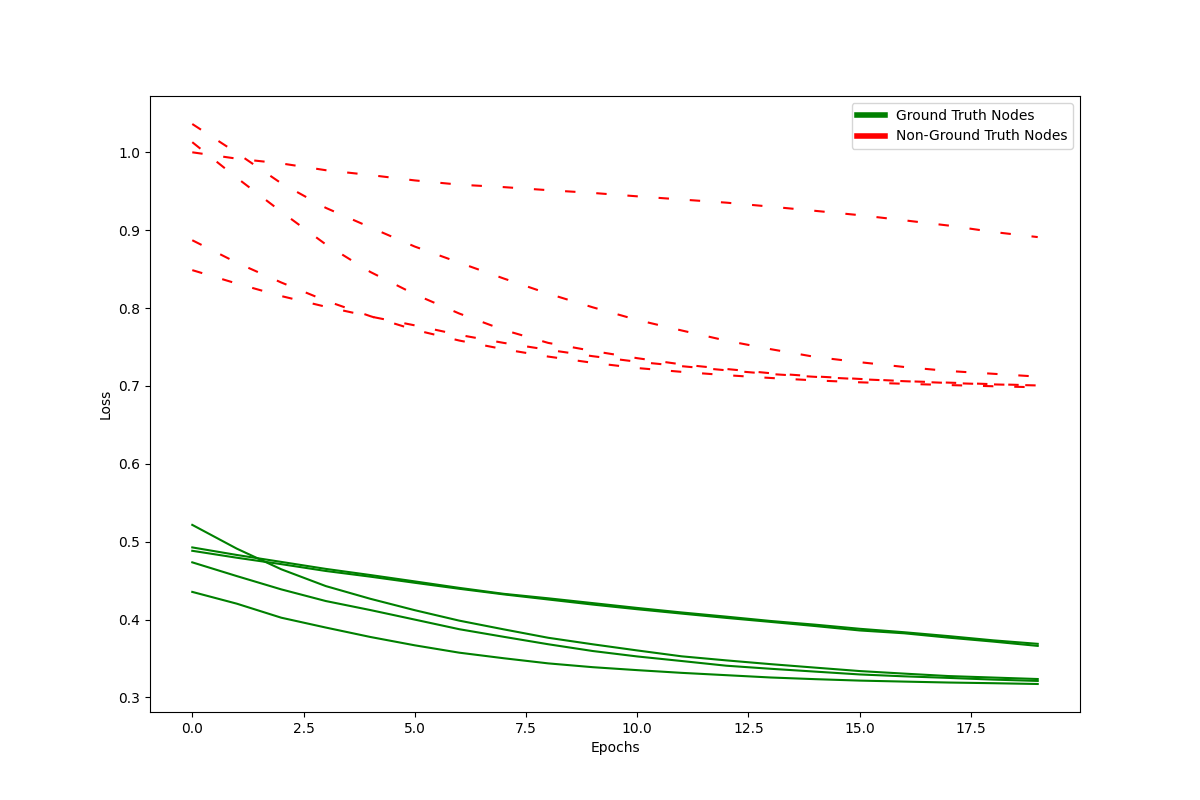}
  \end{subfigure}
  \begin{subfigure}[b]{0.32\textwidth}
    \includegraphics[width=\linewidth]{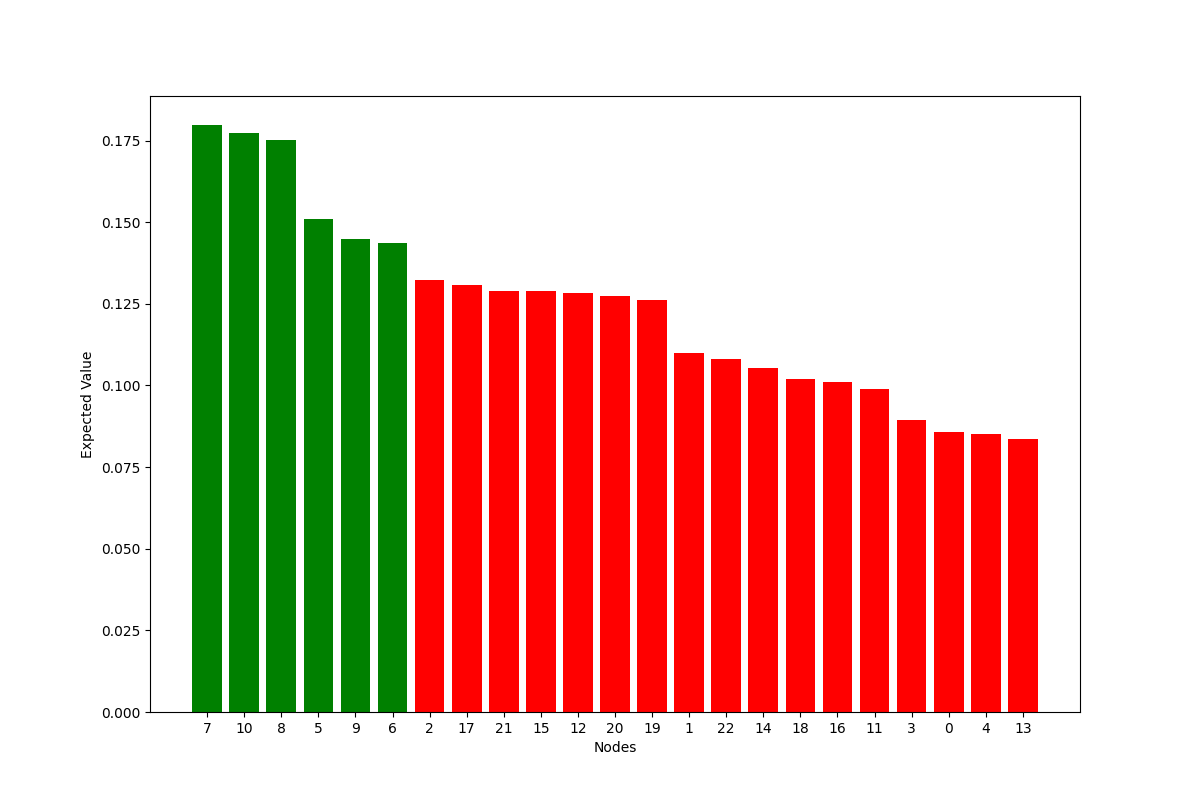}
  \end{subfigure}
  \begin{subfigure}[b]{0.32\textwidth}
    \includegraphics[width=\linewidth]{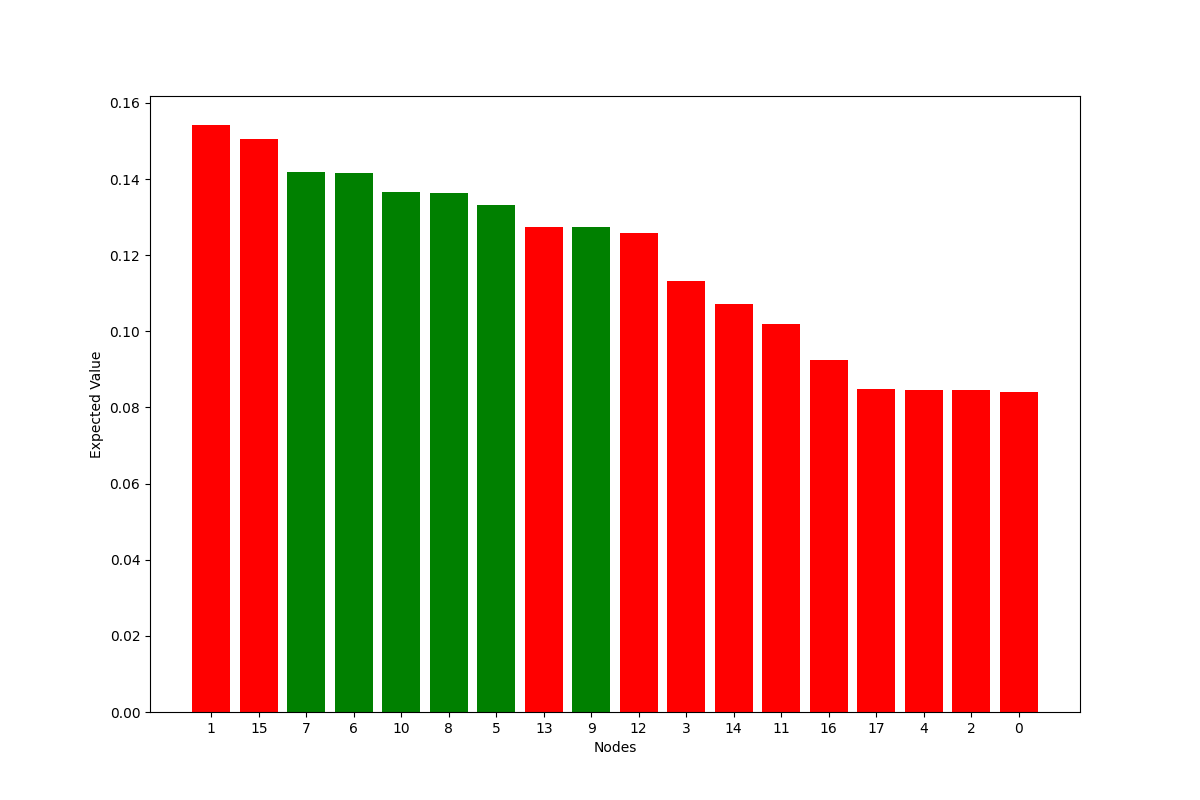}
  \end{subfigure}
  \begin{subfigure}[b]{0.32\textwidth}
    \includegraphics[width=\linewidth]{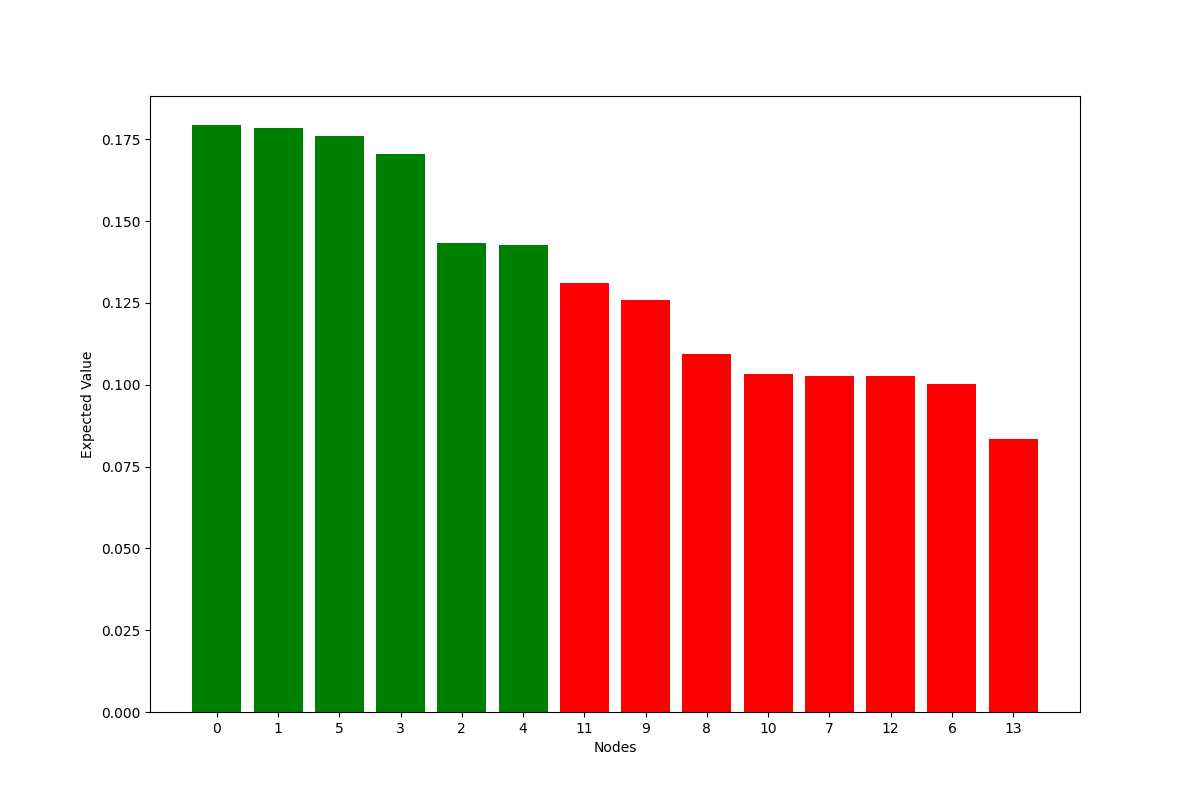}
  \end{subfigure}
\caption{More results on Benzene. 
}
\label{fig:vis_bagrid}
\end{figure}

\begin{figure}
  \centering
  \begin{subfigure}[b]{0.32\textwidth}
    \includegraphics[width=\linewidth]{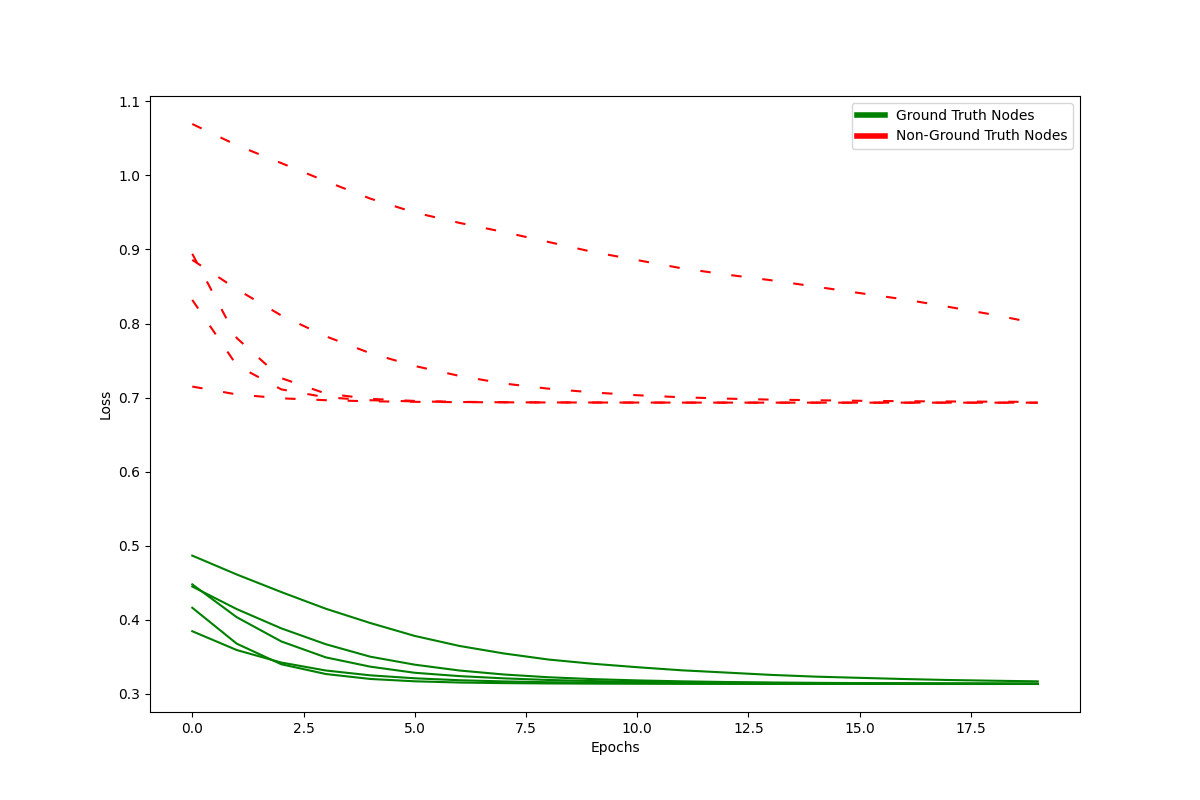}
  \end{subfigure}
  \begin{subfigure}[b]{0.32\textwidth}
    \includegraphics[width=\linewidth]{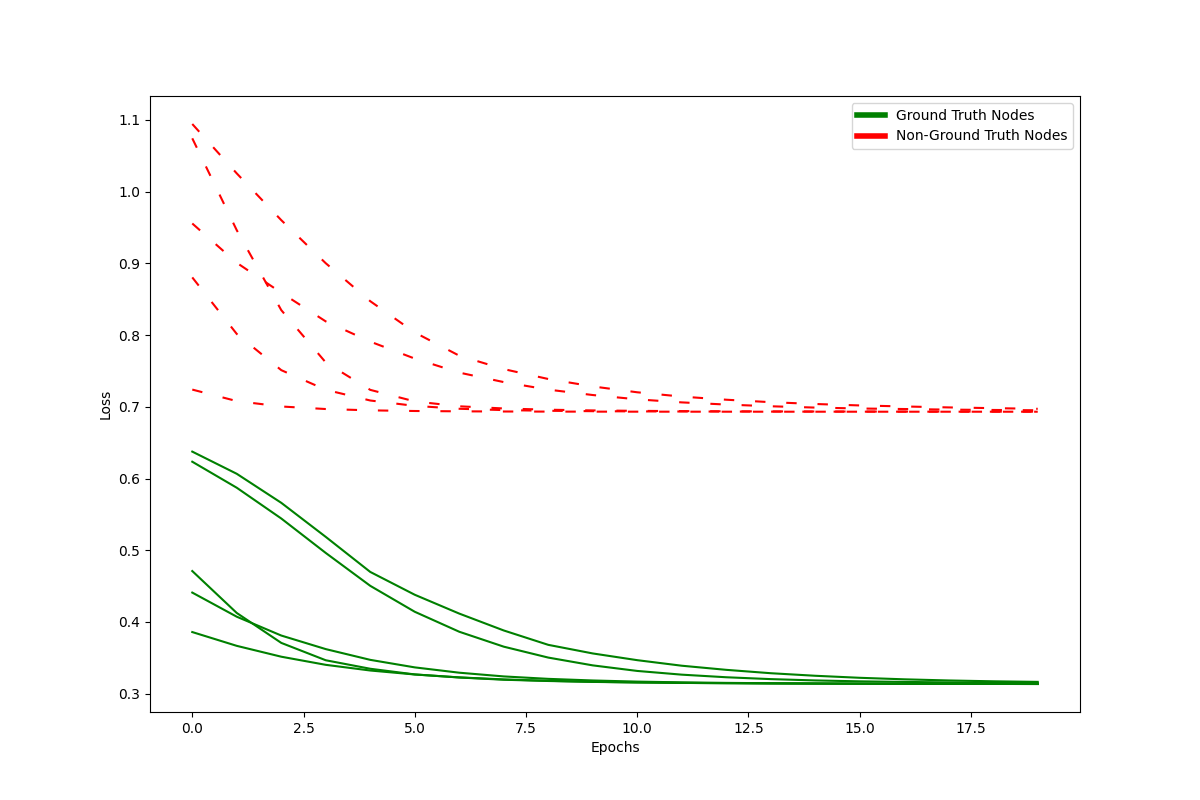}
  \end{subfigure}
    \begin{subfigure}[b]{0.32\textwidth}
    \includegraphics[width=\linewidth]{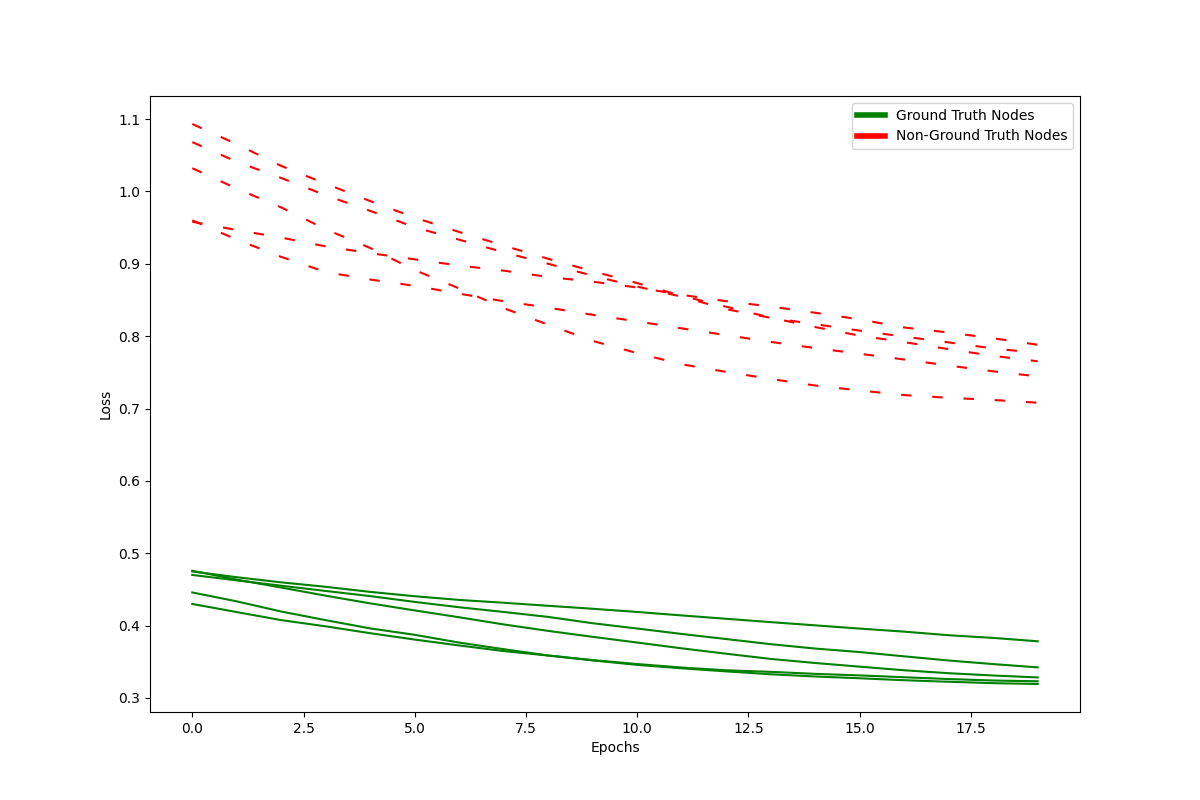}
  \end{subfigure}
    \begin{subfigure}[b]{0.32\textwidth}
    \includegraphics[width=\linewidth]{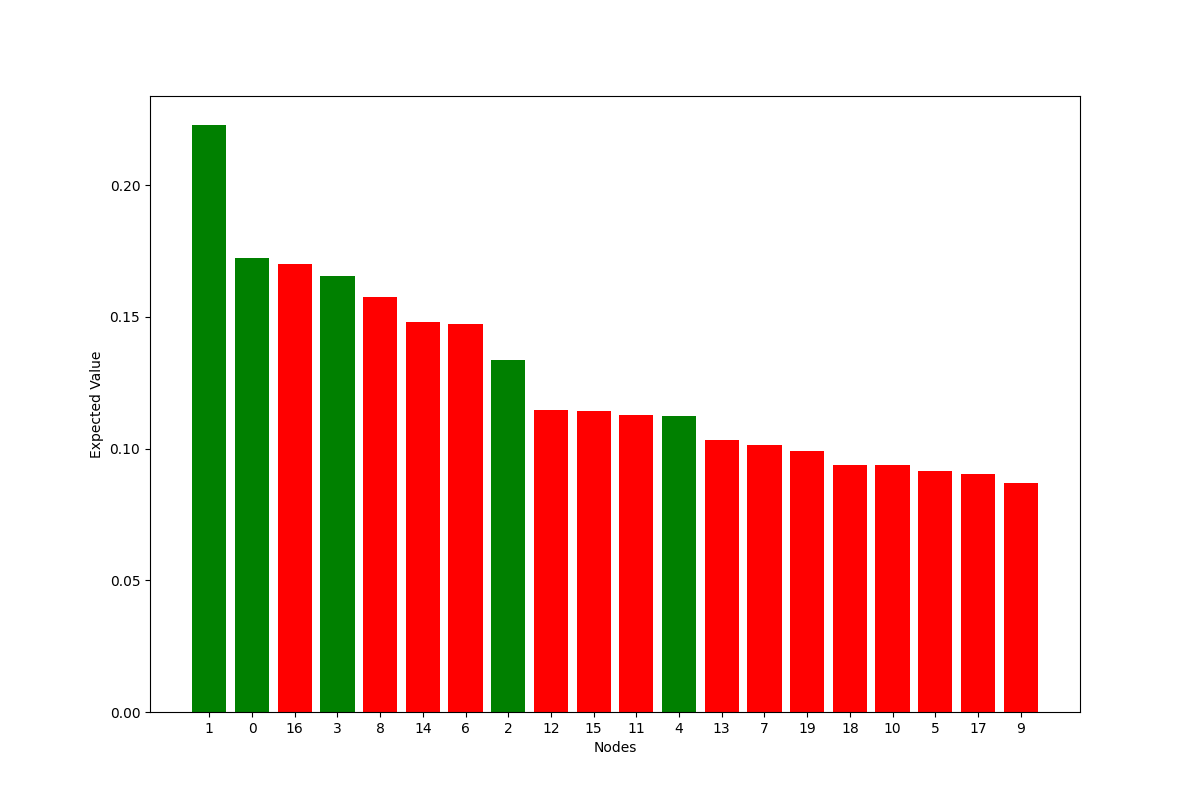}
  \end{subfigure}
  \begin{subfigure}[b]{0.32\textwidth}
    \includegraphics[width=\linewidth]{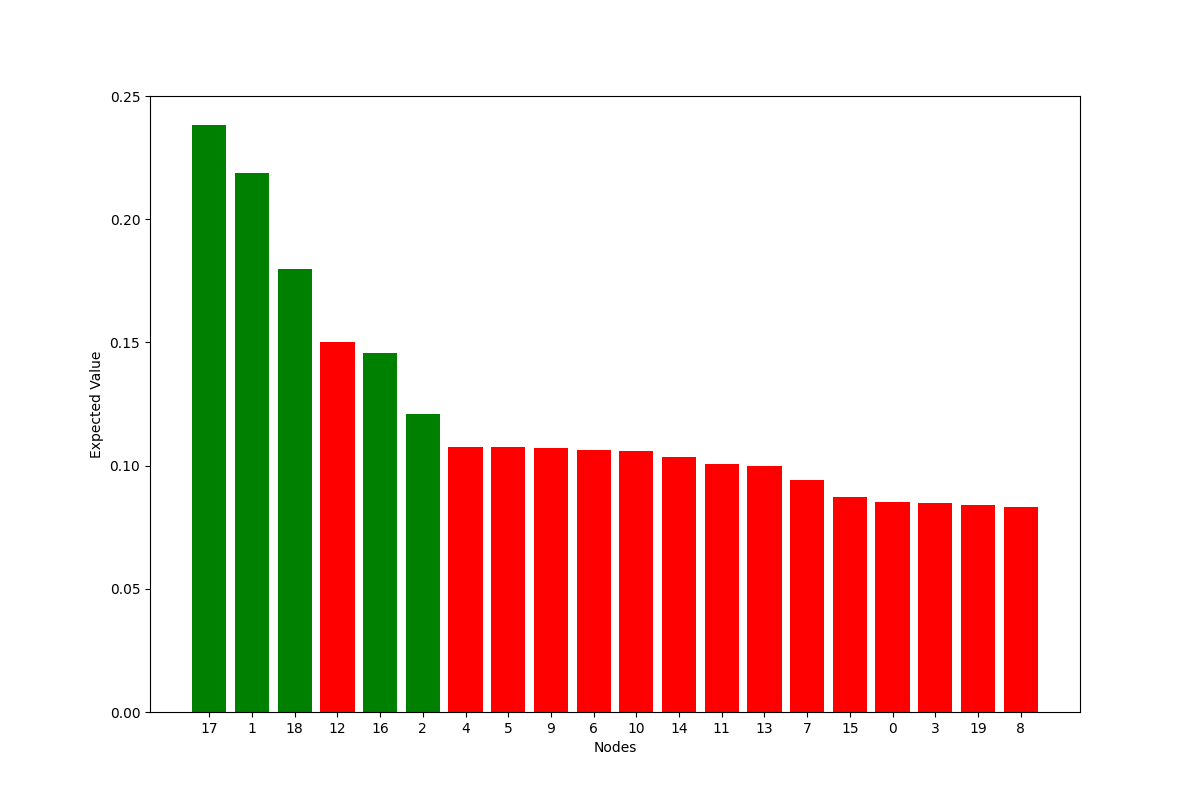}
  \end{subfigure}
  \begin{subfigure}[b]{0.32\textwidth}
    \includegraphics[width=\linewidth]{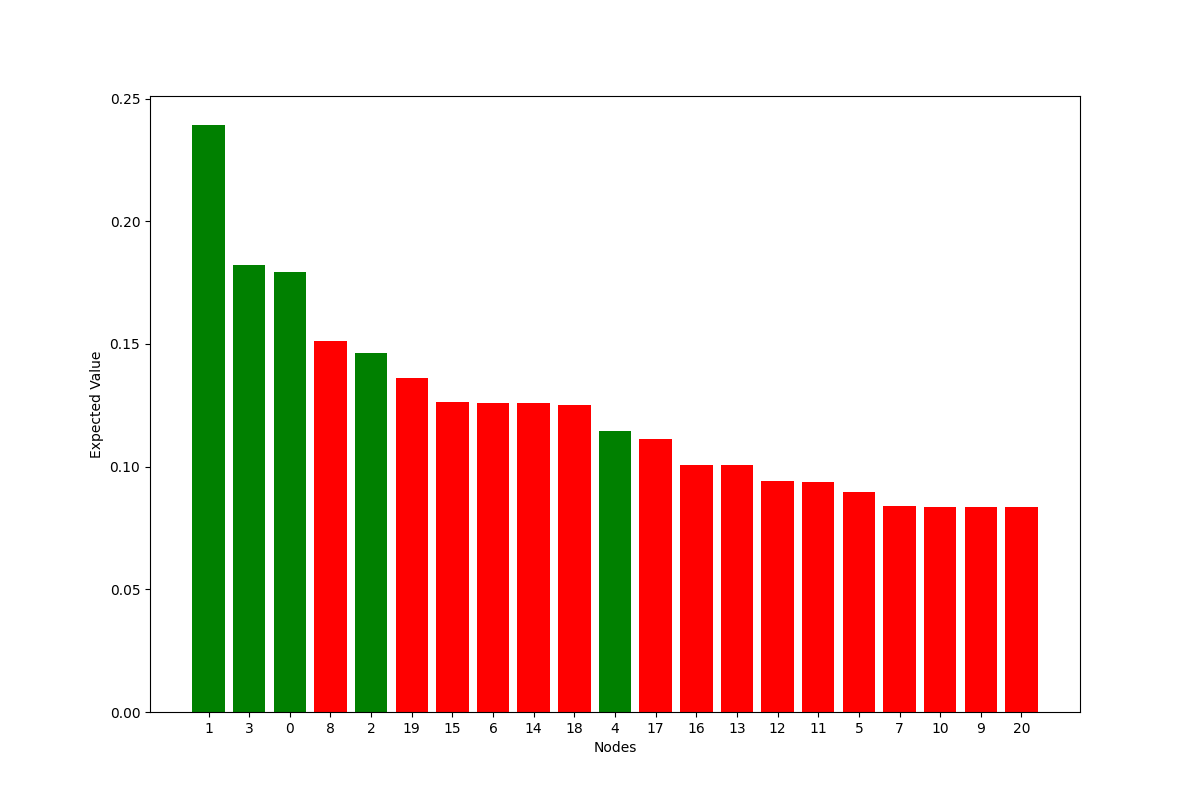}
  \end{subfigure}
\caption{More results on Fluoride Carbonyl. 
}
\label{fig:vis_bagrid}
\end{figure}

\subsection{More experimental results}

More results on the synthetic graphs and real-world graphs in terms of loss curves and node expressivity distributions are shown in the Figure~\ref{fig:vis_bahouse}-Figure~\ref{fig:vis_bagrid}.

\begin{table}[!t]\renewcommand{\arraystretch}{0.9}
\centering
\footnotesize
\caption{Dominant time complexity of the compared explainers. $N, E, d, h, T, L, K$ are \#nodes, \#edges, \#node features, \#neurons, \#training epochs, \#layers, and \#samples. $h=64 \ll d=\sim1000$.} 
\begin{tabular}{lc}
\hline \\
\textbf{GNNExp.} & $O(T*L*N*d^2)$  \\ 
\textbf{PGMExp.} & $O(T*L*N*d^2 + K*(N+E))$  \\
\textbf{Guidedbp} & $O(T * L * N)$  \\
\textbf{GEM} & $O(T * L * N * d^2)$  \\
\textbf{RCExp.} & $O(T*L*(N+E))$  \\
\textbf{OrphicX} & $O(T*L*N*d^2)$  \\
\textbf{\name} & $O(T*L*N*h^2)$ \\
\hline
\end{tabular}
\label{tab:complexity}
\end{table}

\section{Complexity Analysis}
Within a GNN-NCM, we train a feed-forward network. With an \( L \)-layer network and each layer has \( h \) neurons, by training $K$ epochs, the time complexity for a graph with \( n \) nodes is \( O(T*L*N*h^2)  \). Note that training GNN-NCMs for all nodes independently can be easily paralleled via multi-threads/processors. We also show the dominant complexity of compared GNN explainers in Table~\ref{tab:complexity}. Though computing GNN-NCM per node, we can see CXGNN is still more efficient than most of the SOTA explainers (GNNExp., PGMExp., OrphicX, GEM).

\section{Discussion}

\noindent  {\bf Potential risk of overfitting.} In our experiments, we tuned the number of hidden layers and hidden neurons and observed that deeper/wider networks indeed could cause overfitting. Through hyperparameters tuning, we found 2 hidden layers with each layer having 64 neurons that can well balance between model complexity and performance. 

\noindent  {\bf Practical issue of applying CXGNN to large graphs.} We admit 
directly 
running CXGNN in large graphs could have a scalability issue. One solution to speed up the computation is using multi-threads/processors as all nodes can be run independently in CXGNN. Note that all existing GNN explainers also face the same scalability issue, even worse than ours as shown in Table~\ref{tab:complexity}. We acknowledge it is valuable future work to design scalable GNN causal explainers.

\noindent  {\bf True causal subgraph is not present.} Our explainer and causality-inspired ones are all based on the common assumption that a graph consists of the causal subgraph that interprets the prediction. If real-world applications do not satisfy this assumption, all these explainers may not work well.

\noindent {\bf Complexity comparison between NCM and not using NCM (i.e., SCM).}
 Computing the cause-effect in a graph via SCM is computationally intractable. The complexity is exponential to the number of node/edge latent variables. Instead, training an NCM to learn the cause-effect is in the polynomial time.
\end{document}